\documentclass[journal]{IEEEtran}

\usepackage{verbatim}
\usepackage{mathtools}
\usepackage{array}
\usepackage{url}
\usepackage{bm}
\usepackage{amsmath}
\usepackage{amssymb}
\usepackage{graphicx}
\usepackage{caption}
\usepackage{float,booktabs}
\usepackage{subcaption}
\usepackage{setspace}

\usepackage[unicode=true,
 bookmarks=false,
 breaklinks=false,
 pdfborder={0 0 1},
 colorlinks=false]
 {hyperref}
\hypersetup{colorlinks,citecolor=blue,filecolor=blue,linkcolor=blue,urlcolor=blue}

\newcommand{\ltwo}[1]{\|#1\|_2}

\newcommand{\real}{\ensuremath{\mathbb{R}}}

\newcommand{\thetastar}{\ensuremath{\bm{\theta}^\star}}

\newcommand{\Ind}{\ensuremath{\textbf{I}}}

\newcommand{\defn}{\coloneqq}

\newcommand{\Exs}{\ensuremath{\mathbb{E}}}

\newcommand{\HACKPROOF}{\begin{proof}}
\newcommand{\HACKENDPROOF}{\end{proof}}

\newlength{\widebarargwidth}
\newlength{\widebarargheight}
\newlength{\widebarargdepth}

\makeatletter
\long\def\@makecaption#1#2{
        \vskip 0.8ex
        \setbox\@tempboxa\hbox{\small {\bf #1:} #2}
        \parindent 1.5em  
        \dimen0=\hsize
        \advance\dimen0 by -3em
        \ifdim \wd\@tempboxa >\dimen0
                \hbox to \hsize{
                        \parindent 0em
                        \hfil 
                        \parbox{\dimen0}{\def\baselinestretch{0.96}\small
                                {\bf #1.} #2
                                } 
                        \hfil}
        \else \hbox to \hsize{\hfil \box\@tempboxa \hfil}
        \fi
        }
\makeatother

\makeatletter

\usepackage{amsthm}
\newcommand{\cS}{\mathcal{S}}
\newcommand{\cA}{\mathcal{A}}

\newcommand{\bmtheta}{\bm\theta}
\newcommand{\thetaavg}[1]{\overline{\bm{\theta}}_{#1}}
\newcommand{\bmSigma}{\bm \Sigma}
\newcommand{\lsigma}[1]{\|#1\|_{\bmSigma}}

\newcommand{\lsigmaBig}[1]{\Big\|#1\Big\|_{\bmSigma}}
\newcommand{\Del}[1]{\bm \Delta_{#1}}
\newcommand{\Delbar}[1]{\overline{\bm \Delta}_{#1}}

\newcommand{\citep}{\cite}
\newcommand{\citet}{\cite}

\newcommand{\lambdamax}{\lambda_{\mathsf{max}}(\bm{\Sigma})}
\newcommand{\lambdamin}{\lambda_{\mathsf{min}}(\bm{\Sigma})}

\newcommand{\mymid}{\,|\,}
\usepackage{dsfont}
\newcommand{\ind}{\mathds{1}} 
\usepackage{babel}
\usepackage{color}

\definecolor{yxc}{RGB}{255,0,0}
\definecolor{yjc}{RGB}{225,0,0}
\definecolor{cm}{RGB}{0,0,200}
\definecolor{ytw}{RGB}{255,0,0}
\definecolor{gen}{RGB}{0,200,0}
\definecolor{wwc}{RGB}{0,125,0}

\newcommand{\yc}[1]{\textcolor{yjc}{[YJC: #1]}}

\newcommand{\wwc}[1]{\textcolor{wwc}{[WWC: #1]}}

\allowdisplaybreaks
\let\hat\widehat

\makeatother

\begin{document}
\theoremstyle{plain} \newtheorem{lemma}{\textbf{Lemma}}\newtheorem{proposition}{\textbf{Proposition}}\newtheorem{theorem}{\textbf{Theorem}}\newtheorem{assumption}{Assumption}\newtheorem{corollary}{\textbf{Corollary}}\newtheorem{claim}{\textbf{Claim}}

\theoremstyle{remark}\newtheorem{remark}{\textbf{Remark}} 

\title{High-probability sample complexities for policy evaluation with linear function approximation}

\author{Gen~Li,~\IEEEmembership{Member,~IEEE,}
	    Weichen~Wu,~
	    Yuejie~Chi,~\IEEEmembership{Fellow,~IEEE,}
	    Cong~Ma,~\IEEEmembership{Member,~IEEE,}
	    Alessandro~Rinaldo,~
	    Yuting~Wei,~\IEEEmembership{Member,~IEEE}%
\thanks{The first two authors contributed equally.}
\thanks{The work of Gen Li was supported in part by the Chinese University of Hong Kong Direct Grant for Research.
The work of Weichen Wu and Alessandro Rinaldo were supported in part by the NIH Grant R01 NS121913. 
The work of Yuejie Chi was supported in part by the grants Office of Naval Research (ONR) N00014-19-1-2404 and NSF CCF-2106778.
The work of Cong Ma was supported in part by NSF grant DMS-2311127.
The work of Yuting Wei was supported in part by the NSF grants CCF-2106778, DMS-2147546/2015447 and NSF CAREER award DMS-2143215 and Google Research Scholar Award. \texttt{(Correspondent author: Yuting Wei.)}}
\thanks{Gen Li is with the Department of Statistics, The Chinese University of Hong Kong, Hong Kong.}%
\thanks{Weichen Wu is with the Department of Statistics and Data Science, Carnegie Mellon University, Pittsburgh, PA 15213, USA.}%
\thanks{Yuejie Chi is with the Department of Electrical and Computer Engineering, Carnegie Mellon University, Pittsburgh, PA 15213, USA.}%
\thanks{Cong Ma is with the Department of Statistics, University of Chicago, Chicago, IL 60637, USA.}%
\thanks{Alessandro Rinaldo is with the Department of Statistics and Data Science, University of Texas at Austin, Austin, UT 78712, USA.}%
\thanks{Yuting Wei is with the Department of Statistics and Data Science, The Wharton School, University of Pennsylvania, Philadelphia, PA 19104, USA; email: \texttt{ytwei@wharton.upenn.edu}}
}


\maketitle

\begin{abstract}
This paper is concerned with the problem of policy evaluation with linear function approximation in discounted infinite horizon Markov decision processes. We investigate the sample complexities required to guarantee a predefined estimation error of the best linear coefficients for two widely-used policy evaluation algorithms: the temporal difference (TD) learning algorithm and the two-timescale linear TD with gradient correction (TDC) algorithm. In both the on-policy setting, where observations are generated from the target policy, and the off-policy setting, where samples are drawn from a behavior policy potentially different from the target policy, we establish the first sample complexity bound with high-probability convergence guarantee that attains the optimal dependence on the tolerance level. We also exhibit an explicit dependence on problem-related quantities, and show in the on-policy setting that our upper bound matches the minimax lower bound on crucial problem parameters, including the choice of the feature map and the problem dimension.

\end{abstract}

\begin{IEEEkeywords}
policy evaluation, temporal difference learning, two-timescale stochastic approximation, minimax optimal, function approximation
\end{IEEEkeywords}




\section{Introduction}

Policy evaluation plays a critical role in many scientific and engineering applications in which practitioners aim to evaluate the performance of a target strategy based on either sequentially collected or a batch of offline data samples \citep{murphy2003optimal,bojinov2019time,tang2021model,dann2014policy}. 
For example, in clinical trials \citep{tang2021model}, real-time data acquisition might be expensive and risky; it is thus of essential value if historical data can be analyzed and information can be transferred to new tasks. While in other applications, such as mobile health \citep{bertsimas2022data}, it is practical to implement the desired policy and collect its feedback in a timely manner. 

Mathematically, Markov decision processes (MDPs) provide a general framework to design policy evaluation methods in dynamic settings; reinforcement learning (RL) is often modeled using MDPs when the exact model configuration is not available \citep{sutton2018reinforcement,bertsekas2017dynamic}. 
In this framework, a target policy is assessed through its corresponding value function. 
In practice, evaluating value functions often require an overwhelming number of samples due to the large dimensionality of the underlying state space. For this reason, RL methods are normally concerned with some form of function approximation. 
Dating back to the seminal work of \cite{tsitsiklis1997analysis}, there has been an extensive line of works that consider different types of function approximation, including linear function approximation \citep{bhandari2018finite,fan2020theoretical}, reproducing kernel Hilbert space \citep{farahmand2016regularized,duan2021optimal}, deep neural networks \citep{bertsekas1995neuro,arulkumaran2017brief} or function approximation on the model itself (see, e.g.~\cite{jin2020provably,wang2021sample,li2021sample-eff}), with a focus on improving the sample efficiency of RL algorithms.

\paragraph{Two settings: on-policy vs.~off-policy} 
The main goal of this paper is to provide sharp statistical guarantees of policy evaluation algorithms with linear function approximation in two different settings.
As the aforementioned examples already indicated, there are typically two different types of data-generating mechanisms to consider: the \emph{on-policy} setting when we have access to the outcomes of the target policy and the \emph{off-policy} setting, in which the only available data are generated from a behavior policy that is potentially different from the target policy.

In the on-policy setting, temporal difference (TD) learning is arguably the most popular  algorithm \citep{sutton1988learning} for policy evaluation in RL practice, partly because it is easy to implement and lends itself well to function approximations. 
As a model-free algorithm, TD learning processes data in an online manner without explicitly modeling the environment and is, therefore, memory efficient. 
While the asymptotic convergence of TD with linear function approximation has been known since \cite{tsitsiklis1997analysis}, the finite-sample minimax optimality of TD has been established only recently for the tabular MDP \citep{li2021q}. 
For TD learning with linear function approximation, several recent contributions have produced new non-asymptotic analyses and insights (e.g.~\cite{bhandari2018finite,dalal2018finite,srikant2019finite,lakshminarayanan2018linear}), which partially unveil impacts of both the tolerance level and various problem-related parameters on its sample efficiency. 
However, minimax-optimal dependence on the tolerance level (i.e.~target level of estimation accuracy) is only established in expectation instead of with high probability; furthermore, the optimal dependence on problem-related parameters, such as the size of the state space and the effective horizon,  still remains unsettled, and it is unclear whether existing sample complexity bounds can be further improved. Failing to understand these questions, however, casts doubt on whether TD with linear function approximation is statistically efficient in practice, and brings difficulties to performing statistical inference based on TD estimators. 
In this paper, we seek to answer these questions by providing tighter characterizations of the performance of TD with linear function approximation.

In the off-policy setting, it is known that the error of TD learning with linear function approximation may diverge to infinity \citep{baird1995residual}. 
In order to address this issue, \cite{sutton2009fast} proposed a now popular alternative with two-timescale learning rates, called the linear TD with gradient correction (TDC) algorithm, which enjoys convergence guarantees in the off-policy case. 
In terms of finite-sample guarantees, although a number of recent efforts (see, e.g.~\cite{dalal2018finite_two,xu2021sample,wang2021non,dalal2020tale,kaledin2020finite,gupta2019finite}) tried to characterize the statistical performance of TDC for both $i.i.d.$ and Markovian data, 
they remain inadequate in providing either a convergence guarantee with high-probability, an explicit dependence on salient problem parameters, or a sharp dependence on the sample size. 
The challenge lies in dealing with the statistical dependence between two separate iterate sequences at different timescales. To tackle this challenge, it calls for a new analysis framework for the TDC algorithm.

\subsection{Our main contributions}

This paper is concerned with evaluating the performance of a given target policy $\pi$ in an infinite-horizon $\gamma$-discounted MDP with a finite but large number of states. 
The goal is to learn the best linear approximation of the value function in a pre-specified feature space given $i.i.d.$ transition pairs drawn from the stationary distribution. 
In the on-policy setting, we focus on the TD learning algorithm; in the off-policy setting, we shift gear to  the TDC learning algorithm. 
We summarize our main contributions as follows, with  their exact statements and consequences postponed to later sections. 
\begin{itemize}
\begin{comment}
\item We prove that in order to guarantee an estimation error of at most $\epsilon$, the number of independent samples needed for the LSTD algorithm with linear function approximation is no more than the order of
\[
\frac{\max_s \bm{\phi}(s)^\top \bm{\Sigma}^{-1} \bm{\phi}(s) (1+\|\bm{\theta}^\star\|_{\bm{\Sigma}}^2)}{(1-\gamma)^2 \epsilon^2};
\]
\end{comment}

\item Via a careful analysis of TD learning with Polyak-Ruppert averaging, we show that, in the on-policy setting, a number of samples of order 
\begin{align*}
	\widetilde{O} \left(\frac{\max_s \{\bm{\phi}(s)^\top \bm{\Sigma}^{-1} \bm{\phi}(s)\} (1+\|\bm{\theta}^\star\|_{\bm{\Sigma}}^2)}{(1-\gamma)^2 \varepsilon^2}\right)
\end{align*}
is sufficient to achieve an accuracy level (estimation error) of $\varepsilon > 0$, with high probability. 
Here, $\bm{\phi}(s) \in \real^{d}$ indicates the linear feature vector for the state $s$ in the state space $\mathcal{S}$, $\bm{\theta}^\star$ is the best linear approximation coefficient of the value function, and $\bm{\Sigma}$ corresponds to the feature covariance matrix weighted by the stationary distribution. See Section~\ref{sec:models} for the definitions of these parameters.
Compared to prior work by~\cite{bhandari2018finite} and ~\cite{srikant2019finite}, our sample complexity bound can be tighter by a factor of $\mathsf{cond}(\bm{\Sigma})$ which can be as large as $|\mathcal{S}|$ (the cardinality of the state space). Our result also controls $\varepsilon$-convergence with high probability that matches the minimax-optimal dependence on the tolerance level $\varepsilon$ with lowest burn-in cost. 
To assess the tightness of this upper bound, we provide a minimax lower bound in Section~\ref{sec:minimax-lb}, which certifies the optimal dependence of our bound on both the tolerance level $\varepsilon$ and problem-related parameters $\bm{\Sigma}$ and $\thetastar$.


\item In the off-policy setting, we establish a sample complexity bound for the TDC algorithm of order 
\begin{align*}
\widetilde{O} \left(\frac{\rho_{\max}^7}{\lambda_1^4 \lambda_2^3 }\frac{\|\widetilde{\bm{\Sigma}}\|^2}{\varepsilon^2} (1+\|\widetilde{\bm{\theta}}^\star\|_{\widetilde{\bm{\Sigma}}}^2)
\right),
\end{align*}
where $\widetilde{\bm{\theta}}$ corresponds to the best linear approximation coefficient of the value function in the off-policy setting, $\widetilde{\bm{\Sigma}}$ is the feature covariance matrix under the behavior policy, 
$\rho_{\max}$ denotes the largest importance sampling ratio measuring the discrepancy between the target policy and the behavior policy, 
and lastly, $\lambda_1$ and $\lambda_2$ denote the smallest eigenvalues of some problem-dependent matrices. 
Details about these constants are deferred to Section~\ref{sec:TDC}.
To the best of our knowledge, our bound is the first one to control $\varepsilon$-convergence with high probability that matches the minimax-optimal dependence on the tolerance level $\varepsilon$. At the same time, our sample complexity bound also provides an explicit dependence on the salient parameters.
\end{itemize}
Comparisons of our results to existing bounds and relevant commentary can be found in Table~\ref{table-prior-TD} and ~\ref{table-prior-TDC}.

\begin{table*}[t]
\centering
\renewcommand{\arraystretch}{2.3}
\begin{tabular}{c|c|c|c|c} 
\hline
\toprule
paper & algorithm  &  stepsize & sample complexity & error control \\ 
\hline
\toprule
	\cite{bhandari2018finite} & TD & $\eta_t \asymp t^{-1}$ &$O\left(\frac{\|\bm{\Sigma}^{-2}\| \|\bm{\Sigma}\| }{(1-\gamma)^4\varepsilon^2}\right)$ & in expectation \\ \hline
	
	\cite{srikant2019finite} & TD &  $\eta_t \asymp T^{-1}$  & $O\left(\frac{\|\bm{\Sigma}^{-2}\| \|\bm{\Sigma}\| }{(1-\gamma)^4\varepsilon^2}\right)$ & in expectation\\ \hline

	\cite{dalal2018finite} & TD &  $\eta_t = t^{-1}$ &$O\left(\frac{1} {\varepsilon^{\max\{ 2, 1+\frac{1}{\lambda}\}}}\right)$,$\lambda \in (0,\lambda_{\min}(\bm{A}))$ & w. high-prob\\ \hline
	\cite{durmus2022finite} & Averaged TD & $\eta_t \asymp T^{-1/2}$ &{$O\left(\frac{\|\bm{\Sigma}^{-1}\| }{(1-\gamma)^4 \varepsilon^2} \lor \frac{\|\bm{\Sigma}^2\|\|\bm{\Sigma}^{-4}\|}{(1-\gamma)^6} \right)$} & w. high-prob\\ \hline
	\textbf{This work} & {Averaged TD }  & {$\eta_t = \eta$} &{$O\left(\frac{\|\bm{\Sigma}^{-1}\| }{(1-\gamma)^4 \varepsilon^2} \lor \frac{\|\bm{\Sigma}^2\|\|\bm{\Sigma}^{-3}\|}{(1-\gamma)^4} \right)$} & w. high-prob\\ 
	\hline
	\toprule
\end{tabular}
\caption{Comparisons with prior results (up to logarithmic terms) in finding an $\varepsilon$-optimal solution using TD learning.
Using the Polyak-Ruppert averaging, our high-probability sample complexity bound improves upon previous works in the dependence on the tolerance level $\varepsilon$ and problem-related parameters. }
\label{table-prior-TD}
\end{table*}

\subsection{Other related works}

In this section, we review several recent lines of works and provide a broader context of the current paper.

\paragraph{Finite-sample guarantees for policy evaluation} 
Classical analyses of policy evaluation algorithms have mainly focused on providing asymptotic guarantees given a fixed model
\citep{tsitsiklis1997analysis,szepesvari1998asymptotic}. 
New tools developed in high-dimensional statistics and probability allow for a fine-grained understanding of these algorithms especially from a finite-sample and finite-time perspective. 
As argued in this paper, understanding how statistical errors depend on the effective horizon, dimension of the problem and the number of samples, is essential as it provides important insights on how these RL algorithms perform in practice. 
A highly incomplete list of prior art includes \cite{khamaru2020temporal,jin2018q,srikant2019finite,bhandari2018finite,dalal2018finite,lakshminarayanan2018linear,boyan1999least} with a focus on the non-asymptotic analyses for model-free algorithms, 
and \cite{sidford2018near,agarwal2020model,pananjady2019value,li2020breaking} which derive non-asymptotic bounds for model-based algorithms.

\paragraph{Stochastic approximation}
The idea of stochastic approximation (SA) \citep{lai2003stochastic,robbins1951stochastic} lies at the core of the TD and TDC learning algorithms considered in this paper. 
With the intention of solving a deterministic fixed-point equation, SA methods perform stochastic updates based on approximations of the current residual. 
The asymptotic theory of SA methods are relatively well-developed, where SA iterates provably track the trajectory of a limiting ordinary differential equation \citep{borkar2009stochastic,borkar2000ode} and with properly decaying step sizes, the Polyak-Ruppert averaged iterates asymptotically follow the central limit theorem. 
Recently, non-asymptotic results have also been obtained for SA for different problems especially in the RL setting; see \cite{mou2020linear,moulines2011non,lakshminarayanan2018linear,nemirovski2009robust} and references therein. 
The TDC algorithm is a special case of two-timescale linear SA, whose convergence rates have also been investigated in 
\cite{dalal2020tale,xu2019two,wu2020finite,gupta2019finite}, among others.

\paragraph{Off-policy learning} Policy evaluation in the off-policy setting is closely related to offline or batch RL, which aims to learn purely based on historical data without actively exploring the environment. The main challenge here lies in the discrepancy between the behavior policy and the target or optimal policy. 
One natural approach is to use importance sampling (IS) in order to form an unbiased estimator of the target policy \citep{precup2000eligibility}, and various different techniques have been applied to reduce the high variance of IS (see, e.g.~\cite{jiang2016doubly,ma2022minimax, thomas2016data,xie2019towards,kallus2020double,yang2020off}). Non-asymptotic guarantees are also provided for off-policy evaluation using a fitted $Q$-iteration approach under linear function approximation in~\cite{duan2020minimax}. 
A recent line of works also considered finding the optimal policy using batch datasets \citep{jin2021pessimism,xie2021policy,li2022settling,shi2022pessimistic,rashidinejad2021bridging}.

\begin{table*}[t]
\centering
\renewcommand{\arraystretch}{2.3}
\begin{tabular}{c|c|c|c|c} 
\hline
\toprule
paper & algorithm  &  stepsize & sample complexity & error control \\ 
\hline
\toprule
	\cite{dalal2020tale} & \textbf{Projected} TDC & $\alpha_t = t^{-\alpha}$, $\beta_t = t^{-\beta}$ & $O\left(\frac{1}{\varepsilon^{2\alpha}}\right),\alpha < 1$ & w. high-prob \\ \hline
	\cite{kaledin2020finite} & TDC & $\alpha_t, \beta_t \asymp \frac{1}{T}$ & $O\left(\frac{1}{\varepsilon^2}\right)$ & in expectation\\ \hline
	\cite{xu2021sample} & \textbf{Batched} TDC  & $\alpha_t = \alpha$, $\beta_t = \beta$ & $O\left(\frac{1}{\varepsilon^2} \log \frac{1}{\varepsilon}\right)$ & in expectation \\ \hline
	\textbf{This work} & {TDC}  & $\alpha_t,\beta_t \asymp \frac{1}{T}$ &{$O\left(\frac{1}{\varepsilon^2}\right)$} & w. high-prob\\ 
	\hline
	\toprule
\end{tabular}
\caption{Comparisons with prior results (up to logarithmic terms) in finding an $\varepsilon$-optimal solution using TDC learning. We omit dependence on problem-related parameters in this table. Our sample complexity bound for TDC is the first to achieve high-probability convergence guarantee with non-varying stepsizes and without using projection steps or batched updates; in the mean time, we also provide explicit dependence on problem-related parameters.}
\label{table-prior-TDC}
\end{table*}

\subsection{Notation}

Throughout this paper, we denote by $\Delta(\cS)$ (resp.~$\Delta(\cA)$) the probability simplex over the finite set $\cS$ (resp.~$\cA$). For any positive integer $n$, we use $[n]$ to denote the set of positive integers that are no larger than $n$: $[n] = \{1,2,...,n\}$.
When a function is applied to a vector, it should be understood as being applied in a component-wise fashion; for example, 
$\sqrt{\bm{z}}:=[\sqrt{z_{i}}]_{1\leq i\leq n}$ and $|\bm{z}|:=[|z_{i}|]_{1\leq i\leq n}$.
For any vectors $\bm{z}=[a_{i}]_{1\leq i\leq n}$ and $\bm{w}=[w_{i}]_{1\leq i\leq n}$, the notation $\bm{z}\geq\bm{w}$ (resp.~$\bm{z}\leq\bm{w}$) stands for $z_{i}\geq w_{i}$ (resp.~$z_{i}\leq w_{i}$) for all $1\leq i\leq n$.
Additionally, we write $\bm{1}$ for the all-one vector, $\bm{I}$
for the identity matrix, and $\ind\{\cdot\}$ for the indicator function.

For any matrix $\bm{P}=[P_{ij}]$, we denote $\|\bm{P}\|_{1}:=\max_{i}\sum_{j}|P_{ij}|$.
Given a symmetric positive definite matrix $\bm{D}$, define the inner product 
$\langle\cdot,\cdot\rangle_{\bm{D}}$ as $\langle \bm{u},\bm{v}\rangle_{\bm{D}} = \bm{u}^{\top}\bm{D}\bm{v}$
and the associated norm $\|\bm{v}\|_{\bm{D}} = \sqrt{\langle \bm{v},\bm{v}\rangle_{\bm{D}}}$. 
For any matrix $\bm{M}$, we use $\|\bm{M}\|$ to denote its operator norm (i.e. the largest singular value), if not specified otherwise. 
Throughout this paper, we use $c,c_{0},c_{1},C,\cdots$ to denote universal constants that do not depend either on the parameters of the MDP or the target levels $(\varepsilon,\delta)$; their exact values may change from line to line. 
Given two sequences, $\{f_t\}_{t \geq 0}$ and $\{g_t\}_{t \geq 0}$, we write $f_t \lesssim g_t$ (resp.~$f_t \gtrsim g_t$) or $f_t = O(g_t)$ (resp.~$g_t = O(f_t)$) if there exists some universal constant $c_{1} > 0$, such that $f_t \leq c_{1} g_t$ (resp.~$f \ge c_{1}g$). If both $f = O(g)$ and $g = O(f)$ hold simultaneously, we write $f_t \asymp g_t$ or $f_t = \Theta(g_t).$ 
We adopt the notation $f = \widetilde{O}(g)$ to indicate $f = O(g)$ up to logarithmic factors in $g$. For any symmetric matrix $\bm{X}$, we use $\lambda_{\min}(\bm{X})$ to denote its smallest eigenvalue.



\section{Problem formulation}
\label{sec:models}

\subsection{Model and settings}

\paragraph{Markov decision process}

Consider an infinite-horizon MDP $\mathcal{M}=(\cS,\cA,\mathcal{P},r,\gamma)$ with discounted rewards, where $\cS$
and $\cA$ denote respectively the (finite) state space and action
space, and $\gamma\in(0,1)$ indicates the discount factor \citep{bertsekas2017dynamic}. The probability
transition kernel of the MDP is given by $\mathcal{P}:\cS\times\cA\mapsto\Delta(\cS)$,
where for each state-action pair $(s,a)\in\cS\times\cA$, $\mathcal{P}(\cdot\mid{s,a})\in\Delta(\cS)$
denotes the transition probability distribution from state $s$ when
action $a$ is executed.  The reward function is represented by  the function $r:\cS\times\cA\mapsto[0,1]$,
where $r(s,a)$ denotes the immediate reward from state $s$ when
action $a$ is taken; for simplicity, we assume throughout that all immediate 
rewards lie within $[0,1]$.


A policy $\pi:\cS\mapsto\Delta(\cA)$ is an action selection rule
that maps a state to a distribution over the set of actions; in particular,
it is said to be stationary if it is time-invariant. The value function
$V^{\pi}:\cS\mapsto\mathbb{R}$ is used to measure the quality of
a policy $\pi$, defined as 
\begin{align}
\label{eqn:value-func}
\forall s\in\cS:\qquad V^{\pi}(s)\defn\mathbb{E}\left[\sum_{t=0}^{\infty}\gamma^{t}r(s_{t},a_{t})\,\big|\,s_{0}=s\right],
\end{align}
which is the expected discounted cumulative reward received by following
the policy $\pi$ under the MDP $\mathcal{M}$ when initialized at
state $s_{0}=s$. Here, $a_{t}\sim\pi(\cdot \mid s_{t})$ and $s_{t+1}\sim \mathcal{P}(\cdot \mid s_{t},a_{t})$
for all $t\geq0$.  It can be easily verified that $0\leq V^{\pi}(s)\leq\frac{1}{1-\gamma}$
for any $\pi$.

For a given policy $\pi$, we can define the reward function of every state $s \in \mathcal{S}$ as the expected reward for $(s,a)$ when $a$ is chosen according to $\pi$:
\begin{align}
r(s) = \mathbb{E}_{a \sim \pi(\cdot|s)}[r(s,a)].
\end{align}
For simplicity, we introduce the vector notation for the reward function $\bm{r}:=[r(s)]_{1 \leq s \leq |\mathcal{S}|} \in \mathbb{R}^{|\mathcal{S}|}$, and the value function $\bm{V}^{\pi} = [V^{\pi}(s)]_{1 \leq s \leq |\mathcal{S}|} \in \mathbb{R}^{|\mathcal{S}|}$. 
\begin{comment}
\textcolor{red}{Minor notational comment: here we are implicitly assuming an ordering of the states: we could just write, for example,  $\bm{V}^{\pi} = [V^{\pi}(s)]_{s \in \mathcal{S}} \in \mathbb{R}^{\mathcal{S}}$. }\wwc{Later in the definition of $\bm{D}_{\bm{\mu}}$, $\bm{D}_{\bm{\mu}_{\mathsf{b}}}$ and $\bm{\Phi}$, we still assumed ordering of the states. Should we change them all?}
\end{comment}
We can also define the transition matrix $\bm{P}^{\pi}$ for this given policy $\pi$, such that its $(i,j)$ element represents the probability that state $i$ is transited to state $j$ under the policy $\pi$; formally,
\begin{align}
{P}^\pi_{ij}  = \sum_{a \in \mathcal{A}} \mathcal{P}(s_{t+1} = j \mid s_t = i, a_t = a) \pi(a_t = a \mid s_t = i).
\end{align} 
We denote by $\mu$ the stationary distribution corresponding to the Markov chain when the transition follows  $\bm{P}^{\pi}$, which we assume to be well-defined, and introduce the vector notation $\bm{\mu}:=[\mu(s)]_{1\leq s\leq|\mathcal{S}|}\in\mathbb{R}^{|\mathcal{S}|}$.

\paragraph{Linear approximation for the value function}
As discussed previously, it is often infeasible to collect a number of samples that scales with the ambient dimension $|\cS|$. This motivates the search for lower dimension approximation of the value function, of which linear approximation emerges as a convenient option. 
Mathematically, for  ${\bm \theta} \in \real^{d}$, define $V_{\bm\theta}(s)$ as 
\begin{align*}
\forall s\in\cS:\qquad 	V_{\bm\theta}(s) = {\bm\phi}(s)^\top {\bm\theta},  
\end{align*}
where $\bm{\phi}(s)\in\mathbb{R}^{d}$ is the feature vector associated with state $s \in \mathcal{S}$, with $d\leq|\mathcal{S}|$. The vector ${\bm\theta}$ of linear coefficients is shared across states. 
 
Using matrix notation, we let 
\begin{align}
\label{eq:defn-Phi-feature-matrix}
\bm{\Phi}
:=\left[
\bm{\phi}(1), \bm{\phi}(2), \cdots, \bm{\phi}(|\mathcal{S}|)
\right]^\top\in\mathbb{R}^{|\mathcal{S}|\times d},
\end{align}
be the feature matrix that concatenates the feature
vectors for all states and  ${\bm V}_\theta  = [V_{\bm \theta}(s) ]_{s \in \mathcal{S}} \in \real^{|S|}$ be the linear approximation vector to the value function. It follows that
\begin{align*}
	{\bm V}_\theta = \bm{\Phi} \bm{\theta}.
\end{align*}

We impose the following mild assumption on the feature vectors.
\begin{assumption}
\label{assumption:Phi}
The columns of $\bm{\Phi}$
are linearly independent with Euclidean norm uniformly bounded by one, i.e.  $\max_{s \in \mathcal{S}} \|\bm{\phi}(s)\|_{2}\leq1$.
\end{assumption}

\subsection{Policy evaluation with linear approximation}

\paragraph{On-policy evaluation with linear approximation}
The task of policy evaluation is to measure the value function $V^{\pi}(s)$ for every $s \in \cS$ (see definition~\eqref{eqn:value-func}) given a policy $\pi$ of interest. 
In the \textbf{on-policy} setting, data samples are collected while the policy $\pi$ is executed and a sequence of samples are obtained
\begin{align*}
	&\{(s_{0}, a_0,r_{0}), \ldots, (s_{T}, a_0,r_{T})\}, \\
	 &\text{where } a_t \sim \pi(\cdot \mid s_t), \; r_t = r(s_t,a_t).
\end{align*}

In this setting, in order to find the best linear approximation to $\bm{V}^{\pi}$, we find it helpful to first introduce some shorthand notation. 
First, given the stationary distribution ${\mu}$ for $\bm{P}^{\pi}$, we let 
\begin{align}\label{eq:defn-Dmu}
\bm{D}_{{\mu}}=\mathsf{diag}\big(\mu(1),\mu(2),\cdots,\mu(|\mathcal{S}|)\big)
\end{align}
\begin{comment}
\begin{align}
\bm{D}_{\mu}:=\left[\begin{array}{cccc}
\mu(1)\\
 & \mu(2)\\
 &  & \ddots\\
 &  &  & \mu(|\mathcal{S}|)
\end{array}\right]\in\mathbb{R}^{|\mathcal{S}|\times|\mathcal{S}|}.\label{eq:defn-Dmu}
\end{align}
\textcolor{red}{Why not $\mathsf{diag}\big(\mu(1),\mu(2),\cdots,\mu(|\mathcal{S}|)\big)$, like we do below?}
\end{comment}
and denote with 
\begin{align}
\bm{\Sigma}:=\bm{\Phi}^{\top}\bm{D}_{\mu}\bm{\Phi}=\mathop{\mathbb{E}}\limits _{s\sim\mu}\left[\bm{\phi}(s)\bm{\phi}(s)^{\top}\right] \in \mathbb{R}^{d \times d} \label{eq:defn-Sigma}
\end{align}
the feature covariance matrix with respect to this stationary distribution. 

The best linear approximation coefficients, $\bm{\theta}^{\star}$, is defined as the unique
solution to the following projected Bellman equation \citep{tsitsiklis1997analysis}
\begin{align}
\bm{\Phi}\bm{\theta}=\Pi_{\bm{D}_{\mu}}\mathcal{T}^{\pi}\left(\bm{\Phi}\bm{\theta}\right).\label{eq:projected-Bellman-eqn}
\end{align}
Here, $\Pi_{\bm{D}_{\mu}}$ denotes the projection operator onto the
column space of $\bm{\Phi}$ (namely, the subspace $\{\bm{\Phi}\bm{x}\mid\bm{x}\in\mathbb{R}^{d}\}$)
w.r.t.~the inner product $\langle\cdot,\cdot\rangle_{\bm{D}_{\mu}}$,
where for any vector $\bm{v}\in\mathbb{R}^{|S|}$ one has 
\[
\Pi_{\bm{D}_{\mu}}(\bm{v}):=\underset{\bm{z}\,\in\,\{\bm{\Phi}\bm{x}\,\mid\,\bm{x}\in\mathbb{R}^{d}\}}{\arg\min}\|\bm{z}-\bm{v}\|_{\bm{D}_{\mu}}^{2}.
\]
The function $\mathcal{T}^{\pi}:\mathbb{R}^{|\mathcal{S}|}\mapsto\mathbb{R}^{|\mathcal{S}|}$
is known as the {\it Bellman operator,}  which is given by
\begin{align}
\bm{v} \mapsto \mathcal{T}^{\pi}(\bm{v}):=\bm{r}+\gamma\bm{P}^\pi \bm{v}.\label{eq:Bellman-operator}
\end{align}

\paragraph{Off-policy evaluation with linear approximation}
In contrast, in the \textbf{off-policy} setting, we observe a trajectory from a behavior policy $\pi_{\mathsf{b}}$ instead of the target policy $\pi$. The goal is then to learn the value function for the target policy $\pi$ based on  
 \begin{align*}
	&\{(s_{0}, a_0,r_{0}), \ldots, (s_{T}, a_0,r_{T})\}, \\ 
	& \text{where } a_t \sim \pi_{\mathsf{b}}(\cdot \mid s_t), \; r_t = r(s_t,a_t).
\end{align*}
Let $\mu_{\mathsf{b}}$ be the stationary distribution over $\mathcal{S}$ induced by the behavior $\pi_{\mathsf{b}}$, and correspondingly
let 
\[
\bm{D}_{\mu_b}\coloneqq\mathsf{diag}\big(\mu_{\mathsf{b}}(1),\mu_{\mathsf{b}}(2),\cdots,\mu_{\mathsf{b}}(|\mathcal{S}|)\big).
\]
We denote with $\Pi_{{\bm{D}_{\mu_b}}}$ the projection operator associated with $\bm{D}_{\mu_b}$, which is given explicitly as
\[
\Pi_{{\bm{D}_{\mu_b}}}\bm{v}:=\underset{\bm{z}\,\in\,\{\bm{\Phi}\bm{x}\,\mid\,\bm{x}\in\mathbb{R}^{d}\}}{\arg\min}\|\bm{z}-\bm{v}\|_{\bm{D}_{\mu_{\mathsf{b}}}}^{2}.
\]

In the off-policy setting, instead of trying to solve the projected Bellman's equation \eqref{eq:projected-Bellman-eqn}, we aim at minimizing the Mean-Squared Projected Bellman Error (MSPBE):
\begin{equation}
\text{minimize}_{\bm{\theta}}\quad\mathsf{MSPBE}(\bm{\theta})\coloneqq\frac{1}{2}\|\bm{V}_{\bm{\theta}}-\Pi_{{\bm{D}_{\mu_b}}}\mathcal{T}^{\pi}\bm{V}_{\bm{\theta}}\|_{\bm{D}_{\mu_b}}^{2}.
\label{eq:MSPBE}
\end{equation}
Throughout, we shall denote the minimizer of the above problem~\eqref{eq:MSPBE} as $\widetilde{\bm{\theta}}^\star$. 
We remark here that the norm and the projection are both induced by $\bm{D}_{\mu_b}$, while the Bellman operator is again in terms of the target policy $\pi$.
For this reason, solving \eqref{eq:MSPBE} is different from solving the projected Bellman's equation \eqref{eq:projected-Bellman-eqn}; as a result, in general, $\bm{\theta}^\star \neq \widetilde{\bm{\theta}}^\star$. 

\section{On-policy evaluation with TD learning}

\label{sec:on-policy-evaluation} 

In this section, we study the accuracy of the estimator of $\thetastar$ (cf.~\eqref{eq:projected-Bellman-eqn}) returned by the TD learning algorithm in the on-policy setting.
Specifically, we seek to determine the tightest sample complexity for this algorithm that ensures an $\varepsilon$-close solution. To better highlight our analysis strategy, we only consider the stylized generative model\footnote{We believe that our framework can be potentially generalized to Markovian samples using similar techniques in \cite{li2021sample}. We will briefly discuss the techniques and difficulties in the following sections, but the full details are beyond the scope of the current paper.} whereby, at each time stamp $t$, one acquires an independent sample pair
\begin{align}\label{eq:iid-samples-on-policy}
	&(s_t, s'_t),  \quad \text{ where } \nonumber \\ 
	&s_t \stackrel{\text{i.i.d.}}{\sim} \mu, ~ a_t \sim \pi(s_t), ~ \text{ and }s_t' \sim \mathcal{P}(\cdot \mid s_t, a_t).
\end{align}
Here recall that $\mu$ is the stationary distribution corresponding to ${\bm P^\pi}.$ Notice that in the on-policy setting, since we are focused on a fixed policy $\pi$ and interested only in the state pairs $\{(s_t,s_t')\}_{t=0}^T$ and not the actions $\{a_t\}_{t=0}^T$, the Markov decision process reduces to a Markov reward process (MRP).
Given a sequence of sample pairs $\{(s_t, s'_t)\}_{t=0}^{T}$ and a given level of tolerance $\varepsilon > 0$, our goal is to derive a sharp lower bound on  the number of samples $T$ that is required for TD learning to produce an estimator $\widehat{\bm\theta}$ such that, with high probability, 
\begin{align*}
	\| \widehat{\bm\theta} - \thetastar \|_{\bm \Sigma} \leq \varepsilon. 
\end{align*}

\subsection{The TD learning algorithm}

\label{sec:TD-on-policy}

To motivate TD learning, it is helpful to first consider the properties of the best linear approximation coefficients $\bm{\theta}^\star$; see \eqref{eq:projected-Bellman-eqn}. For any sample transition $(s_{t},s_{t}')$ (see \eqref{eq:iid-samples-on-policy}), define the random quantities 
\begin{subequations}
\begin{align}
\bm{A}_{t} & :=\bm{\phi}(s_{t})\left(\bm{\phi}(s_{t})-\gamma\bm{\phi}(s_{t}')\right)^{\top}\in\mathbb{R}^{d\times d},\label{eq:defn-At}\\
\bm{b}_{t} & :=\bm{\phi}(s_{t})r(s_{t})\in\mathbb{R}^{d},\label{eq:defn-bt}
\end{align}
\end{subequations}whose means are given respectively by 
\begin{subequations} 
\begin{align}
\bm{A} & :=\mathop{\mathbb{E}}\limits _{s\sim\mu,s'\sim P^\pi(\cdot\mid s)}\left[\bm{\phi}(s)\left(\bm{\phi}(s)-\gamma\bm{\phi}(s')\right)^{\top}\right] \\ 
&~=\bm{\Phi}^{\top}\bm{D}_{\mu}(\Ind-\gamma\bm{P}^\pi)\bm{\Phi}\,\in\mathbb{R}^{d\times d},\label{eq:defn-At-mean}\\
\bm{b} & :=\mathop{\mathbb{E}}\limits _{s\sim\mu}\left[\bm{\phi}(s)r(s)\right]=\bm{\Phi}^{\top}\bm{D}_{\mu}\bm{r}\,\in\mathbb{R}^{d}.\label{eq:defn-bt-mean}
\end{align}
\end{subequations}

It turns out that the target vector $\thetastar$ satisfies the equation~\citep{tsitsiklis1997analysis}  
\begin{align}
\label{eq:defn-theta-star}
\bm{\theta}^{\star}:=\bm{A}^{-1}\bm{b}. 
\end{align}
The TD learning algorithm leverages this representation by iteratively improving the linear approximation of the value function at each time stamp through the updates
\begin{subequations}
\label{eq:TD-update-all} 
\begin{align}
\bm{\theta}_{t+1} & =\bm{\theta}_{t}-\eta_t(\bm{A}_{t}\bm{\theta}_{t}-\bm{b}_{t}), \quad t =0,1,2,\ldots,
\label{eq:TD-update-rule}
\end{align}
where, for each $t$, $\eta_t>0$ denotes the learning rate or stepsize.
After $T$ iterations, the TD learning algorithm returns $\bm{\theta}_T$ as the estimator. In contrast, TD learning with Polyak-Ruppert averaging, or \emph{averaged TD learning} in short, returns an average across all iterates
\begin{align}
	\thetaavg{T} & =\frac{1}{T}\sum_{i=1}^{T}\bm{\theta}_{i}. \label{eq:TD-averaging}
\end{align}
\end{subequations} 
While we are mainly concerned with the averaged estimator $\thetaavg{T}$,  we also obtain some theoretical properties of $\bm{\theta}_T$ as a by-product of our analysis. 

\subsection{Sample complexity of TD learning}

In this section, we present a finite-sample bound for the estimation error of $\thetaavg{T}$ assuming independent data, from which we derive a novel sample complexity guarantee for TD learning. 
Below, we denote by $\kappa$ the condition number of $\bm{\Sigma}$ as follows 
\begin{align}
\kappa\defn\lambdamax/\lambdamin \geq 1.
\label{defn:kappa-Sigma}
\end{align}

\begin{theorem} \label{thm:ind-td} 
There exist universal, positive constants $C_{0},c_{0}>0$ and $c_{1}>0$ , such that for any given $0<\delta<1$, the averaged TD learning estimator $\thetaavg{T}$  \eqref{eq:TD-update-all} after $T$ iterations satisfies the bound
\begin{align}
&\big\|\overline{\bm{\theta}}_{T}-\bm{\theta}^{\star}\big\|_{\bm{\Sigma}} \\
 &\leq C_{0}\Bigg\{\sqrt{\frac{\max_{s}\bm{\phi}(s)^{\top}\bm{\Sigma}^{-1}\bm{\phi}(s)\log(\frac{d}{\delta})}{T(1-\gamma)^{2}}}\big(\|\bm{\theta}^{\star}\|_{\bm{\Sigma}}+1\big)\nonumber \\
 & +\frac{\|\bm{\Sigma}^{-1}\|\Big[\big(\|\bm{\theta}^{\star}\|_{\bm{\Sigma}}+1\big)\sqrt{\frac{\kappa\log(\frac{dT}{\delta})}{\eta(1-\gamma)^{3}}}+\frac{1}{\eta(1-\gamma)}\|\bm{\theta}^\star\|_{\bm{\Sigma}}\Big]}{T}\Bigg\}
\end{align}
with probability at least $1-\delta$, provided that $\bm{\theta}_0 =\bm{0}$, $\eta_0 = \ldots = \eta_T = \eta <\frac{c_{0}(1-\gamma)}{\kappa\log(Td/\delta)}$ 
and 
\[
T\geq\frac{c_{1}\kappa(\|\bm{\theta^\star}\|_{\bm{\Sigma}}+1)^{2}\log^{2}\frac{\kappa dT(\|\bm{\theta}^\star\|_{2}+1)}{(1-\gamma)\delta}}{\eta(1-\gamma)\lambda_{\mathsf{min}}(\bm{\Sigma})}.
\] 

\end{theorem}

{\paragraph{Proof sketch and technical novelty} An essential step in our proof of Theorem \ref{thm:ind-td} is to guarantee with high probability that the estimation error of the original TD estimator, $\bm{\Delta}_t$, is bounded by a time-invariant value. Towards this end, we combine the matrix Freedman's inequality with an induction argument. For the technical details, we refer the readers to Steps 2 and 3 in Section 6. 
		Controlling the norm of $\bm{\Delta}_t$ with high probability in a time-invariant manner, 
		paves way for bounding the norm of $\overline{\bm{\Delta}}_T$ with high probability without the need of  performing another projection step to restrict the norm of $\bm{\Delta}_t$ during the TD learning iterations.

\paragraph{Generalize to Markov samples} We give a brief introduction of how our results can be extended to Markov samples. The main difficulty towards this end lies in bounding the sequence of temporal difference errors
\begin{align*}
\{\bm{A}_t \bm{\theta}_t - \bm{b}_t\}_{t \geq 0}
\end{align*}
which is no longer a Martingale difference process, so the Freedman's inequality is not directly applicable anymore. In order to tackle this problem, the popular strategy is  to assume that the Markov chain mixes at a geometric rate. Specifically, for arbitrarily small $\varepsilon>0$, there exists a positive integer $t_{\mathsf{mix}}(\varepsilon) \asymp \log \left(\frac{1}{\varepsilon}\right)$, such that for any $t \geq t_{\mathsf{mix}}$,
\begin{align}\label{eq:mixing}
\left\|\mathbb{E}_{t-t_{\mathsf{mix}}}[\bm{A}_t] - \bm{A}\right\| < \varepsilon, \quad \text{and} \quad \left\|\mathbb{E}_{t-t_{\mathsf{mix}}}[\bm{b}_t] - \bm{b}\right\|_2 < \varepsilon.
\end{align}
Here, $\mathbb{E}_{i}[\cdot]$ represents the expectation conditioned
on the filtration $\mathcal{F}_{i}$ --- the $\sigma$-algebra generated
by $\{(s_{j},s_{j}')\}_{j\leq i}$. Under this assumption, the temporal difference error can be decomposed as
\begin{align*}
\bm{A}_t \bm{\theta}_t - \bm{b}_t &= \bm{A}_t (\bm{\theta}_t - \bm{\theta}_{t-t_{\mathsf{mix}}}) \\ 
&+ \left[\left(\bm{A}_t - \mathbb{E}_{t-t_{\mathsf{mix}}}[\bm{A}_t]\right)\bm{\theta}_{t-t_{\mathsf{mix}}} - (\bm{b}_t - \mathbb{E}_{t-t_{\mathsf{mix}}}[\bm{b}_t] )\right] \\ 
&+ \left[\left(\mathbb{E}_{t-t_{\mathsf{mix}}}[\bm{A}_t] - \bm{A}\right)\bm{\theta}_{t-t_{\mathsf{mix}}}-\left(\mathbb{E}_{t-t_{\mathsf{mix}}}[\bm{b}_t] - \bm{b}\right)\right].
\end{align*}
On the right hand side, the last term can be bounded by the mixing property \eqref{eq:mixing}; the first term can be further expanded as
\begin{align*}
\bm{\theta}_t - \bm{\theta}_{t-t_{\mathsf{mix}}} &= \sum_{j=t-t_{\mathsf{mix}}}^{t-1} (\bm{\theta}_{j+1} - \bm{\theta}_j)\\ 
&= \sum_{j=t-t_{\mathsf{mix}}}^{t-1} \eta (\bm{A}_t \bm{\theta}_j - \bm{b}_j),
\end{align*}
and bounded in terms of the step size $\eta$ and the mixing time $t_{\mathsf{mix}}$; the second term can be controlled by separating the sequence
\begin{align*}
\left\{\left(\bm{A}_t - \mathbb{E}_{t-t_{\mathsf{mix}}}[\bm{A}_t]\right)\bm{\theta}_{t-t_{\mathsf{mix}}} - (\bm{b}_t - \mathbb{E}_{t-t_{\mathsf{mix}}}[\bm{b}_t] )\right\}_{t \geq 0}
\end{align*}
into $t_{\mathsf{mix}}$ disjoint Martingale difference processes and invoking the matrix Freedman's inequality. We leave the details to our future work.

Theorem \ref{thm:ind-td} directly implies the following corollary, which gives an upper bound for the sample complexity of TD learning with independent samples.

\begin{corollary}[Sample complexity of TD learning]
\label{cor:sample-complexity-TD-iid}
There exists a universal constant $c >0$ such that, for any $\varepsilon \in \left(0,\|\bm{\theta}^{\star}\|_{\bm{\Sigma}} \right)$  and $\delta \in (0,1)$, the averaged TD
estimator~\eqref{eq:TD-averaging} achieves 
\begin{align}
\label{eq:TD-converge}
\big\|\bm{V}_{\bar{\bm{\theta}}_T}-\bm{V}_{\bm{\theta}^{\star}}\big\|_{\bm{D}_{\mu}}=\big\|\bar{\bm{\theta}}_T-\bm{\theta}^{\star}\big\|_{\bm{\Sigma}}\leq\varepsilon
\end{align}
with probability exceeding $1-\delta$, provided that 
\begin{align}
T\geq\frac{c\big\{\max_{s}\bm{\phi}(s)^{\top}\bm{\Sigma}^{-1}\bm{\phi}(s)\big\}\big(1+\big\|\bm{\theta}^{\star}\big\|_{\bm{\Sigma}}^{2}\big)\log\big(\frac{d}{\delta}\big)}{(1-\gamma)^{2}\varepsilon^{2}}.
\label{eq:sample-complexity-TD-bound}
\end{align}
\end{corollary}

\paragraph{Comparisons to prior literature}  We remark that the best finite-sample results for TD learning obtained so far are given by \cite[Theorem 2(c)]{bhandari2018finite} and \cite[Corollary 1]{srikant2019finite}, with decaying stepsizes $\eta_t \asymp t^{-1}$ and sample size-related stepsizes $\eta_t \asymp T^{-1}$ respectively. Translated into our notation, they both prove that in order for the \emph{expected} estimation error to be controlled by $\varepsilon$, namely
\begin{align*}
\mathbb{E}\left\|\bm{\theta}_T - \bm{\theta}^\star\right\|_{\bm{\Sigma}}^2 \leq \varepsilon^2,
\end{align*}
it suffices to take (up to some logarithmic factors)
\begin{align}\label{eq:prior-TD}
T^{\textsf{prior}} \asymp \frac{\kappa \|\bm{\Sigma}^{-1}\|\left(\|\bm{\theta}^{\star}\|_{\bm{\Sigma}}^2+1\right)}{(1-\gamma)^2 } \frac{1}{\varepsilon^2} .
\end{align}
We refer readers to Appendix \ref{app:compare-Srikant} and \ref{app:compare-Russo} for a detailed translation of their results. Comparing \eqref{eq:sample-complexity-TD-bound} and \eqref{eq:prior-TD}, our result improves upon previous works by a multiplicative factor of
\begin{align*}
\frac{T^{\textsf{prior}}}{T^{\textsf{ours}}} = \kappa,
\end{align*}
the condition number of $\bm{\Sigma}$; $\kappa$ can be as large as $d$, the dimension of the features, which can scale with $|\mathcal{S}|.$

As for sample complexity with high-probability convergence guarantees, another recent work~\cite{dalal2018finite} shows that in order for \eqref{eq:TD-converge} to hold with probability at least $1-\delta$, it suffices to take
\begin{align}\label{eq:Dalal-TD}
T \asymp \max\Bigg\{&\left(\frac{1}{\varepsilon}\right)^2 \left(\log \frac{1}{\delta}\right)^3,  \\ 
&\left(\frac{1}{\varepsilon}\right)^{1+1/\lambda_{\min}(\bm{A})} \left(\log \frac{1}{\delta}\right)^{1+1/\lambda_{\min}(\bm{A})}\Bigg\}.
\end{align}
Comparing \eqref{eq:sample-complexity-TD-bound} and \eqref{eq:Dalal-TD}, we can see that our result improves on the dependence of both the error tolerance $\varepsilon$ and the probability tolerance $\delta$; in fact, our result is the first sample complexity for TD learning with high-probability convergence guarantee that matches the minimax-optimal dependence of $\varepsilon$ and displays a clear dependence on the problem-related parameters, as would be shown in the following section.

After the initial post of the current paper, we are pointed to the work~\cite{durmus2022finite}, which provides a general treatment of linear stochastic approximation with Polyak-Ruppert averaging. Their results lead to the same sample complexity as Corollary~\ref{cor:sample-complexity-TD-iid} with a slightly higher burn-in cost. We include the detailed comparisons of their result in Section~\ref{sec:durmus}. We also point out the works of \cite{mou2020optimal} and \cite{li2021accelerated}, which derived similar results regarding the error bound for averaged TD learning in expectation. Our result, as shown in Theorem 1, improves upon theirs in the sense that we provide high probability guarantees and offer explicit dependencies on problem-related parameters. Detailed comparisons can be found in Section D.3.

\subsection{Minimax lower bounds}
\label{sec:minimax-lb}

To assess the tightness of our upper bounds in Corollary \ref{cor:sample-complexity-TD-iid}, in this section, we provide a minimax lower bound for the value function estimation problem with linear approximation. More specifically, the question we intend to answer is: 
for any target accuracy level $\varepsilon$, do there exist estimators that achieve an $\varepsilon$-approximation of $\bm{V}_{\thetastar}$ with fewer samples?
As shown in the following result, the answer is, by and large, negative.

\begin{theorem}[Minimax lower bound]
\label{thm:minimax}
Consider any $\frac{1}{2} < \gamma < 1$, $1 < d \le |\cS|$, and $0<\varepsilon<c_1\max\{1,\|\bm{\theta}^{\star}\|_{\bm{\Sigma}}\}$ for some universal constant $c_1 > 0$. 
There exist universal constants $c_2, c_3 >0$ such that for any estimator $\hat{\bm{\theta}}$ based on $T$ independent pairs $\{ (s_t,s'_t)\}_{t=1}^T$ as in \eqref{eq:iid-samples-on-policy}, there exists a Markov reward process and a choice of the feature matrix $\bm{\Phi}$ such that 
\begin{align}\label{eq:minimax}
	\mathbb{P}\Big\{\big\|\widehat{\bm{\theta}}-\bm{\theta}^{\star}\big\|_{\bm{\Sigma}}>c_2\varepsilon\Big\} 
	\ge \frac{1}{4},
\end{align}
provided that the number of samples $T$ satisfies 
\begin{align}
T\leq\frac{c_{3}\big\{\max_{s}\bm{\phi}(s)^{\top}\bm{\Sigma}^{-1}\bm{\phi}(s)\big\}\big(1+\big\|\bm{\theta}^{\star}\big\|_{\bm{\Sigma}}^{2}\big)}{(1-\gamma)\varepsilon^{2}}.\label{eq:sample-complexity-minimax}
\end{align}
\end{theorem}

\begin{remark}
We remark that minimax lower bounds are also previously investigated in a general framework in \citep{duan2021optimal} where the value function is approximated using a general reproducing kernel Hilbert space (RKHS). When it comes to linear function approximation, for completeness, we include in Section~\ref{sec:pf-lb} a different but simpler construction tailored to the linear space. 
Compared to the results of \cite{duan2021optimal}, our lower bound is stated in terms of different parameters, which allows us to evaluate the tightness of Corollary~\ref{cor:sample-complexity-TD-iid} directly. Instantiating both lower bounds, they do agree and equal to 
\begin{align}
\label{eq:minimax-agree}
O\left(\frac{d}{\varepsilon^2 (1-\gamma)^3}\right),
\end{align}
as one plugs in the exact parameters from our construction. 
\end{remark}

As asserted by this theorem, no algorithm whatsoever can attain an $\varepsilon$-approximation of the best linear coefficient --- in a minimax sense --- unless the total sample size exceeds
\begin{align*}
	O\Bigg(\frac{\big\{\max_{s}\bm{\phi}(s)^{\top}\bm{\Sigma}^{-1}\bm{\phi}(s)\big\}\big(1+\big\|\bm{\theta}^{\star}\big\|_{\bm{\Sigma}}^{2}\big)}{(1-\gamma)\varepsilon^{2}}\Bigg). 
\end{align*}
Consequently, the upper bounds developed in Corollary \ref{cor:sample-complexity-TD-iid} are sharp in terms of the accuracy level $\varepsilon$, the dependence of the feature map $\bm{\Phi}$, the underlying coefficient $\thetastar$, and the covariance matrix $\bm{\Sigma}$. 
Therefore, it implies that the performances of the TD learning algorithms can not be further improved in the minimax sense other than a factor of $\frac{1}{1-\gamma}$---the effective horizon. 

We believe that the gap in terms of $\frac{1}{1-\gamma}$ mainly comes from the function approximation paradigm.
In fact, as far as we know, with linear function approximation, there has been no minimax optimal results established for this problem either for the model based method (e.g. LSTD) or for the variance-reduced approach, both of which are known to be minimax optimal in the tabular setting; the latter is 
also proved to be instance-optimal from \cite{khamaru2020temporal} and \cite{li2021accelerated}. 
We conjecture the minimax optimal dependency of $\frac{1}{1-\gamma}$ to be the same as that of the tabular setting and TD with LFA to be minimax optimal. 
Establishing this result, however, requires developing completely new analysis tools, particularly in dealing with the structure of variance across different steps, which we leave as an interesting open direction.

\section{Off-policy evaluation with TDC learning}
\label{sec:TDC}

In this section, we aim to estimate the optimizer $\widetilde{\bm{\theta}}^{\star}$ of the optimization problem~\eqref{eq:MSPBE} in the off-policy setting by means of the TDC algorithm.
We continue to focus on the case when samples are generated in the i.i.d. fashion by the behavior policy $\pi_b$. At each time stamp $t$, one obtains
\begin{align}\label{eq:iid-samples-off-policy}
	&(s_t, a_t, s'_t),\quad \text{ where } \nonumber \\ 
	& s_t \stackrel{\text{i.i.d.}}{\sim} \mu_b, ~ a_t \sim \pi_{\mathsf{b}}(\cdot \mid s_t), ~ \text{ and } s_t' \sim \mathcal{P}(\cdot \mid s_t, a_t).
\end{align}
Here, recall that $\mu_{\mathsf{b}}$ is the stationary distribution corresponding to the behavior policy $\pi_{\mathsf{b}}$. We first provide some intuition behind the TDC algorithm before describing novel bounds on its sample complexity for obtaining an $\varepsilon$-accurate solution.

\subsection{The TDC algorithm}

The TDC algorithm is designed to solve the optimization problem~\eqref{eq:MSPBE} using a two-timescale linear TD with gradient correction \citep{sutton2009fast}. 
To provide some high-level ideas behind the design of this algorithm, it is helpful to rewrite the objective function in the following form by directly expanding the terms in expression~\eqref{eq:MSPBE}. 
\begin{claim}
\label{clm:MSPBE}
The quantity $\mathsf{MSPBE}(\bm{\theta})$ can be equivalently written as
\begin{align}
\mathsf{MSPBE}(\bm{\theta})&=\frac{1}{2}\mathbb{E}_{\mu_{\mathsf{b}},\pi,\mathcal{P}}\left[\bm{\phi}(s_t)\delta_{t}\right]^{\top}\nonumber \\ 
&\left\{ \mathbb{E}_{\mu_{\mathsf{b}}}\left[\bm{\phi}(s_t)\bm{\phi}(s_t)^{\top}\right]\right\} ^{-1} 
\mathbb{E}_{\mu_{\mathsf{b}},\pi,\mathcal{P}}\left[\bm{\phi}(s_t)\delta_{t}\right],\label{eq:MSPBE_new}
\end{align}
where $\delta_{t} \defn r_{t}+\gamma\bm{\phi}(s_t')^{\top}\bm{\theta}-\bm{\phi}(s_t)^{\top}\bm{\theta}$ is the temporal difference error. 
\end{claim}
%
In light of the above expression, the gradient of $\mathsf{MSPBE}(\bm{\theta})$ with respect to $\bm{\theta}$ equals to 
\begin{align}
\label{eq:MSPBE-gradient}
&\nabla_{\bm{\theta}}\mathsf{MSPBE}(\bm{\theta}) \nonumber \\ 
 & =\mathbb{E}_{\mu_{\mathsf{b}},\pi,\mathcal{P}}\left[\left(\gamma\bm{\phi}(s_t')-\bm{\phi}(s_{t})\right)\bm{\phi}(s_t)^{\top}\right] \nonumber \\ 
 &\left\{ \mathbb{E}_{\mu_{\mathsf{b}}}\left[\bm{\phi}(s_t)\bm{\phi}(s_t)^{\top}\right]\right\} ^{-1}\mathbb{E}_{\mu_{\mathsf{b}},\pi,\mathcal{P}}\left[\bm{\phi}(s_t)\delta_{t}\right]\nonumber \\
& =-\mathbb{E}_{\mu_{\mathsf{b}},\pi,\mathcal{P}}\left[\bm{\phi}(s_t)\delta_{t}\right]+\gamma\mathbb{E}_{\mu_{\mathsf{b}},\pi,\mathcal{P}}\left[\bm{\phi}(s_{t}')\bm{\phi}(s_t)^{\top}\right] \nonumber \\ 
&\left\{ \mathbb{E}_{\mu_{b}}\left[\bm{\phi}(s_t)\bm{\phi}(s_t)^{\top}\right]\right\} ^{-1}\mathbb{E}_{\mu_{\mathsf{b}},\pi,\mathcal{P}}\left[\bm{\phi}(s_t)\delta_{t}\right]\nonumber \\
& =-\mathbb{E}_{\mu_{\mathsf{b}},\pi_{\mathsf{b}},\mathcal{P}}\left[\rho_t\bm{\phi}(s_t))\delta_{t}\right]+\gamma\mathbb{E}_{\mu_{\mathsf{b}},\pi_{\mathsf{b}},\mathcal{P}}\left[\rho_t\bm{\phi}(s_t')\bm{\phi}(s_t)^{\top}\right] \bm{w}_t,
\end{align}
where in the last step we have defined
\begin{align}
\label{eq:defn-w}
\bm{w}_t &= \bm{w}(\bm{\theta}_t) \nonumber \\ 
&= \left\{ \mathbb{E}_{\mu_{\mathsf{b}}}\left[\bm{\phi}(s_t)\bm{\phi}(s_t)^{\top}\right]\right\} ^{-1}\mathbb{E}_{\mu_{\mathsf{b}},\pi_{\mathsf{b}},\mathcal{P}}\left[\rho_t\bm{\phi}(s_t)\delta_{t}\right].
\end{align} 
and have used the importance weights
\begin{align}
	\rho_t\coloneqq\frac{\pi(a_{t}|s_{t})}{\pi_{\mathsf{b}}(a_{t}|s_{t})}
\end{align}
 to replace the expectation w.r.t.~$\pi$ with the expectation w.r.t. $\pi_{\mathsf{b}}$.

The high-level idea of TDC is to estimate the right hand side of \eqref{eq:MSPBE-gradient} based on the sample trajectory~\eqref{eq:iid-samples-off-policy}, and then perform stochastic gradient updates for $\bm{\theta}_t.$
However, the challenge is that the second term in the gradient of MSPBE \eqref{eq:MSPBE-gradient} involves the product of two expectations. Simultaneously sampling and using the sample product is inappropriate due to their correlation. 
In order to address this issue, \cite{sutton2008convergent} and \cite{sutton2009fast} introduced an auxiliary parameter $\bm{w}$ to estimate $\bm{w}(\bm{\theta}_t)$ by solving a linear stochastic approximation (SA) problem corresponding to the linear system
\begin{align}
\label{eqn:ode}
\mathbb{E}_{\mu_{\mathsf{b}}}\left[\bm{\phi}(s_t)\bm{\phi}(s_t)^{\top}\right] \bm{w} = \mathbb{E}_{\mu_{\mathsf{b}},\pi_{\mathsf{b}},\mathcal{P}}\left[\rho_t\bm{\phi}(s_t)\delta_{t}\right].
\end{align}
Putting these ideas together, TDC amounts to the following two-timescale linear stochastic  method
\begin{align*}
\widetilde{\bm{\theta}}_{t+1} 
& =\widetilde{\bm{\theta}}_{t}-\alpha_t[\gamma\rho_{t}\bm{\phi}(s_t')\bm{\phi}(s_t)^{\top}\bm{w}_{t} - \rho_{t}\delta_{t}\bm{\phi}(s_t)];\\
\bm{w}_{t+1} & =\bm{w}_{t}-\beta_t \left[\bm{\phi}(s_t)\bm{\phi}(s_t)^{\top}\bm{w}_{t}-\rho_{t}\delta_{t}\bm{\phi}(s_t)\right].
\end{align*}
Here, the update of $\widetilde{\bm{\theta}}_{t}$ corresponds to a gradient step regarding~\eqref{eq:MSPBE_new}, the update of $\bm{w}_{t}$ corresponds to linear SA for solving \eqref{eqn:ode},
and $\delta_{t} \defn r_{t}+\gamma\bm{\phi}(s_t')^{\top}\widetilde{\bm{\theta}}_t-\bm{\phi}(s_t)^{\top}\widetilde{\bm{\theta}}_t$
is the temporal difference error. In addition, $\alpha_t$, $\beta_t$ are the corresponding stepsizes.
For notational convenience, let us denote 
\begin{equation}
\label{eqn:TDC-emp-para}
\begin{aligned}
&\widetilde{\bm{A}}_{t}=\rho_{t}\bm{\phi}(s_t)\left(\bm{\phi}(s_t)-\gamma\bm{\phi}(s_t')\right)^{\top},\\ 
&\widetilde{\bm{b}}_{t}\coloneqq\rho_{t}\bm{\phi}(s_t)r_{t},\\
&\bm{\Pi}_{t}\coloneqq\rho_t\bm{\phi}(s_t)\bm{\phi}(s_t')^{\top},\\ 
&\widetilde{\bm{\Sigma}}_{t}\coloneqq\bm{\phi}(s_t)\bm{\phi}(s_t)^{\top}.
\end{aligned}
\end{equation}
With these definitions, the TDC iterates can be written compactly as 
\begin{subequations}
\label{subeq:TDC}
\begin{align}
\widetilde{\bm{\theta}}_{t+1} & =\widetilde{\bm{\theta}}_{t}-\alpha_t(\widetilde{\bm{A}}_{t}\widetilde{\bm{\theta}}_{t}-\widetilde{\bm{b}}_{t}+\gamma\bm{\Pi}_{t}^{\top}\bm{w}_{t});\\
\bm{w}_{t+1} & =\bm{w}_{t}-\beta_t(\widetilde{\bm{A}}_{t}\widetilde{\bm{\theta}}_{t}-\widetilde{\bm{b}}_{t}+\widetilde{\bm{\Sigma}}_{t}\bm{w}_{t}).
\end{align}
\end{subequations}

\subsection{Sample complexity of TDC}

Our finite-sample characterization of TDC builds upon a careful analysis of the population dynamics of TDC, which we then show to be uniformly well approximated by the empirical dynamics of TDC via matrix concentration inequalities. Before stating our main result, we find it helpful to introduce some extra pieces of notation. Specifically, define the population parameters as 
\begin{subequations}
\begin{align}
\widetilde{\bm{A}} & \coloneqq\mathbb{E}_{\mu_{\mathsf{b}},\pi_{\mathsf{b}},\mathcal{P}}[\widetilde{\bm{A}}_{t}]=\mathbb{E}_{\mu_{\mathsf{b}},\pi_{\mathsf{b}},\mathcal{P}}[\rho_{t}\bm{\phi}(s_t)\left(\bm{\phi}(s_t)-\gamma\bm{\phi}(s_t')\right)^{\top}];\\
\widetilde{\bm{b}} & \coloneqq\mathbb{E}_{\mu_{\mathsf{b}}}[\widetilde{\bm{b}}_{t}]=\mathbb{E}_{\mu_{\mathsf{b}},\pi_{\mathsf{b}}}[\rho_{t}\bm{\phi}(s_t)r_{t}];\\
\bm{\Pi} & \coloneqq\mathbb{E}_{\mu_{\mathsf{b}},\pi_{\mathsf{b}},\mathcal{P}}[\bm{\Pi}_{t}]=\mathbb{E}_{\mu_{\mathsf{b}},\pi_{\mathsf{b}},\mathcal{P}}[\rho_{t}\bm{\phi}(s_t)\bm{\phi}(s_t')^{\top}];\\
\widetilde{\bm{\Sigma}} & \coloneqq\mathbb{E}_{\mu_{\mathsf{b}}}[\widetilde{\bm{\Sigma}}_{t}]=\mathbb{E}_{\mu_{\mathsf{b}}}[\bm{\phi}(s_t)\bm{\phi}(s_t)^{\top}].
\end{align}
\end{subequations}
In addition, denote the parameters  
\begin{equation}
\begin{aligned}
&\lambda_1 = \lambda_{\min}(\widetilde{\bm{A}}^{\top}\widetilde{\bm{\Sigma}}^{-1}\widetilde{\bm{A}}), \\
&\lambda_2 = \lambda_{\min}(\widetilde{\bm{\Sigma}}), \\
&\lambda_{\Sigma} = \|\widetilde{\bm{\Sigma}}^{-1}\| = 1/\lambda_2,  \\
&\widetilde{\kappa} = \lambda_{\Sigma} \cdot \|\widetilde{\bm{\Sigma}}\|, \\
&\rho_{\max} = \max_{s,a} [\pi(a | s)/\pi_{\mathsf{b}}(a | s)].
\end{aligned}
\end{equation}
With these notation in place, we are ready to state our main result for TDC learning, with its proof deferred to Section~\ref{sec:pf-TDC}.

\begin{theorem}
\label{thm:TDC-error}
There exist universal constants $\widetilde{C}_0,\widetilde{c}_1 > 0$, such that for any given $0 \leq \delta \leq 1$, the output $\widetilde{\bm{\theta}}_T$ of the TDC learning iterate \eqref{subeq:TDC} at time $T$ satisfies the bound
\begin{align}
\label{eqn:sigma-error-tdc}
\|\widetilde{\bm{\theta}}_T - \widetilde{\bm{\theta}}^\star\|_{\widetilde{\bm{\Sigma}}} \leq \widetilde{C}_0 \frac{\rho_{\max}^2\|\widetilde{\bm{\Sigma}}\|}{\lambda_1} \sqrt{\frac{\beta}{\lambda_2}\log \frac{2dT}{\delta}}(\|\widetilde{\bm{\theta}}^\star\|_{\widetilde{\bm{\Sigma}}} + 2),
\end{align}
with probability at least $1-\delta$, provided that 
\begin{equation}
\begin{aligned}
\label{eq:TDC-conditions}
&\widetilde{\bm{\theta}}_0 = \bm{0}, \\ 
&\alpha_0 = \ldots = \alpha_T = \alpha,\\ 
&\beta_0 = \ldots = \beta_T = \beta,\\ 
&0 < \alpha < \frac{1}{\lambda_1\lambda_{\Sigma}^2 \|\widetilde{\bm{\Sigma}}\|\log \frac{2dT}{\delta}},\\  
&\frac{\alpha}{\beta} = \frac{1}{128} \frac{\lambda_1 \lambda_2}{\rho_{\max}^2 (1+\lambda_{\Sigma}\rho_{\max})},\\
&T \geq \widetilde{c}_1 \frac{\log \|\widetilde{\bm{\theta}}^\star\|_2}{\alpha \lambda_1} 
\log \max \left\{\sqrt{\widetilde{\kappa}},~\|\widetilde{\bm{\theta}}^\star\|_{\widetilde{\bm{\Sigma}}} \sqrt{\frac{\alpha \lambda_1}{\log \frac{2dT}{\delta}}}\right\}.
\end{aligned}
\end{equation}
\end{theorem}

\begin{remark}
A similar result in terms of the $\ell_{2}$ error (namely, $\|\widetilde{\bm{\theta}}_T - \widetilde{\bm{\theta}}^\star\|_2$) can be derived in the same way as in \eqref{eqn:sigma-error-tdc}.
In particular, under the same conditions as in \eqref{eq:TDC-conditions}, it can be derived with probability at least $1-\delta$ that
\begin{align}
\|\widetilde{\bm{\theta}}_T - \widetilde{\bm{\theta}}^\star\|_2 \lesssim \widetilde{C}_0 \frac{\rho_{\max}^2}{\lambda_1} \sqrt{\frac{\beta}{\lambda_2}\log \frac{2dT}{\delta}}(\|\widetilde{\bm{\theta}}^\star\|_2 + 2).
\end{align}
Since the proof follows in the similar fashion, we omit here for brevity. 

\end{remark}

\paragraph{Proof sketch and technical novelty} Our proof of Theorem \ref{thm:TDC-error} considers the convergence of the vector
\begin{align*}
\bm{x}_{t}\coloneqq\left[\begin{array}{c}
\tilde{\bm{\theta}}_{t} - \tilde{\bm{\theta}}^\star\\
\varkappa\left[\tilde{\bm{w}}_t + \tilde{\bm{\Sigma}}^{-1}\tilde{\bm{A}}(\tilde{\bm{\theta}}_{t} - \tilde{\bm{\theta}}^\star)\right]
\end{array}\right],
\end{align*}
where $\varkappa \in (0,1)$ is a constant to be specified. Firstly, we identify the conditions on $\alpha, \beta$ and $\varkappa$ that guarantee the exponential convergence of $\bm{x}_t$ in the noise-free scenario; after this, we again combine the matrix Freedman's inequality and an induction argument to bound the norm of $\bm{x}_t$ for $i.i.d.$ samples with high probability by a time-invariant value in terms of the step size. And finally, with a careful choice of $\alpha$, $\beta$ and $\varkappa$, we establish the finite-sample guarantee as is shown in Theorem \ref{thm:TDC-error}. The main technical novelty of this proof lies in the construction of the vector $\bm{x}_t$ and the choice of the paramter $\varkappa$.

Next, we state a direct consequence of Theorem \ref{thm:TDC-error} below, which gives an upper bound for the sample complexity of TDC.
\begin{corollary}\label{cor:TDC-sample-complexity}
There exists a universal constant $\widetilde{c}$ such that, for any $\delta \in (0,1)$ and $\varepsilon \in (0,\|\widetilde{\bm{\theta}}^\star\|_{\widetilde{\bm{\Sigma}}})$, the TDC estimator $\widetilde{\bm{\theta}}_T$ at iterate $T$ satisfies the bound  
\begin{align}
\|\bm{V}_{\widetilde{\bm{\theta}}_T} - \bm{V}_{\widetilde{\bm{\theta}}^\star}\|_{\bm{D}_{\mu_b}} = \|\widetilde{\bm{\theta}}_T - \widetilde{\bm{\theta}}^\star\|_{\widetilde{\bm{\Sigma}}} \leq \varepsilon
\end{align}
with probability exceeding $1-\delta$, provided that 
\begin{align}
T \geq \widetilde{c} \frac{\rho_{\max}^7}{\lambda_1^4 \lambda_2^3 }\frac{\|\widetilde{\bm{\Sigma}}\|^2}{\varepsilon^2} (1+\|\widetilde{\bm{\theta}}^\star\|_{\widetilde{\bm{\Sigma}}}^2) \log \Big(\frac{d\|\widetilde{\bm{\theta}}^\star\|_{\widetilde{\bm{\Sigma}}}}{\delta}\Big),
\end{align}
and the stepsize parameters $\alpha_t$ and $\beta_t$ are chosen as 
\begin{align}
\alpha_t \asymp \frac{\log \|\widetilde{\bm{\theta}}^\star\|_{\widetilde{\bm{\Sigma}}}}{T\lambda_1}, \qquad
\beta_t = 128 \frac{\rho_{\max}^2 (1+\lambda_{\Sigma}\rho_{\max})}{\lambda_1 \lambda_2}\alpha.
\end{align}
\end{corollary}

\paragraph{Comparisons to other sample complexity bounds for TDC} 

Let us compare our results in Theorem \ref{thm:TDC-error} and Corollary \ref{cor:TDC-sample-complexity} with the state-of-the-art sample complexities for the TDC algorithm.
The result that is most comparable to ours is obtained by \cite{dalal2020tale}, where a projected version of TDC is considered with decaying stepsizes $\alpha_t = O(t^{-\alpha})$ and $\beta_t = O(t^{-\beta})$ for $0 < \beta < \alpha < 1$. The sample complexity therein, with high-probability convergence guarantee at tolerance level $\varepsilon$, is of order $O \left( \left(\frac{1}{\varepsilon}\right)^{2\alpha} \right)$  without explicit dependence on the problem-related parameters. If one chooses $\alpha = 1-\delta$ with $\delta$ sufficiently small, their sample complexity bound can be improved, but it cannot achieve the rate $\Theta\left(\frac{1}{\varepsilon^2}\right)$. Regarding finite-sample  in-expectation error control for TDC, the best result so far is developed by \cite{kaledin2020finite}, who shows that with the choice of $\alpha_t, \beta_t \asymp \frac{1}{T}$, the sample complexity for TDC with tolerance level $\varepsilon$ can be upper bounded by $O\left(\frac{1}{\varepsilon^2}\right)$. Our result in Corollary \ref{cor:TDC-sample-complexity} is the first sample complexity for the original TDC algorithm that guarantees high-probability convergence and achieves the minimax-optimal rate of $O\left(\frac{1}{\varepsilon^2}\right)$; it is also noteworthy that we display an explicit dependence on problem-related parameters. The key to achieving this again lies in our combination of the matrix Freedman's inequality with an induction argument; the details of the proof is postponed to Section \ref{sec:pf-TDC}.
We also remark that \cite{xu2021sample} considers a variant of TDC where $\widetilde{\bm{\theta}}_t$ is updated \emph{not} with every sample tuple $(s_t,a_t,s_t')$, but with every batch of samples, and obtains a sample complexity of order $O(\frac{1}{\varepsilon^2}\log(\frac{1}{\varepsilon})).$

\section{Numerical experiments}

In this section, we corroborate our theoretical results with illustrative numerical experiments. In what follows, we will consider the on-policy and off-policy settings respectively.

\subsection{On-policy evaluation: averaged TD learning}

In the on-policy setting, we will investigate the empirical performance of the averaged TD learning algorithm. 

\paragraph{MDP setting} 
We consider a member of the family of MDPs constructed in proof of Theorem \ref{thm:minimax}, which provides a minimax lower bound. This family of MDPs is designed to be difficult to distinguish between each other, and hence, is a natural instance for evaluating the performance of TD learning. 
For construction details of this MDP, we refer the reader to Appendix \ref{sec:pf-lb}. 
In these simulations, we set $|\mathcal{S}| = 10$, $\gamma = 0.2$, and choose the stepsize of TD as $\eta = 0.01$. We examine both the original and the averaged TD iterates when the feature dimension equals to $d = 3$ and $d = 9$. Under each setting, $100$ independent trials for $T = 10^5$ iterations were conducted, and we report the mean value as well as the $95\%$ confidence band for the estimation error $\|\bm{\theta}_t - \bm{\theta}^\star\|_{\bm{\Sigma}}$ for TD and $\|\overline{\bm{\theta}}_t - \bm{\theta}^\star\|_{\bm{\Sigma}}$ for averaged TD.

\paragraph{Experimental results} 
Figure~\ref{fig:tda_minimax}(a) compares the performances of TD and averaged TD of an MDP with feature dimension $d=3$. While the estimation error of TD levels off at around $5 \times 10^{-3}$ after $10^3$ iterations, the error of averaged TD keeps decreasing to below $5 \times 10^{-4}$ when $T = 10^5$. 
In addition, Figure~\ref{fig:tda_minimax}(b) demonstrates the estimation error of averaged TD for MDPs with feature dimension $d=3$ and $d=9$. The slopes of these curves on the right part of this log-log plot match our theoretical prediction: the estimation error decreases in the order of $O(t^{-1/2})$. Moreover, the difference between the two curves indicates that the lower-dimension problem enjoys a faster convergence rate.

\begin{figure*}[t]
\begin{center}
\begin{tabular}{cc}
\includegraphics[width = 0.48\textwidth]{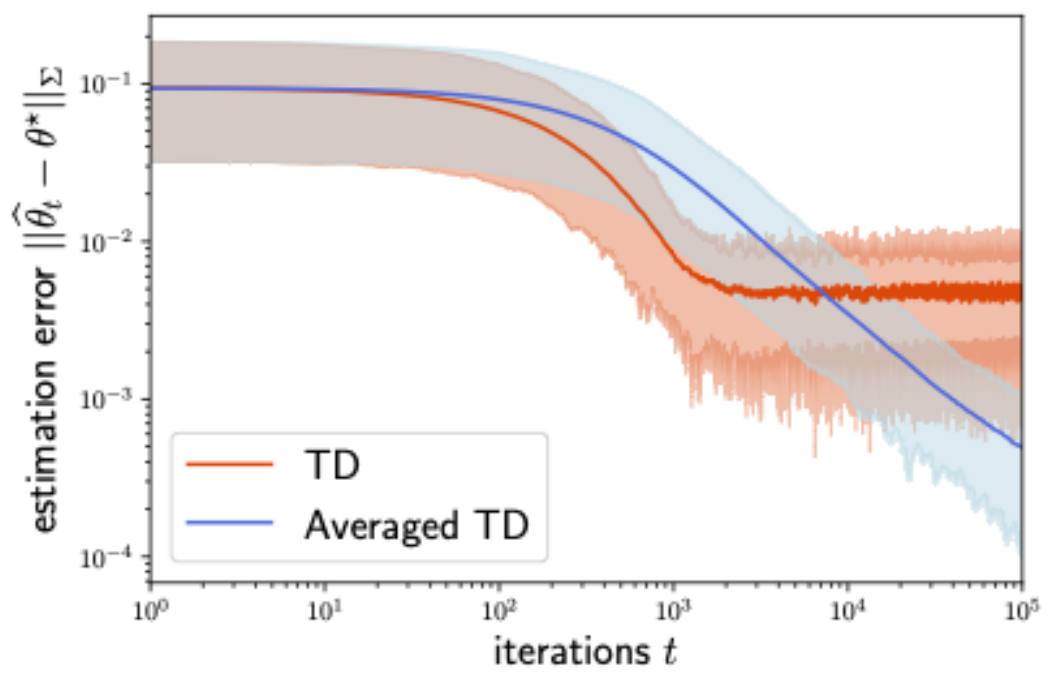} & \includegraphics[width = 0.48\textwidth]{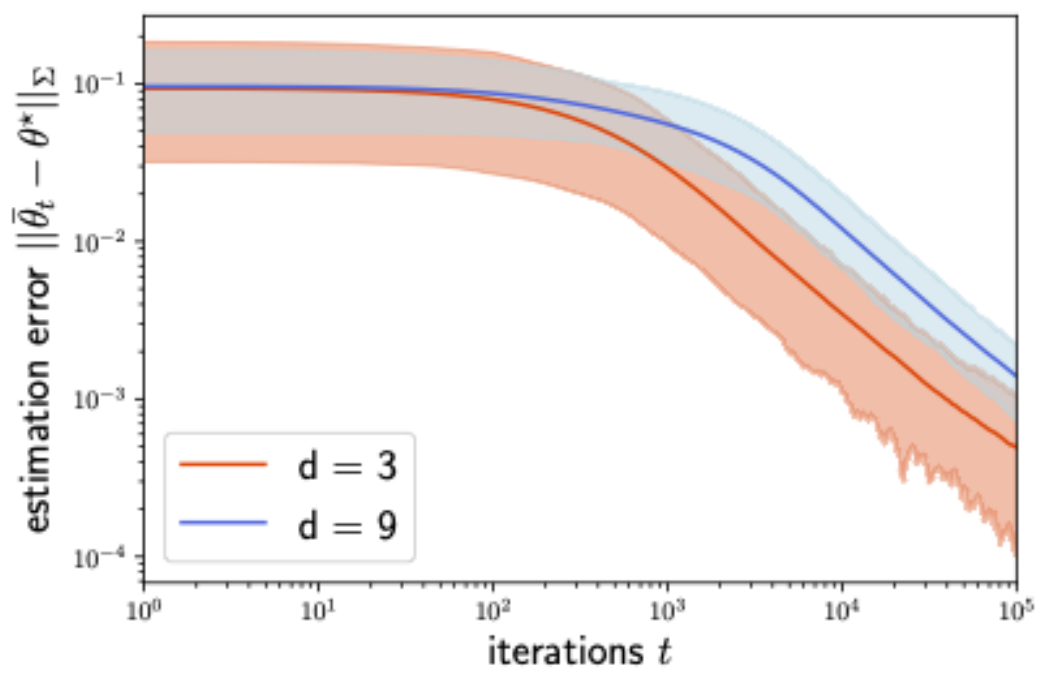}\\
\end{tabular}
\end{center}
\caption{(a) Comparisons of the estimation error of TD and averaged TD when $d=3$.  
(b) Comparisons of the estimation error for averaged TD with $d=3$ and $d=9$. 
Two curves in the middle represent their average errors, while the shaded areas represent the $95\%$ confidence bands. 
}
\label{fig:tda_minimax}
\end{figure*}

\subsection{Off-policy evaluation: TDC learning}

In order to demonstrate the efficiency of TDC for off-policy evaluation, we compare its performance with that of the off-policy TD learning on \emph{Baird's counterexample} \citep{baird1995residual}. 

\paragraph{Baird's counterexample}
We start by introducing Baird's counterexample, which was constructed to illustrate the instability of TD learning in the off-policy regime. Consider an MDP $(\mathcal{S},\mathcal{A},\mathcal{P},r,\gamma)$, with the discount factor $\gamma = 0.9$, state space $\mathcal{S} = [7]$, action space $\mathcal{A} = \{0,1\}$ and the reward function $r=0$ for all states and actions. 
The action $a=1$ transitions any initial state $s$ to $s'=7$, while the action $a=0$ transitions any initial state $s$ to $s' \in [6]$ with the same probability. 
The target policy $\pi$ selects $a=1$ at any given state $s$, while the behavior policy $\pi_{\mathsf{b}}$ takes $a = 0$ with probability $\frac{6}{7}$ and $a=1$ with probability $\frac{1}{7}$. 
Formally, the MDP satisfies the equations (see also Figure \ref{fig:Baird-MDP} for an illustration)
\begin{align*}
&\mathcal{P}(s'|s,1) = \ind\{s' = 7\}, \quad \forall s \in [7]; \\ 
&\mathcal{P}(s'|s,0) = \frac{1}{6} \ind\{1 \leq s' \leq 6\}, \quad \forall s \in [7];\\ 
& \pi(1|s) = 1, \quad \forall s \in [7];\\ 
& \pi_{\mathsf{b}}(0|s) = \frac{6}{7},  \quad \forall s \in [7]; \\ 
& \pi_{\mathsf{b}}(1|s) = \frac{1}{7}, \quad \forall s \in [7].
\end{align*}

\begin{figure*}[t]
\centering
\includegraphics[width = 0.6 \textwidth]{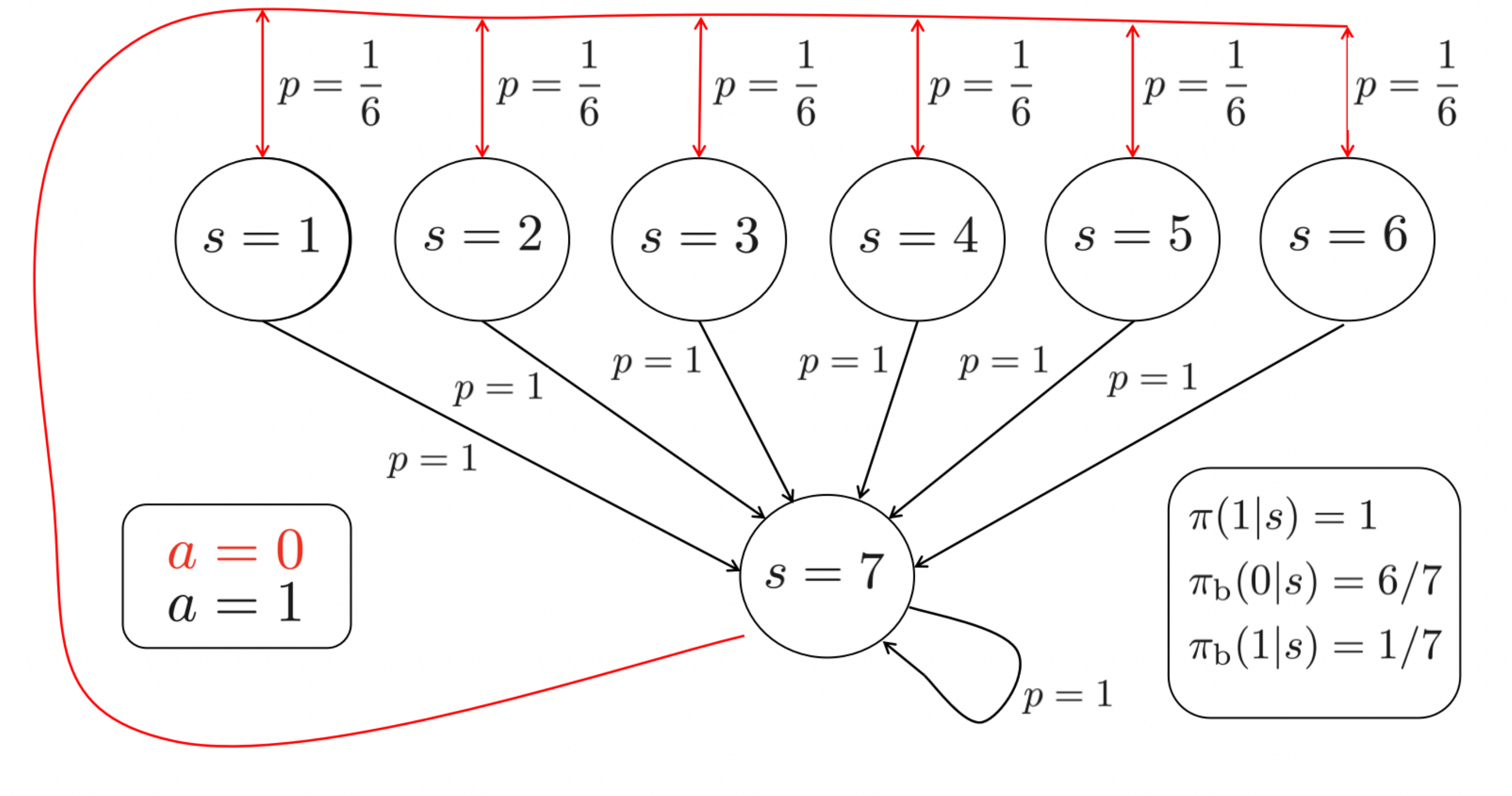}
\caption{Baird's counterexample. Taking action $a=1$ always leads to state $s = 7$, while taking $a = 0$ leads to one of the other six states with equal probability. The reward is set to be always zero.}
\label{fig:Baird-MDP}
\end{figure*}
In this example, it is easy to check that the stationary distribution corresponding to the behavior policy $\pi_{\mathsf{b}}$ is the uniform distribution among all states, and that the value function is $0$ for all states. We apply the following linear approximation of the value function: for $\bm{\theta} \in \mathbb{R}^8$, 
\begin{align}
\label{eq:Baird-LFA}
\begin{cases}
& V(i)= 2\theta_i + \theta_8, \quad \text{for} \quad 1 \leq i \leq 6;\\
& V(7)  = \theta_7 + 2 \theta_8.
\end{cases}
\end{align}
We remark that with this approximation, the feature space has a higher dimension ($d=8$) than the state space ($|\mathcal{S}| = 7$). Consequently, the optimal estimator $\widetilde{\bm{\theta}}^\star$ is not unique, and instead can be any $\bm{\theta} \in \mathbb{R}^8$ such that the estimated value vector is $\bm{V}_{\bm{\theta}} = \bm{0}$. Technically, this issue can be circumvented by creating several identical states as state $s=7$; we omit this detail here for simplicity, since we use $\|\widehat{\bm{\theta}}_t - \widetilde{\bm{\theta}}^\star\|_{\widetilde{\bm{\Sigma}}} = \|\bm{V}_{\widehat{\bm{\theta}}_t} - \bm{V}^\star\|_{\bm{D}_{\mu_{\mathsf{b}}}}$ to evaluate the estimation error, and our experimental results would remain the same. 

\paragraph{Experimental results} 
We perform $100$ independent trials for both off-policy averaged TD learning (with stepsize $\eta = 0.02$) and TDC (with stepsizes $\alpha = 0.02,\beta = 0.002$), starting at $\widehat{\bm{\theta}}_0 = (1,1,1,1,1,1,10,1)^\top$, as suggested by \cite{baird1995residual}. 
In these experiments, we set $\alpha = \eta$ to ensure that the stepsize for $\bm{\theta}$-updates are the same between the two algorithms. 
Figure \ref{fig:Baird-errs} demonstrates how the estimation error $\|\widehat{\bm{\theta}}_t - \widetilde{\bm{\theta}}^\star\|_{\widetilde{\bm{\Sigma}}}$ changes as two algorithms execute. As can be seen in this figure, TDC converges to an error of below $0.01$ after $T = 10^5$ iterations
while the off-policy averaged TD diverges to infinity. 

\begin{figure*}[t]
\centering
\includegraphics[width = 0.48\textwidth]{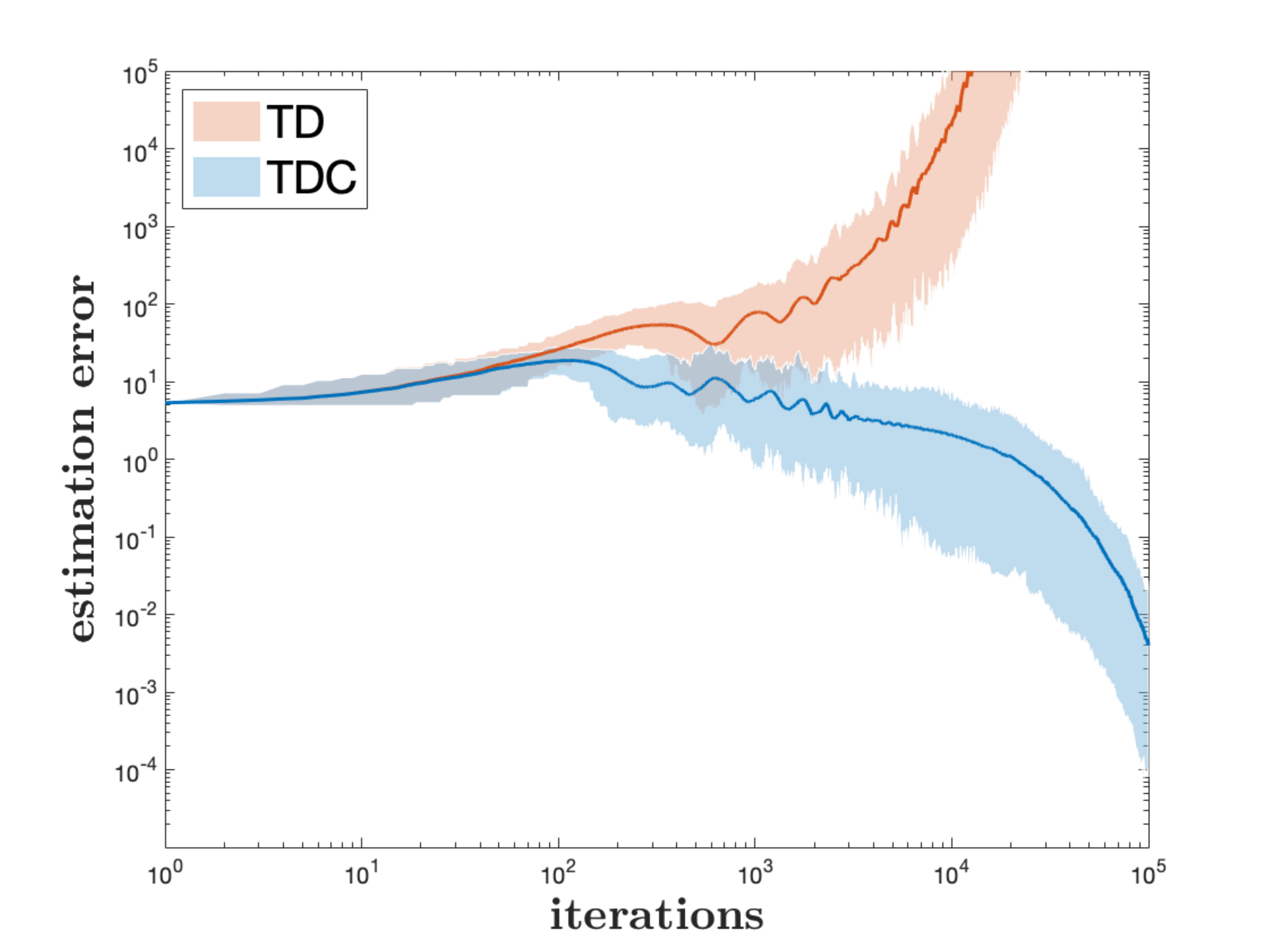}
\caption{Performances of off-policy averaged TD (red, $\eta = 0.02$) and TDC (blue, $\alpha = 0.02$, $\beta = 0.002$). Two curves in the middle represent their average errors, while the shaded areas correspond to $95\%$ confidence bands.}
\label{fig:Baird-errs}
\end{figure*}

\section{Proof of Theorem~\ref{thm:ind-td} (TD learning)}\label{sec:proof-TD}

\label{sec:proof-thm-ind-td}

For the sake of convenience, let us introduce the following notation 
\begin{align}
\Del{t}\defn\bm{\bmtheta}_{t}-\thetastar,\qquad\text{and}\qquad\Delbar{t}\defn\thetaavg{t}-\thetastar .
\end{align}

\paragraph{Step 1: a recursive relation}

To understand the convergence behavior of $\Delbar{t}$, the idea
is to first look at the following decomposition 
\begin{align*}
\Del{t+1} & =\bmtheta_{t+1}-\thetastar=\bmtheta_{t}-\thetastar-\eta(\bm{A}_{t}\bm{\bmtheta}_{t}-\bm{b}_{t})\\
 & =\bmtheta_{t}-\thetastar-\eta \big( \bm{A}_{t}\bm{\bmtheta}_{t}-\bm{b}_{t}-(\bm{A}\thetastar-\bm{b}) \big)\\
 & =\bmtheta_{t}-\thetastar-\eta\big(\bm{A}(\bmtheta_{t}-\thetastar)+(\bm{A}_{t}-\bm{A})\bmtheta_{t}-(\bm{b}_{t}-\bm{b})\big)\\
 & =(\Ind-\eta\bm{A})\Del{t}-\eta\bm{\xi}_{t},
\end{align*}
where we define 
\begin{align} 
\bm{\xi}_{t}\defn(\bm{A}_{t}-\bm{A})\bmtheta_{t}-(\bm{b}_{t}-\bm{b}).\label{eqn:defn-xi}
\end{align}
Here, the second line invokes the update rule \eqref{eq:TD-update-rule}
and the identity $\bm{A}\thetastar=\bm{b}$, whereas the third line is
obtained by properly rearranging terms. Applying the above relation
recursively, one arrives at 
\begin{align}
\Del{t}&=(\Ind-\eta\bm{A})\Del{t-1}-\eta\bm{\xi}_{t-1}\nonumber \\ 
&=(\Ind-\eta\bm{A})^{t}\Del{0}-\eta\sum_{i=0}^{t-1}(\Ind-\eta\bm{A})^{t-i-1}\bm{\xi}_{i}.\label{eqn:iterate-delt}
\end{align}

\paragraph{Step 2: a crude bound on $\|\Del{t}\|_{\bm{\Sigma}}$}

We aim to establish, via an induction argument, that with probability
at least $1-\delta$,
\begin{align}
\|\Del{t}\|_{\bm{\Sigma}}\leq32\sqrt{\frac{\eta\kappa\log\frac{2dT}{\delta}}{1-\gamma}}(1+\|\bmtheta^{\star}\|_{\bm{\Sigma}})+2\sqrt{\kappa}\|\Del{0}\|_{\bm{\Sigma}}\eqqcolon R_{0}
	\label{eq:bound_delta}
\end{align}
simultaneously over all $0\leq t\leq T$, 
as long as $0<\eta_t\le\frac{c_{3}(1-\gamma)}{\kappa\log\frac{dT}{\delta}}$
for some sufficiently small constant $c_{3}>0$. As a side remark, this boundedness
property saves us from enforcing additional projection steps as adopted in \cite{bhandari2018finite}. 

To start with, note that the inequality \eqref{eq:bound_delta} holds
trivially for the base case with $t=0$, given that $\kappa\geq1$.
Next, suppose that the hypothesis \eqref{eq:bound_delta} holds for
$\Del{0},\ldots,\Del{t-1}$, and we intend to establish it for $\Del{t}$
as well. Towards this end, invoking the decomposition~\eqref{eqn:iterate-delt}
and the triangle inequality yields 
\begin{align}
\lsigma{\Del{t}}\leq\big\|(\Ind-\eta\bm{A})^{t}\Del{0}\big\|_{\bm{\Sigma}}+\eta\,\lsigmaBig{\sum_{i=0}^{t-1}(\Ind-\eta\bm{A})^{t-i-1}\bm{\xi}_{i}}.\label{eqn:induction}
\end{align}
As for the first term of \eqref{eqn:induction}, it is seen that 
\begin{align}
&\big\|(\Ind-\eta\bm{A})^{t}\Del{0}\big\|_{\bm{\Sigma}} \nonumber \\ 
& =\big\|\bmSigma^{1/2}(\Ind-\eta\bm{A})^{t}\bmSigma^{-1/2}\bmSigma^{1/2}\Del{0}\big\|_{2} \nonumber \\ 
&\leq\big\|\bmSigma^{1/2}\big\|\cdot\big\|\bmSigma^{-1/2}\big\|\cdot\big\|\Ind-\eta\bm{A}\big\|^{t}\cdot\big\|\bmSigma^{1/2}\Del{0}\big\|_{2}\nonumber \\
 & \leq\sqrt{\kappa}\left(1-\frac{1}{2}\eta(1-\gamma)\lambdamin\right)^{t}\lsigma{\Del{0}}\leq\sqrt{\kappa}\lsigma{\Del{0}},\label{eqn:induction-termone}
\end{align}
where the last inequality arises from the definition of $\kappa$
and the property~\eqref{eq:I-eta-A-spectral-norm-bound} (with the
restriction that $\eta\leq(1-\gamma)/(4\|\bm{\Sigma}\|)$). When it
comes to the second term of \eqref{eqn:induction}, the following
lemma comes in handy. 

\begin{lemma}\label{lemma:exp-decay-sum} Fix any quantity $R>0$ and, for each $0\leq i\leq T-1$, define
the auxiliary random vector
\begin{align}
\bm{\widetilde{\xi}}_{i}\defn\bm{\xi}_{i}\,\ind\{\mathcal{H}_{i}\},\qquad\text{where}\quad\mathcal{H}_{i}\defn\Big\{\|\bm{\Delta}_{i}\|_{\bm{\Sigma}}\leq R\Big\}.\label{eqn:defn-xi-tilde}
\end{align}
Then,  with probability
at least $1-\delta/T$, simultaneously over the indices $(l,u,t)$ such that $0\leq l\leq u\leq t-1 < T$ it holds that
\begin{align*}
&\Big\|\sum_{i=l}^{u}(\Ind-\eta\bm{A})^{t-i-1}\bm{\widetilde{\xi}}_{i}\Big\|_{\bm{\Sigma}}\\
	&\leq 16 \Big(1-\frac{1}{2}\eta(1-\gamma)\lambda_{\min}(\bm{\Sigma}) \Big)^{t-u-1} \\ 
	&\cdot (\|\bm{\theta}^{\star}\|_{\bm{\Sigma}}+R+1)\sqrt{\frac{\kappa\log\frac{2dT}{\delta}}{\eta(1-\gamma)}},
\end{align*}
provided that $0<\eta\leq\frac{1-\gamma}{\kappa\log\frac{2dT}{\delta}}$. 
\end{lemma}
\begin{proof}See Section \ref{subsec:Proof-of-Lemma-exp-decay-sum}. 
\end{proof}

Under the induction hypothesis that
$\|\Del{i}\|_{\bm{\Sigma}}\leq R_{0}$ for $0\leq i\leq t-1$,
we can invoke Lemma \ref{lemma:exp-decay-sum} (with $R=R_{0},l=0$ and
$u=t-1$) to show that
\begin{align}
&\Big\|\sum_{i=0}^{t-1}(\Ind-\eta\bm{A})^{t-i-1}\bm{{\xi}}_{i}\Big\|_{\bm{\Sigma}}\nonumber \\ 
 & =\Big\|\sum_{i=0}^{t-1}(\Ind-\eta\bm{A})^{t-i-1}\bm{{\xi}}_{i}\ind\{\|\bm{\Delta}_{i}\|_{\bm{\Sigma}}\leq R_{0}\}\Big\|_{\bm{\Sigma}}\nonumber \\
 & \leq16(\|\bm{\theta}^{\star}\|_{\bm{\Sigma}}+R_{0}+1)\sqrt{\frac{\kappa\log\frac{2dT}{\delta}}{\eta(1-\gamma)}}\label{eqm:induction-termtwo}
\end{align}
holds with probability at least $1-\delta/T$, provided that $0<\eta\le\frac{1-\gamma}{\kappa\log\frac{dT}{\delta}}$.
Combining
\eqref{eqn:induction}, \eqref{eqn:induction-termone} and \eqref{eqm:induction-termtwo}
together and recalling the definition \eqref{eq:bound_delta} of $R_{0}$,
we can easily verify that 
\begin{align}
\lsigma{\Del{t}}&\leq\sqrt{\kappa}\lsigma{\Del{0}}
+16(\|\bm{\theta}^{\star}\|_{\bm{\Sigma}}+R_{0}+1)\sqrt{\frac{\eta\kappa\log\frac{2dT}{\delta}}{1-\gamma}}\nonumber \\ 
&\leq R_{0},
\end{align}
with the proviso that $32\sqrt{\frac{\eta\kappa\log(2dT/\delta)}{1-\gamma}}\leq1$.
The induction argument coupled with the union bound
then establishes the claim~\eqref{eq:bound_delta}. 

\paragraph{Step 3: a refined bound on $\|\Del{t}\|_{\bm{\Sigma}}$}

It turns out that the upper bound (\ref{eq:bound_delta}) is somewhat
loose due to the complete ignorance of the contraction effect of $\Ind-\eta\bm{A}$;
see (\ref{eqn:induction-termone}). In what follows, we develop a strengthened
 bound. Define 
%
\begin{align}
	t_{\mathsf{seg}} := \frac{c_1\log \max\{ 4\sqrt{\kappa}, \frac{16\kappa\|\bm{\Delta}_{0}\|_{\bm{\Sigma}}}{\|\bm{\theta}^\star\|_{\bm{\Sigma}} + 1}, \|\bm{\Delta}_{0}\|_{\bm{\Sigma}} \sqrt\frac{1-\gamma}{\eta\kappa\log\frac{2dT}{\delta}}\}}{\eta(1-\gamma)\lambdamin}
	\label{defn:t-seg}
\end{align}
for some sufficiently large constant $c_{1}>0$. For any integer $k\geq1$,
we aim to establish that
\begin{align}
	\lsigma{\Del{t}}\leq32\sqrt{\frac{\eta\kappa\log\frac{2dT}{\delta}}{1-\gamma}}\Big(\|\bm{\theta}^{\star}\|_{\bm{\Sigma}}+\frac{\sqrt{\kappa}\lsigma{\Del{0}}}{2^{k-1}}+\frac{3}{2}\Big)\eqqcolon R_{k}
	\label{eq:Delta-t-refined-bound-iid}
\end{align}
for any $t$ obeying $kt_{\mathsf{seg}}\leq t\le T$, provided that
$0<\eta\leq\frac{c_{3}(1-\gamma)}{\kappa\log\frac{dT}{\delta}}$ for
some small enough constant $c_{3}>0$. 

Because of relation \eqref{eqn:induction}, we claim that it suffices
to prove that 
\begin{align}
&\lsigmaBig{\sum_{i=0}^{t-1}(\Ind-\eta\bm{A})^{t-i-1}\bm{\xi}_{i}}\nonumber \\ 
&\leq32\Big(\|\bm{\theta}^{\star}\|_{\bm{\Sigma}}+\frac{2\sqrt{\kappa}\lsigma{\Del{0}}}{2^{k}}+1\Big)\sqrt{\frac{\kappa\log\frac{2dT}{\delta}}{\eta(1-\gamma)}},\\
& \text{when } kt_{\mathsf{seg}}\leq t\leq T.
	\label{eq:bound-2nd-term-refined}
\end{align}
To see this: note that the first term on the right-hand side of~\eqref{eqn:induction} 
has already been bounded in \eqref{eqn:induction-termone}, which combined with the definition \eqref{defn:t-seg} of $t_{\mathsf{seg}}$ indicates that
\begin{align}
&\big\|\bmSigma^{1/2}(\Ind-\eta\bm{A})^{t}\Del{0}\big\|_{2} \\ 
&\leq\sqrt{\kappa}\left(1-\frac{1}{2}\eta(1-\gamma)\lambdamin\right)^{t_{\mathsf{seg}}}\lsigma{\Del{0}} \\ 
&\leq\sqrt{\frac{\eta\kappa\log\frac{2dT}{\delta}}{1-\gamma}}
	\label{eq:bound-first-term-refined}
\end{align}
for any $t\ge t_{\mathsf{seg}}$. Clearly,
combining (\ref{eq:bound-2nd-term-refined}) with \eqref{eqn:induction} and \eqref{eq:bound-first-term-refined}
shall immediately lead to the claim \eqref{eq:Delta-t-refined-bound-iid}.
The remainder of this step is thus devoted to demonstrating (\ref{eq:bound-2nd-term-refined}) inductively. 

The base case (i.e.~$k=1$) follows immediately from our bounds \eqref{eq:bound_delta}
and \eqref{eqm:induction-termtwo} in Step 2, given that $\sqrt{\frac{\eta\kappa\log\frac{2dT}{\delta}}{1-\gamma}}$
is sufficiently small. Suppose now that the claim \eqref{eq:bound-2nd-term-refined}
holds for a given integer $k\geq1$ and any $t$ obeying $kt_{\mathsf{seg}}\leq t\leq T$,
and we intend to show that \eqref{eq:bound-2nd-term-refined} continues
to hold for $k+1$ and any $t$ obeying $(k+1)t_{\mathsf{seg}}\leq t\leq T$.
Towards this, we first single out the following straightforward decomposition
\begin{align*}
&\lsigmaBig{\sum_{i=0}^{t-1}(\Ind-\eta\bm{A})^{t-i-1}\bm{\xi}_{i}} \\ 
 & \leq\lsigmaBig{\sum_{i=t-t_{\mathsf{seg}}+1}^{t-1}(\Ind-\eta\bm{A})^{t-i-1}\bm{\xi}_{i}}+\lsigmaBig{\sum_{i=0}^{t-t_{\mathsf{seg}}}(\Ind-\eta\bm{A})^{t-i-1}\bm{\xi}_{i}},
\end{align*}
which allows us to upper bound the two terms on the right-hand side above
seperately. 
\begin{itemize}
\item Under the induction hypothesis that $\|\bm{\Delta}_{i}\|_{\bm{\Sigma}}\leq R_{k}$
for all $i$ obeying $ kt_{\mathsf{seg}} \leq i\leq T$, one can invoke Lemma \ref{lemma:exp-decay-sum} with $R=R_{k},l=t-t_{\mathsf{seg}}+1$
and $u=t-1$ to see that
\begin{align*}
&\lsigmaBig{\sum_{i=t-t_{\mathsf{seg}}+1}^{t-1}(\Ind-\eta\bm{A})^{t-i-1}\bm{\xi}_{i}}\\ 
 & =\lsigmaBig{\sum_{i=t-t_{\mathsf{seg}}+1}^{t-1}(\Ind-\eta\bm{A})^{t-i-1}\bm{\xi}_{i}\ind\{\|\bm{\Delta}_{i}\|_{\bm{\Sigma}}\leq R_{k}\}}\\
 & \leq16\big(\|\bm{\theta}^{\star}\|_{\bm{\Sigma}}+R_{k}+1\big)\sqrt{\frac{\kappa\log\frac{2dT}{\delta}}{\eta(1-\gamma)}}\\
 & \leq24\Big(\|\bm{\theta}^{\star}\|_{\bm{\Sigma}}+\frac{2\sqrt{\kappa}\lsigma{\Del{0}}}{2^{k+1}}+1\Big)\sqrt{\frac{\kappa\log\frac{2dT}{\delta}}{\eta(1-\gamma)}},
\end{align*}
where  the last line uses the definition \eqref{eq:Delta-t-refined-bound-iid}
of $R_{k}$ and holds as long as $\frac{\eta\kappa\log\frac{2dT}{\delta}}{1-\gamma}$
is sufficiently small.

\item In addition, we make the observation that: for any $t\geq t_{\mathsf{seg}}$,
\begin{align*}
&\Big\|\sum_{i=0}^{t-t_{\mathsf{seg}}}(\Ind-\eta\bm{A})^{t-i-1}\bm{\xi}_{i}\Big\|_{\bm{\Sigma}} \\ 
& =\Big\|\sum_{i=0}^{t-t_{\mathsf{seg}}}(\Ind-\eta\bm{A})^{t-i-1}\bm{\xi}_{i}\ind\{\|\bm{\Delta}_{i}\|_{\bm{\Sigma}}\leq R_{0}\}\Big\|_{\bm{\Sigma}}\\
 & \leq16(1-\frac{1}{2}\eta(1-\gamma)\lambda_{\min}(\bm{\Sigma}))^{t_{\mathsf{seg}}-1} \\ 
 &\cdot (\|\bm{\theta}^{\star}\|_{\bm{\Sigma}}+R_{0}+1)\sqrt{\frac{\kappa\log\frac{2dT}{\delta}}{\eta(1-\gamma)}}\\
 & \leq8\big(\|\bm{\theta}^{\star}\|_{\bm{\Sigma}}+1\big)\sqrt{\frac{\kappa\log\frac{2dT}{\delta}}{\eta(1-\gamma)}}.
\end{align*}
Here, the first equality uses the crude bound $\|\bm{\Delta}_{i}\|_{\bm{\Sigma}}\leq R_{0}$ for all $i$ (see \eqref{eq:bound_delta}), 
the second to last inequality utilizes Lemma \ref{lemma:exp-decay-sum}
with $R=R_{0},l=0$ and $u=t-t_{\mathsf{seg}}$, whereas the last
inequality relies on the definition \eqref{eq:bound_delta} of $R_{0}$
and invokes the fact that $\sqrt{\kappa}\left(1-\frac{1}{2}\eta(1-\gamma)\lambdamin\right)^{t_{\mathsf{seg}}-1}\leq\min\big\{\frac{1}{4},\frac{1}{4\sqrt{\kappa}\|\bm{\Delta}_{0}\|_{\bm{\Sigma}}}\big\}$
with our choice \eqref{defn:t-seg} of $t_{\mathsf{seg}}$. 

\end{itemize}
Combine the previous two bounds to reach 
\begin{align*}
&\lsigmaBig{\sum_{i=0}^{t-1}(\Ind-\eta\bm{A})^{t-i-1}\bm{\xi}_{i}} \\ 
 & \leq24\Big(\|\bm{\theta}^{\star}\|_{\bm{\Sigma}}+\frac{2\sqrt{\kappa}\lsigma{\Del{0}}}{2^{k+1}}+1\Big)\sqrt{\frac{\kappa\log\frac{2dT}{\delta}}{\eta(1-\gamma)}} \\ 
 &+8\big(\|\bm{\theta}^{\star}\|_{\bm{\Sigma}}+1\big)\sqrt{\frac{\kappa\log\frac{2dT}{\delta}}{\eta(1-\gamma)}}\\
 & \leq32\Big(\|\bm{\theta}^{\star}\|_{\bm{\Sigma}}+\frac{2\sqrt{\kappa}\lsigma{\Del{0}}}{2^{k}}+1\Big)\sqrt{\frac{\kappa\log\frac{2dT}{\delta}}{\eta(1-\gamma)}}.
\end{align*}
This finishes the induction step and in turn establishes \eqref{eq:bound-2nd-term-refined} (and hence \eqref{eq:Delta-t-refined-bound-iid}). 

As a straightforward consequence, the bounds \eqref{eq:bound_delta}
and \eqref{eq:Delta-t-refined-bound-iid} imply that 
\begin{align}
\lsigma{\Del{t}}\leq\begin{cases}
R_{0},\qquad & 0\leq t<t_{\mathsf{seg}}',\\
32\sqrt{\frac{\eta\kappa\log\frac{2dT}{\delta}}{1-\gamma}}(\|\bm{\theta}^{\star}\|_{\bm{\Sigma}}+2),\qquad & t_{\mathsf{seg}}'\leq t<T,
\end{cases}\label{eq:Delta-t-two-cases}
\end{align}
where 
\begin{align}
t_{\mathsf{seg}}'\coloneqq c_{2}t_{\mathsf{seg}}\log\big(\kappa(\|\bm{\Delta}_{0}\|_{2}+1)\big)\label{defn:tilde-t-seg-defn}
\end{align}
for some large enough constant $c_{2}>0$. To see this, note that
for any $t\geq t_{\mathsf{seg}}'$, it is guaranteed that
the second term on the right-hand side of \eqref{eq:Delta-t-refined-bound-iid} obeys $\frac{4\sqrt{\kappa}\lsigma{\Del{0}}}{2^{\lfloor t/t_{\mathsf{seg}}\rfloor}}\leq2$,
thus confirming the second case in \eqref{eq:Delta-t-two-cases}.

\paragraph{Step 4: controlling $\|\Delbar{T}\|_{\bm{\Sigma}}$}

Now we are positioned to control $\Delbar{T}$.
The key is to write $\Delbar{T}$ as a linear combination of $\{\bm{\xi}_{i}\}_{0\leq i\leq T-1}$
as follows, which is a direct consequence of the relation~\eqref{eqn:iterate-delt}:
\begin{align}
\Delbar{T} & =\frac{1}{T}\sum_{j=1}^{T}\Del{j}\nonumber \\ 
&=\frac{1}{T}\sum_{j=1}^{T}(\Ind-\eta\bm{A})^{j}\Del{0}-\frac{1}{T}\sum_{j=1}^{T}\eta\sum_{i=0}^{j-1}(\Ind-\eta\bm{A})^{j-i-1}\bm{\xi}_{i}\nonumber \\
 & =\frac{1}{T}\sum_{j=1}^{T}(\Ind-\eta\bm{A})^{j}\Del{0}-\frac{1}{T}\sum_{i=0}^{T-1}\eta\sum_{j=i+1}^{T}(\Ind-\eta\bm{A})^{j-i-1}\bm{\xi}_{i}\nonumber \\
 & =\frac{1}{T\eta}\bm{A}_{0}^{(T+1)}\Del{0}-\frac{1}{T}\Del{0}-\frac{1}{T}\sum_{i=0}^{T-1}\bm{A}_{i}^{(T)}\bm{\xi}_{i},\label{eq:Deltbar-decomposition-TD}
\end{align}
where the middle line follows from swapping the summation over $i$
and $j$, and in the last line we define 
\begin{align}
\bm{A}_{i}^{(t)}\defn\eta\sum_{j=i+1}^{t}(\Ind-\eta\bm{A})^{j-i-1}=\bm{A}^{-1}\big(\Ind-(\Ind-\eta\bm{A})^{t-i}\big).\label{eq:defn-Ait}
\end{align}
We claim that the following two inequalities hold, the first deterministically and the second with probability of at least $1-\delta$ (with their proofs
deferred to Section~\ref{sec:schubert}) \begin{subequations} 
\begin{align}
\big\|\bm{A}_{0}^{(T+1)}\bm{\Delta}_{0}\big\|_{\bm{\Sigma}} & \leq\frac{2\|\bm{\Sigma}^{-1}\|}{1-\gamma}\big\|\bm{\Delta}_{0}\|_{\bm{\Sigma}};
	\label{eqn:deltazeroterm}\\
\Big\|\sum_{i=0}^{T-1}\bm{A}_{i}^{(T)}\bm{\xi}_{i}\Big\|_{\bm{\Sigma}} & \lesssim \Bigg\{ \sqrt{\frac{\max_{s}\bm{\phi}(s)^{\top}\bm{\Sigma}^{-1}\bm{\phi}(s)\log\frac{2d}{\delta}}{T(1-\gamma)^{2}}}\\ 
&+\frac{\big\|\bm{\Sigma}^{-1}\big\|}{T}\sqrt{\frac{\kappa\log\frac{2dT}{\delta}}{\eta(1-\gamma)^{3}}}\Bigg\} \big(\|\bm{\theta}^{\star}\|_{\bm{\Sigma}}+1\big).\label{eqn:schubert}
\end{align}
\end{subequations} Putting the above two inequalities together with
\eqref{eq:Deltbar-decomposition-TD}, 
 we arrive at 
\begin{align*}
	 \lsigma{\Delbar{T}}  &\leq\Big\|\frac{1}{T\eta}\bm{A}_{0}^{(T+1)}\Del{0}\Big\|_{\bm{\Sigma}}+\Big\|\frac{1}{T}\Del{0}\Big\|_{\bm{\Sigma}}+\Big\|\frac{1}{T}\sum_{i=0}^{T-1}\bm{A}_{i}^{(T)}\bm{\xi}_{i}\Big\|_{\bm{\Sigma}}\\
	& \lesssim\frac{1}{\eta T}\frac{\big\|\bm{\Sigma}^{-1}\big\|}{1-\gamma}\big\|\bm{\Delta}_{0}\|_{\bm{\Sigma}}+ \frac{1}{T} \|\bm{\Delta}_0\|_{\bm{\Sigma}}  \\ 
	&+
	\Bigg\{ \sqrt{\frac{\max_{s}\bm{\phi}(s)^{\top}\bm{\Sigma}^{-1}\bm{\phi}(s)\log\frac{2d}{\delta}}{T(1-\gamma)^{2}}} \\ 
	&+\frac{\big\|\bm{\Sigma}^{-1}\big\|}{T}\sqrt{\frac{\kappa\log\frac{2dT}{\delta}}{\eta(1-\gamma)^{3}}}\Bigg\} \big(\|\bm{\theta}^{\star}\|_{\bm{\Sigma}}+1\big) \\
 & \asymp \frac{1}{\eta T}\frac{\big\|\bm{\Sigma}^{-1}\big\|}{1-\gamma}\big\|\bm{\Delta}_{0}\|_{\bm{\Sigma}}+\Bigg\{ \sqrt{\frac{\max_{s}\bm{\phi}(s)^{\top}\bm{\Sigma}^{-1}\bm{\phi}(s)\log \frac{2d}{\delta}}{T(1-\gamma)^{2}}} \\ 
 &+\frac{\big\|\bm{\Sigma}^{-1}\big\|}{T}\sqrt{\frac{\kappa\log\frac{2dT}{\delta}}{\eta(1-\gamma)^{3}}}\Bigg\} \big(\|\bm{\theta}^{\star}\|_{\bm{\Sigma}}+1\big),
\end{align*}
where the last line follows since $\|\bm{\Sigma}^{-1}\|\geq 1$ (see \eqref{eq:Sigma-norm-bound}) and $\eta < 1$. 
This finishes the proof of Theorem~\ref{thm:ind-td}.

\section{Proof of Theorem \ref{thm:TDC-error} (TDC learning)}
\label{sec:pf-TDC}

Firstly, let us analyze the population dynamics of TDC. It turns out that the convergence of this dynamics can be described via a contractive linear mapping. 
Given this nice property of population TDC, we shall decompose the empirical TDC into two parts: the first part can be controlled via the aforementioned population dynamics, and the rest is treated as a stochastic component, which is controlled via matrix martingale concentration.

\subsection{Population analysis}
\label{subsec:TDC-population}

First recall that the population parameters are defined as  
\begin{align*}
\widetilde{\bm{A}} & \coloneqq\mathbb{E}_{\mu_{\mathsf{b}},\pi_{\mathsf{b}},\mathcal{P}}[\widetilde{\bm{A}}_{t}]=\mathbb{E}_{\mu_{\mathsf{b}},\pi_{\mathsf{b}},\mathcal{P}}[\rho_{t}\bm{\phi}(s_t)\left(\bm{\phi}(s_t)-\gamma\bm{\phi}(s_t')\right)^{\top}];\\
\widetilde{\bm{b}} & \coloneqq\mathbb{E}_{\mu_{\mathsf{b}},\pi_{\mathsf{b}}}[\widetilde{\bm{b}}_{t}]=\mathbb{E}_{\mu_{\mathsf{b}},\pi_{\mathsf{b}}}[\rho_{t}\bm{\phi}(s_t)r_{t}];\\
\bm{\Pi} & \coloneqq\mathbb{E}_{\mu_{\mathsf{b}},\pi_{\mathsf{b}},\mathcal{P}}[\bm{\Pi}_{t}]=\mathbb{E}_{\mu_{\mathsf{b}},\pi_{\mathsf{b}},\mathcal{P}}[\rho_{t}\bm{\phi}(s_t)\bm{\phi}(s_t')^{\top}];\\
\widetilde{\bm{\Sigma}} & \coloneqq\mathbb{E}_{\mu_{\mathsf{b}}}[\bm{\Sigma}_{t}]=\mathbb{E}_{\mu_{\mathsf{b}}}[\bm{\phi}(s_t)\bm{\phi}(s_t)^{\top}].
\end{align*}
Corresponding to the empirical version of TDC as given in \eqref{subeq:TDC}, 
we can define its population analogue of TDC as 
\begin{align}
\label{subeq:TDC-population}
\breve{\bm{\theta}}_{t+1} & =\breve{\bm{\theta}}_{t}-\alpha(\widetilde{\bm{A}}\breve{\bm{\theta}}_{t}-\widetilde{\bm{b}}+\gamma\bm{\Pi}^{\top}\breve{\bm{w}}_{t}), \nonumber \\
\breve{\bm{w}}_{t+1} & =\breve{\bm{w}}_{t}-\beta(\widetilde{\bm{A}}\breve{\bm{\theta}}_{t}-\widetilde{\bm{b}}+\widetilde{\bm{\Sigma}}\breve{\bm{w}}_{t}),
\end{align}
where sampled parameters are replaced by their corresponding expectations. 
In this section, we analyze the population dynamics of TDC as given above; in order to control the finite-sample dynamics, we bound the difference of these two in the section to follow.

Since $\bm{\phi}(s_t)$ is independent of the transition, the expectation of $\widetilde{\bm{\Sigma}}_t$ is independent of which policy is being adopted. Hence, $\widetilde{\bm{\Sigma}}$ can also be presented as
\begin{align}
\widetilde{\bm{\Sigma}} &= \sum_{s_t \in \mathcal{S}} \mu_{\mathsf{b}}(s_t) \bm{\phi}(s_t)\bm{\phi}(s_t)^\top \nonumber \\ 
&= \sum_{s_t \in \mathcal{S}} \mu_{\mathsf{b}}(s_t) \left(\sum_{a_t \in \mathcal{A}} \pi(a_t |s_t)\right) \bm{\phi}(s_t)\bm{\phi}(s_t)^\top \nonumber \\
&= \sum_{s_t \in \mathcal{S}} \sum_{a_t \in \mathcal{A}} \mu_{\mathsf{b}}(s_t) \pi_{\mathsf{b}}(a_t|s_t)\left( \frac{\pi(a_t |s_t)}{\pi_{\mathsf{b}}(a_t|s_t)}\right)  \bm{\phi}(s_t)\bm{\phi}(s_t)^\top \\ 
&=\mathbb{E}_{\mu_{\mathsf{b}},\pi_{\mathsf{b}}}[\rho_t \bm{\phi}(s_t)\bm{\phi}(s_t)^{\top}].
\end{align}
In view of this relation, $\widetilde{\bm{A}}$ admits another characterization, namely 
\begin{align}
\widetilde{\bm{A}}=\widetilde{\bm{\Sigma}}-\gamma\bm{\Pi}.
\end{align}
Consequently, the fixed point $(\breve{\bm{\theta}}^{\star},\bm{w}^{\star})$
of the population dynamics obeys 
\begin{align*}
\begin{cases}
\widetilde{\bm{A}}\breve{\bm{\theta}}^{\star}-\widetilde{\bm{b}}+\gamma\bm{\Pi}^{\top}\bm{w}^{\star}=\bm{0},\\
\widetilde{\bm{A}}\breve{\bm{\theta}}^{\star}-\widetilde{\bm{b}}+\widetilde{\bm{\Sigma}}\bm{w}^{\star}=\bm{0}.
\end{cases}
\end{align*}
As long as $\widetilde{\bm{A}}$ is invertible, this
set of conditions is equivalent to 
\begin{align*}
\widetilde{\bm{A}}\breve{\bm{\theta}}^{\star}=\widetilde{\bm{b}},\qquad\text{and}\qquad\bm{w}^{\star}=\bm{0}.
\end{align*}

In order to study the population dynamics, it is useful to consider two auxiliary parameters 
\begin{align*}
\breve{\bm{\Delta}}_{t} &\coloneqq\breve{\bm{\theta}}_{t}-\breve{\bm{\theta}}^{\star}, \\
\breve{\bm{z}}_{t} &\coloneqq\breve{\bm{w}}_{t}+\widetilde{\bm{\Sigma}}^{-1}\widetilde{\bm{A}}\breve{\bm{\Delta}}_{t};
\end{align*}
here $\breve{\bm{\Delta}}_{t}$ tracks the convergence of $\breve{\bm{\theta}}_{t}$ to $\breve{\bm{\theta}}^{\star}$, and $\breve{\bm{z}}_{t}$ tracks the size of the residual $\widetilde{\bm{A}}\breve{\bm{\theta}}_{t}-\widetilde{\bm{b}}+\widetilde{\bm{\Sigma}}\breve{\bm{w}}_{t}$. 
With these two parameters in place, the population dynamics satisfy 
\begin{align}
&\left[\begin{array}{c}
\breve{\bm{\Delta}}_{t}\\
\breve{\bm{z}}_{t}
\end{array}\right] \nonumber \\ 
&=\left[\begin{array}{cc}
\bm{I}-\alpha\widetilde{\bm{A}}^{\top}\widetilde{\bm{\Sigma}}^{-1}\widetilde{\bm{A}} & -\alpha\gamma\bm{\Pi}^{\top}\\
-\alpha\bm{X}\widetilde{\bm{A}}^{\top}\widetilde{\bm{\Sigma}}^{-1}\widetilde{\bm{A}} & \bm{I}-\beta\widetilde{\bm{\Sigma}}-\alpha\gamma\bm{X}\bm{\Pi}^{\top}
\end{array}\right] \cdot \left[\begin{array}{c}
\breve{\bm{\Delta}}_{t-1}\\
\breve{\bm{z}}_{t-1}
\end{array}\right],
\end{align}
in which we use $\bm{X}$ to denote $\bm{I}-\gamma \widetilde{\bm{\Sigma}}^{-1}\widetilde{\bm{\Pi}}$.
To analyze this optimization dynamics, for every positive constant $\varkappa \in (0,1)$, consider 
\begin{align*}
\breve{\bm{x}}_{t}\coloneqq\left[\begin{array}{c}
\breve{\bm{\Delta}}_{t}\\
\varkappa\breve{\bm{z}}_{t}
\end{array}\right]
\end{align*}
then $\breve{\bm{x}}_{t}$ yields
\begin{align}
\breve{\bm{x}}_{t} = \bm{\Psi}\breve{\bm{x}}_{t-1},
\end{align}
where $\bm{\Psi}$ represents the matrix
\begin{align}
\label{eq:defn-Psi}
\left[\begin{array}{cc}
\bm{I}-\alpha\widetilde{\bm{A}}^{\top}\widetilde{\bm{\Sigma}}^{-1}\widetilde{\bm{A}} & -\frac{1}{\varkappa}\alpha\gamma\bm{\Pi}^{\top}\\
-\varkappa \alpha\bm{X}\widetilde{\bm{A}}^{\top}\widetilde{\bm{\Sigma}}^{-1}\widetilde{\bm{A}} & \bm{I}-\beta\widetilde{\bm{\Sigma}}-\alpha\gamma\bm{X}\bm{\Pi}^{\top}
\end{array}\right].
\end{align}
It is known that how fast $\breve{\bm{x}}_{t}$ converges to $\bm{0}$ is determined by the spectral norm of $\bm{\Psi}$, which is characterized in the lemma below. 

\begin{lemma}
\label{lemma:psi-norm}
Suppose that 
\begin{align*}
\lambda_1  =\lambda_{\min}(\widetilde{\bm{A}}^{\top}\widetilde{\bm{\Sigma}}^{-1}\widetilde{\bm{A}}), \quad
\lambda_2  =\lambda_{\min}(\widetilde{\bm{\Sigma}}), \quad
\lambda_{\Sigma}  = \| \widetilde{\bm{\Sigma}}^{-1}\| = 1/\lambda_2.
\end{align*}
Then as long as the following conditions hold:
\begin{subequations}
\label{eq:TDC-step-conditions}
\begin{align}
&\beta \gtrsim\lambda_{\Sigma}\rho_{\max} \alpha, \\
&\varkappa \beta \gtrsim \alpha,\\
&\alpha\gamma(\rho_{\max}+\gamma\lambda_{\Sigma}\rho_{\max}^{2})  \ll\beta\lambda_{w},\\
&\frac{\alpha\gamma\rho_{\max}}{\varkappa}+\varkappa\alpha(1+\gamma\lambda_{\Sigma}\rho_{\max})\lambda_{\Sigma}(2\rho_{\max})^{2}  \ll \sqrt{\alpha\lambda_1\beta\lambda_{w}}
\end{align}
\end{subequations}
it holds true that  
\begin{align*}
\|\bm{\Psi}\|\leq 1-\frac{1}{2}\alpha\lambda_1.
\end{align*}
\end{lemma}

Therefore, the mapping $\bm{\Psi}$ is contractive, thus ensuring the linear convergence of $\bm{x}_{t}$, with the proviso that $\alpha\lambda_1 < 2.$

\subsection{Finite-sample analysis}\label{subsec:TDC-finite}

Armed with the population analysis, the proof for Theorem \ref{thm:TDC-error} is completed if we can make a connection of the finite-sample performances to that of the population ones.  

\paragraph{Step 1: a recursive relation} Firstly, we define two noise variables
\begin{align*}
\bm{\nu}_{t} & \coloneqq(\widetilde{\bm{A}}_{t}-\widetilde{\bm{A}})\widetilde{\bm{\theta}}_{t}-(\widetilde{\bm{b}}_{t}-\widetilde{\bm{b}})+\gamma(\bm{\Pi}_{t}-\bm{\Pi})^{\top}\bm{w}_{t},\\
\bm{\eta}_{t} & \coloneqq(\widetilde{\bm{A}}_{t}-\widetilde{\bm{A}})\widetilde{\bm{\theta}}_{t}-(\widetilde{\bm{b}}_{t}-\widetilde{\bm{b}})+(\widetilde{\bm{\Sigma}}_{t}-\widetilde{\bm{\Sigma}})\bm{w}_{t}.
\end{align*}
As a result, TDC can be rewritten as 
\begin{align*}
\widetilde{\bm{\theta}}_{t+1} & =\widetilde{\bm{\theta}}_{t}-\alpha(\widetilde{\bm{A}}\widetilde{\bm{\theta}}_{t}-\widetilde{\bm{b}}+\gamma\bm{\Pi}^{\top}\bm{w}_{t})-\alpha\bm{\nu}_{t};\\
\bm{w}_{t+1} & =\bm{w}_{t}-\beta(\widetilde{\bm{A}}\widetilde{\bm{\theta}}_{t}-\widetilde{\bm{b}}+\widetilde{\bm{\Sigma}}\bm{w}_{t})-\beta\bm{\eta}_{t}.
\end{align*}
Using the same notations as in Section \ref{subsec:TDC-population}, we observe that the following iteration holds true for finite-sample TDC:
\begin{align*}
\bm{x}_{t+1} = \bm{\Psi}\bm{x}_{t} - \bm{\zeta}_t,
\end{align*}
in which
\begin{align}\label{eq:defn-zeta}
\bm{\zeta}_t = \left[\begin{array}{c}
\alpha \bm{\nu}_t\\
\varkappa(\alpha(1-\gamma \widetilde{\bm{\Sigma}}^{-1}\bm{\Pi})\bm{\nu}_t + \beta \bm{\eta}_t)
\end{array}\right].
\end{align}

Hence,
\begin{align}\label{eq:TDC-xt}
\bm{x}_t =\bm{\Psi}^t \bm{x}_0 - \sum_{i=0}^{t-1} \bm{\Psi}^{t-i-1} \bm{\zeta}_i,
\end{align}
where $\bm{x}_0 = [\bm{\Delta}_0^\top,\varkappa \bm{z}_0^\top]^\top$. Since the norm of $\bm{\Psi}$ has been bounded by Lemma \ref{lemma:psi-norm}, bounding the norm of $\bm{x}_t$ boils down to bounding the second term of \eqref{eq:TDC-xt}. In the following, with a slight abuse of notation, for any $\bm{x} = (\bm{x}_1,\bm{x}_2) \in \mathbb{R}^{2d}$ with $\bm{x}_1,\bm{x}_2 \in \mathbb{R}^{d}$, we will define $\|\bm{x}\|_{\widetilde{\bm{\Sigma}}}^2$ as
\begin{align*}
\|\bm{x}\|_{\widetilde{\bm{\Sigma}}}^2 = \|\bm{x}_1\|_{\widetilde{\bm{\Sigma}}}^2 + \|\bm{x}_2\|_{\widetilde{\bm{\Sigma}}}^2.
\end{align*}
with this definition, it is easy to see that
\begin{align*}
\|\bm{x}_{t}\|_{\widetilde{\bm{\Sigma}}}^2 = \|\tilde{\bm{\Delta}}_t\|_{\widetilde{\bm{\Sigma}}}^2 + \varkappa^2 \|\bm{w}_t + \widetilde{\bm{\Sigma}}^{-1}\widetilde{\bm{A}} \widetilde{\bm{\Delta}}_t\|_{\widetilde{\bm{\Sigma}}}^2.
\end{align*}
Hence, the norms of $\widetilde{\bm{\Delta}}_t$, $\bm{w}_t$ and $\bm{x}_t$ can be related by the  inequalities
\begin{align}\label{eq:norm-relation}
\begin{cases}
&\|\widetilde{\bm{\Delta}}_t\|_{\widetilde{\bm{\Sigma}}} \leq \|\bm{x}_t\|_{\widetilde{\bm{\Sigma}}};\\
&\|\bm{w}_t\|_{\widetilde{\bm{\Sigma}}} \lesssim \frac{1}{\varkappa}\|\bm{x}_t\|_{\widetilde{\bm{\Sigma}}};\\
&\|\bm{x}_t\|_{\widetilde{\bm{\Sigma}}} \lesssim \|\widetilde{\bm{\Delta}}_t\|_{\widetilde{\bm{\Sigma}}} + |\bm{w}_t\|_{\widetilde{\bm{\Sigma}}}.
\end{cases}
\end{align}
\paragraph{Step 2: crude bound for $\|\bm{x}_t\|_{\widetilde{\bm{\Sigma}}}$} 
We first aim to establish, via an induction argument, that with probability at least $1-\delta$,
\begin{align}\label{eq:defn-tilde-R0}
\|\bm{x}_t\|_{\widetilde{\bm{\Sigma}}} &\leq  2\|\widetilde{\bm{\Delta}}_0\|_{\widetilde{\bm{\Sigma}}} + 80\varkappa \beta \rho_{\max} \sqrt{\frac{1}{\alpha \lambda_1}\log \frac{2dT}{\delta}} (\|\widetilde{\bm{\theta}}^\star\|_{\widetilde{\bm{\Sigma}}} + 1)\nonumber \\ 
&=: \widetilde{R}_0
\end{align}
holds simulatanesouly for all $0 \leq t \leq T$.
To start with, note that the inequality \eqref{eq:defn-tilde-R0} holds trivially for the base case with $t=0$. Next, suppose that the hypothesis \eqref{eq:defn-tilde-R0} holds for $\bm{x}_0,\bm{x}_1,\ldots,,\bm{x}_{t-1}$, and we intend to establish it for $\bm{x}_t$ as well. Towards this end, involking the decomposition \eqref{eq:TDC-xt} and the triangle inequality yields
\begin{align}
\label{eq:xt-induction}
\|\bm{x}_t\|_{\widetilde{\bm{\Sigma}}} \leq \left\|\bm{\Psi}^t \bm{x}_0\right\|_{\widetilde{\bm{\Sigma}}} + \left\| \sum_{i=0}^{t-1} \bm{\Psi}^{t-i-1} \bm{\zeta}_i \right\|_{\widetilde{\bm{\Sigma}}}.
\end{align}
As for the first term of \eqref{eq:xt-induction}, it is seen that
\begin{align}\label{eq:xt-1st}
\left\|\bm{\Psi}^t \bm{x}_0\right\|_{\widetilde{\bm{\Sigma}}} \leq \|\bm{x}_0\|_{\widetilde{\bm{\Sigma}}} = \|\widetilde{\bm{\Delta}}_0\|_{\widetilde{\bm{\Sigma}}}.
\end{align}
When it comes to the second term of \eqref{eq:xt-induction}, the following lemma comes in handy.
\begin{lemma}\label{lemma:exp-decay-sum-xt}
Fix any quantity $\widetilde{R}>0$ and, for each $0 \leq i \leq T-1$, define the random vector
\begin{align}\label{eq:defn-x-tilde}
\widetilde{\bm{\zeta}}_i:= \bm{\zeta}_i \ind\{\widetilde{\mathcal{H}}_i\}, \quad \text{where} \quad \widetilde{\mathcal{H}}_i := \left\{\|\bm{x}_i\|_{\widetilde{\bm{\Sigma}}} \leq \widetilde{R}\right\}.
\end{align}
Then, with probability at least $1-\delta/T$,
\begin{align}
\label{eqn:lemma-part-2}
&\left\|\sum_{i=0}^{t-1} \bm{\Psi}^{t-i-1} \widetilde{\bm{\zeta}}_i\right\|_{\widetilde{\bm{\Sigma}}} \nonumber \\ 
& \lesssim \sqrt{\frac{\|\widetilde{\bm{\Sigma}}\|}{\alpha \lambda_1}\log \frac{2dT}{\delta}}\varkappa \beta \rho_{\max} (\|\widetilde{\bm{\theta}}^\star\|_{\widetilde{\bm{\Sigma}}} + \frac{1}{\varkappa}\widetilde{R}  + 1),
\end{align}
provided that the stepsizes $\alpha,\beta$ satisfy the conditions \eqref{eq:TDC-step-conditions} and that $0 < \alpha < \frac{1}{\lambda_1\lambda_{\Sigma}^2 \|\widetilde{\bm{\Sigma}}\|\log \frac{2dT}{\delta}}$.
\end{lemma}
\begin{proof} See Section \ref{subsec:proof-of-lemma-exp-decay-sum-xt}.\end{proof}

Putting relations~\eqref{eq:xt-induction} and \eqref{eqn:lemma-part-2} together, we find 
\begin{align*}
&\|\bm{x}_t\|_{\widetilde{\bm{\Sigma}}}\\ 
& = 
	\|\widetilde{\bm{\Delta}}_0\|_{\widetilde{\bm{\Sigma}}} \\
	&+ C \sqrt{\frac{\|\widetilde{\bm{\Sigma}}\|}{\alpha \lambda_1}\log \frac{2dT}{\delta}}\varkappa \beta \rho_{\max} (\|\widetilde{\bm{\theta}}^\star\|_{\widetilde{\bm{\Sigma}}} + \frac{1}{\varkappa}\widetilde{R}_0 + 1)  \\
	&\leq \widetilde{R}_0
\end{align*}
by definition of $\widetilde{R}_0$ in \eqref{eq:defn-tilde-R0}, 
provided that $\sqrt{\frac{1}{\alpha \lambda_1}\log \frac{2dT}{\delta}}\beta \rho_{\max} \leq c$ for some constant $c > 0$ small enough. 
Therefore, by induction assumption, one has
\begin{align}
&\mathbb{P}\left\{ \max_{0 \leq i \leq t} \|\bm{x}_i\|_{\widetilde{\bm{\Sigma}}} > \widetilde{R}_0 \right\}\\
&\leq \mathbb{P} \left\{ \max_{0 \leq i < t-1} \|\bm{x}_i\|_{\widetilde{\bm{\Sigma}}} > \widetilde{R}_0 \right\} \nonumber \\ 
&+ \mathbb{P} \left\{ \max_{0 \leq i < t-1} \|\bm{x}_i\|_{\widetilde{\bm{\Sigma}}} \leq \widetilde{R}_0, \|\bm{x}_t\|_{\widetilde{\bm{\Sigma}}} > \widetilde{R}_0 \right\}\nonumber \\
&\leq \frac{(t-1)\delta}{T}  + \mathbb{P}\Bigg\{\left\|\sum_{i=0}^{t-1} \bm{\Psi}^{t-i-1} \widetilde{\bm{\zeta}}_i\right\|_{\widetilde{\bm{\Sigma}}} \nonumber \\ 
&\gtrsim \sqrt{\frac{\|\widetilde{\bm{\Sigma}}\|}{\alpha \lambda_1}\log \frac{2dT}{\delta}}\varkappa \beta \rho_{\max} (\|\widetilde{\bm{\theta}}^\star\|_{\widetilde{\bm{\Sigma}}} + \widetilde{R}_0 + 1)\Bigg\}\nonumber \\
&\leq \frac{(t-1)\delta}{T} + \frac{\delta}{T} = \frac{t\delta}{T}.
\end{align}
This completes our claim at this step.

\paragraph{Step 3: refined bound for $\|\bm{x}_t\|_{\widetilde{\bm{\Sigma}}}$} It turns out that the upper bound \eqref{eq:defn-tilde-R0} can be tightened by taking into account the contraction effect of $\bm{\Psi}$. In what follows, we develop a strengthened bound. Define
\begin{align}\label{eq:defn-t-set-tilde}
\widetilde{t}_{\mathsf{seg}}:= \frac{\widetilde{c_1}\log \max \left\{\sqrt{\widetilde{\kappa}},\frac{\sqrt{\widetilde{\kappa}}\|\widetilde{\bm{\Delta}}_0\|_{\widetilde{\bm{\Sigma}}}}{\|\widetilde{\bm{\theta}}^\star\|_{\widetilde{\bm{\Sigma}}} + 1},\|\widetilde{\bm{\Delta}}_0\|_{\widetilde{\bm{\Sigma}}} \sqrt{\frac{\alpha \lambda_1}{\|\widetilde{\bm{\Sigma}}\|\log \frac{2dT}{\delta}}} \frac{1}{\varkappa \beta \rho_{\max}}\right\}}{\alpha \lambda_1}
\end{align}
for some sufficiently large constant $\widetilde{c}_1 > 0$, where $\widetilde{\kappa}$ is the condition number of $\widetilde{\bm{\Sigma}}$. For any integer $k>1$, we claim that with probability at least $1-\delta$,
\begin{align}
\label{eq:xt-refined-bound}
\|\bm{x}_t\|_{\widetilde{\bm{\Sigma}}} &\lesssim  \varkappa \beta \rho_{\max} \sqrt{\frac{1}{\alpha \lambda_1}\log \frac{2dT}{\delta}}\left(\|\widetilde{\bm{\theta}}^\star\|_{\widetilde{\bm{\Sigma}}} + \frac{\|\widetilde{\bm{\Delta}}_0\|_{\widetilde{\bm{\Sigma}}}}{2^{k-1}} + \frac{3}{2}\right)\nonumber \\ 
&=: \widetilde{R}_k
\end{align}
for any $t$ obeying $k\widetilde{t}_{\mathsf{seg}} \leq t \leq T$, provided that $\sqrt{\frac{1}{\alpha \lambda_1}\log \frac{2dT}{\delta}}\varkappa \beta \rho_{\max} \leq c$ for some constant $c$ small enough. 
The proof of this claim is essentially the same as that of Step 3 for proving Theorem \ref{thm:ind-td}, and we will omit it here. Therefore, by defining
\begin{align}
&\widetilde{t}_{\mathsf{seg}}' \nonumber \\ 
&:= \left(2 + \frac{1}{\log 2} \log \|\widetilde{\bm{\theta}^\star}\|_{\widetilde{\bm{\Sigma}}}\right) \widetilde{t}_{\mathsf{seg}},
\end{align}
we can conclude that with probability at least $1-\delta$, for all $t \geq \widetilde{t}_{\mathsf{seg}}'$,
\begin{align}
\|\bm{x}_t\|_{\widetilde{\bm{\Sigma}}} \lesssim \varkappa \beta \rho_{\max} \sqrt{\frac{\|\widetilde{\bm{\Sigma}}\|}{\alpha \lambda_1}\log \frac{2dT}{\delta}}\left(\|\widetilde{\bm{\theta}}^\star\|_{\widetilde{\bm{\Sigma}}} +2\right).
\end{align}
Recall that this bound holds for any $\varkappa \in (0,1)$ satisfying the conditions \eqref{eq:TDC-step-conditions}. Hence, Theorem \ref{thm:TDC-error} follows by taking $\varkappa = 8\rho_{\max}\sqrt{\frac{\alpha}{\lambda_1\beta \lambda_2}}$ and 
\begin{align*}
\frac{\alpha}{\beta} = \frac{1}{128} \frac{\lambda_1 \lambda_2}{\rho_{\max}^2 (1+\lambda_{\Sigma}\rho_{\max})}.
\end{align*}

\section{Discussion}

Our primary contribution in this paper is obtaining high-probability sample complexity bounds for both the TD and TDC algorithms for policy evaluation in the $\gamma$-discounted infinite-horizon MDPs. 
For TD learning with Polyak-Ruppert averaging, we improve upon existing results in terms of both the accuracy level $\varepsilon$ and other problem-related parameters like the effective horizon $\frac{1}{1-\gamma}$, the weighted feature covariance $\bm{\Sigma}$ and the optimal linear estimator $\bm{\theta}^\star$. 
We have also established a minimax lower bound and showed that our upper bound is near-minimax optimal by a factor of $\frac{1}{1-\gamma}$. 
For TDC with linear function approximation, we provide the first sample complexity bound that achieves the optimal dependence on the error tolerance $\varepsilon$, and characterize the exact dependence on problem-related constants at the same time.

Our analysis leaves open several directions for future investigation; we close by sampling a few of them.
Regarding TD learning, a natural direction of future work is to close the $\frac{1}{1-\gamma}$ gap between our upper bound and the minimax lower bound. 
Notably, this gap also appears in the bounds of \cite{duan2021optimal} for least-square TD in general when no restriction of the variance for the temporal difference residual is imposed. 
In terms of TDC, while our result provides a tight control of the same size $T$, the dependence on problem-related constants can be potentially improved. 
Moreover, it is noteworthy that the analysis in this work is based on the assumption of $i.i.d.$ transition pairs drawn from the stationary distribution; it is of natural interest to generalize these results to other scenarios such as Markovian trajectories. 
Moving beyond linear function approximation, understanding the sample complexities for policy evaluation with other function classes is also an interesting direction.


\appendices

\section{Preliminary facts\label{sec:Preliminary-facts}}

\label{sec:preliminary}

The following two lemmas consider the basic properties of important matrices and vectors that would be useful in the proof of the main theorems in the paper.

\begin{lemma} \label{lemma:D-P-D-Phi} Recall the definitions of
$\bm{\Phi}$, $\bm{D}_{\mu}$ and $\bm{\Sigma}$ in \eqref{eq:defn-Phi-feature-matrix},
\eqref{eq:defn-Dmu} and \eqref{eq:defn-Sigma}, respectively. Then
one has 
\begin{align}
\big\|\bm{D}_{\mu}^{\frac{1}{2}}\bm{\Phi}\bm{\Sigma}^{-\frac{1}{2}}\big\|=1, \qquad\text{and}\qquad\big\|\bm{D}_{\mu}^{\frac{1}{2}}\bm{P}^\pi \bm{D}_{\mu}^{-\frac{1}{2}}\big\|=1.\label{eq:defn-Phitilde-Ptilde-1-1}
\end{align}
\end{lemma} \begin{proof} For notational convenience, let $\widetilde{\bm{\Phi}}:=\bm{D}_{\mu}^{\frac{1}{2}}\bm{\Phi}\bm{\Sigma}^{-\frac{1}{2}}$
and $\bm{P}_{\bm{D}_{\mu}}:=\bm{D}_{\mu}^{\frac{1}{2}}\bm{P}^\pi \bm{D}_{\mu}^{-\frac{1}{2}}$.
First of all, it is seen that 
\begin{align*}
\big\|\widetilde{\bm{\Phi}}\big\|&=\sqrt{\big\|\widetilde{\bm{\Phi}}^{\top}\widetilde{\bm{\Phi}}\big\|} \\ 
&=\sqrt{\big\|\bm{\Sigma}^{-\frac{1}{2}}\bm{\Phi}^{\top}\bm{D}_{\mu}^{\frac{1}{2}}\bm{D}_{\mu}^{\frac{1}{2}}\bm{\Phi}\bm{\Sigma}^{-\frac{1}{2}}\big\|} \\ 
&=\sqrt{\big\|\bm{\Sigma}^{-\frac{1}{2}}\bm{\Sigma}\bm{\Sigma}^{-\frac{1}{2}}\big\|}=1.
\end{align*}
When it comes to $\big\|\bm{P}_{\bm{D}_{\mu}}\big\|$, we make the
observation that 
\begin{align*}
\big\|\bm{P}_{\bm{D}_{\mu}}\big\|&=\sqrt{\big\|\bm{P}_{\bm{D}_{\mu}}\bm{P}_{\bm{D}_{\mu}}^{\top}\big\|} \\ 
&=\sqrt{\left\Vert \bm{D}_{\mu}^{\frac{1}{2}}\bm{P}\bm{D}_{\mu}^{-1}\bm{P}^{\top}\bm{D}_{\mu}^{\frac{1}{2}}\right\Vert }\\ 
&=\sqrt{\left\Vert \bm{D}_{\mu}^{\frac{1}{2}}\left(\bm{P}\bm{D}_{\mu}^{-1}\bm{P}^{\top}\bm{D}_{\mu}\right)\bm{D}_{\mu}^{-\frac{1}{2}}\right\Vert }=1.
\end{align*}
To see why the last identity holds, observe that $\bm{P}\bm{D}_{\mu}^{-1}\bm{P}^{\top}\bm{D}_{\mu}$
is a stochastic matrix, that is $\bm{P}\bm{D}_{\mu}^{-1}\bm{P}^{\top}\bm{D}_{\mu}$
contains nonnegative elements, and 
\[
\bm{P}\bm{D}_{\mu}^{-1}\bm{P}^{\top}\bm{D}_{\mu}\bm{1}=\bm{1}.
\]
In addition, $\bm{D}_{\mu}^{\frac{1}{2}}\left(\bm{P}\bm{D}_{\mu}^{-1}\bm{P}^{\top}\bm{D}_{\mu}\right)\bm{D}_{\mu}^{-\frac{1}{2}}$
is similar to $\bm{P}\bm{D}_{\mu}^{-1}\bm{P}^{\top}\bm{D}_{\mu}$.
As a result, by the Perron-Frobenious theorem,
\begin{align*}
&\left\Vert \bm{D}_{\mu}^{\frac{1}{2}}\left(\bm{P}\bm{D}_{\mu}^{-1}\bm{P}^{\top}\bm{D}_{\mu}\right)\bm{D}_{\mu}^{-\frac{1}{2}}\right\Vert \\ 
 & =\max_{i}|\lambda_{i}(\bm{D}_{\mu}^{\frac{1}{2}}\left(\bm{P}\bm{D}_{\mu}^{-1}\bm{P}^{\top}\bm{D}_{\mu}\right)\bm{D}_{\mu}^{-\frac{1}{2}})|\\
 & =\max_{i}|\lambda_{i}(\bm{P}\bm{D}_{\mu}^{-1}\bm{P}^{\top}\bm{D}_{\mu})|=1,
\end{align*}
where $\lambda_{i}(\bm{B})$ denotes the $i$-th eigenvalue of the matrix
$\bm{B}$. 
\end{proof}

\begin{lemma}\label{lemma:Sigma-inv-A-lower-bound}Suppose that $\|\bm{r}\|_{\infty}\leq1$.
For any $0\leq\gamma<1$, the matrix $\bm{\Sigma}$ defined in \eqref{eq:defn-Sigma}
and the vector $\bm{b}$ defined in \eqref{eq:defn-bt-mean} obey
\begin{subequations}\label{eq:Sigma-A-Sigma-LB}
\begin{align}
&\bm{\Sigma}^{-\frac{1}{2}}\bm{A}^{\top}\bm{\Sigma}^{-1}\bm{A}\bm{\Sigma}^{-\frac{1}{2}}  \succeq(1-\gamma)^{2}\Ind,\label{eq:Sigma-A-Sigma-LB-1}\\
&\bm{\Sigma}^{-\frac{1}{2}}\bm{A}\bm{\Sigma}^{-1}\bm{A}^{\top}\bm{\Sigma}^{-\frac{1}{2}}  \succeq(1-\gamma)^{2}\Ind,\label{eq:Sigma-A-Sigma-LB-2}\\
&\big\|\bm{\Sigma}^{\frac{1}{2}}(\bm{A}^{\top})^{-1}\bm{\Sigma}\bm{A}^{-1}\bm{\Sigma}^{\frac{1}{2}}\big\|  \leq(1-\gamma)^{-2},\label{eq:Sigma-A-Sigma-UB-1}\\
&\big\|\bm{\Sigma}^{\frac{1}{2}}\bm{A}^{-1}\bm{\Sigma}(\bm{A}^{\top})^{-1}\bm{\Sigma}^{\frac{1}{2}}\big\|  \leq(1-\gamma)^{-2},\label{eq:Sigma-A-Sigma-UB-2}\\
&\big\|\bm{\Sigma}^{\frac{1}{2}}\bm{A}^{-1}\bm{\Sigma}^{\frac{1}{2}}\big\|  \leq(1-\gamma)^{-1},\label{eq:Sigma-A-Sigma-UB-12}\\
&\big\|\bm{\Sigma}^{-1/2}\bm{\Phi}^\top \bm{D}_{\bm{\mu}}\big\|\leq \max_{s \in \mathcal{S}} \bm{\phi}(s)^\top \bm{\Sigma}^{-1} \bm{\phi}(s),\label{eq:Sigma-Phi-D-LB}\\
&\|\Ind-\eta\bm{A}\|  \leq1-\frac{1}{2}\eta(1-\gamma)\lambda_{\min}(\bm{\Sigma}),\quad\forall0<\eta<\frac{1-\gamma}{4\left\Vert \bm{\Sigma}\right\Vert },\label{eq:I-eta-A-spectral-norm-bound}\\
	&\|\bm{\Sigma}\| \leq1,\qquad\|\bm{\Sigma}^{-1}\|\geq1, \label{eq:Sigma-norm-bound} \\
&\big\|\bm{\Sigma}^{-\frac{1}{2}}\bm{b}\big\|_{2}  \leq1.\label{eq:Sigma-A-b-UB}
\end{align}
\end{subequations}
\end{lemma}

\begin{proof} We shall establish each of these claims separately as follows.

\paragraph{Proof of Eqn.~\eqref{eq:Sigma-A-Sigma-LB-1} and \eqref{eq:Sigma-A-Sigma-LB-2}}

We start with the lower bound on $\bm{\Sigma}^{-\frac{1}{2}}\bm{A}^{\top}\bm{\Sigma}^{-1}\bm{A}\bm{\Sigma}^{-\frac{1}{2}}$.
To begin with, observe that 
\begin{align*}
\bm{\Sigma}^{-\frac{1}{2}}\bm{A}\bm{\Sigma}^{-\frac{1}{2}} & =\bm{\Sigma}^{-\frac{1}{2}}\bm{\Phi}^{\top}\bm{D}_{\mu}(\Ind-\gamma\bm{P})\bm{\Phi}\bm{\Sigma}^{-\frac{1}{2}}\\
 & =\bm{\Sigma}^{-\frac{1}{2}}\bm{\Phi}^{\top}\bm{D}_{\mu}\bm{\Phi}\bm{\Sigma}^{-\frac{1}{2}} \\ 
 &-\gamma\bm{\Sigma}^{-\frac{1}{2}}\bm{\Phi}^{\top}\bm{D}_{\mu}^{\frac{1}{2}}\Big(\bm{D}_{\mu}^{\frac{1}{2}}\bm{P}\bm{D}_{\mu}^{-\frac{1}{2}}\Big)\bm{D}_{\mu}^{\frac{1}{2}}\bm{\Phi}\bm{\Sigma}^{-\frac{1}{2}}\\
 & =\Ind-\gamma\widetilde{\bm{\Phi}}^{\top}\bm{P}_{\bm{D}_{\mu}}\widetilde{\bm{\Phi}},
\end{align*}
where 
\begin{align}
\widetilde{\bm{\Phi}}:=\bm{D}_{\mu}^{\frac{1}{2}}\bm{\Phi}\bm{\Sigma}^{-\frac{1}{2}}\qquad\text{and}\qquad\bm{P}_{\bm{D}_{\mu}}:=\bm{D}_{\mu}^{\frac{1}{2}}\bm{P}\bm{D}_{\mu}^{-\frac{1}{2}}.\label{eq:defn-Phitilde-Ptilde-1}
\end{align}
Therefore, any unit vector $\bm{x}$ (i.e.~$\|\bm{x}\|_{2}=1$) necessarily
satisfies 
\begin{align*}
\bm{x}^{\top}\bm{\Sigma}^{-\frac{1}{2}}\bm{A}^{\top}\bm{\Sigma}^{-1}\bm{A}\bm{\Sigma}^{-\frac{1}{2}}\bm{x} & =\big\|\bm{\Sigma}^{-\frac{1}{2}}\bm{A}\bm{\Sigma}^{-\frac{1}{2}}\bm{x}\big\|_{2}^{2} \\ 
&\geq\big(\bm{x}^{\top}\bm{\Sigma}^{-\frac{1}{2}}\bm{A}\bm{\Sigma}^{-\frac{1}{2}}\bm{x}\big)^{2}\\
 & =\big(1-\gamma\bm{x}^{\top}\widetilde{\bm{\Phi}}^{\top}\bm{P}_{\bm{D}_{\mu}}\widetilde{\bm{\Phi}}\bm{x}\big)^{2}.
\end{align*}
Further, Lemma \ref{lemma:D-P-D-Phi} tells us that 
\begin{align}
\left|\bm{x}^{\top}\widetilde{\bm{\Phi}}^{\top}\bm{P}_{\bm{D}_{\mu}}\widetilde{\bm{\Phi}}\bm{x}\right|\leq\|\widetilde{\bm{\Phi}}^{\top}\bm{P}_{\bm{D}_{\mu}}\widetilde{\bm{\Phi}}\big\|\leq\|\widetilde{\bm{\Phi}}\big\|^{2}\|\bm{P}_{\bm{D}_{\mu}}\big\|=1.\label{eq:Phi-tilde-P-norm}
\end{align}
Putting the preceding two bounds together, we demonstrate that 
\begin{align*}
\bm{x}^{\top}\bm{\Sigma}^{-\frac{1}{2}}\bm{A}^{\top}\bm{\Sigma}^{-1}\bm{A}\bm{\Sigma}^{-\frac{1}{2}}\bm{x} & \geq\big(1-\gamma\big)^{2}
\end{align*}
for any unit vector $\bm{x}$, thus concluding the proof of \eqref{eq:Sigma-A-Sigma-LB-1}.
The proof for \eqref{eq:Sigma-A-Sigma-LB-2} follows from an identical
argument and is omitted for brevity.

\paragraph{Proof of Eqn.~\eqref{eq:Sigma-A-Sigma-UB-1}, \eqref{eq:Sigma-A-Sigma-UB-2}
and \eqref{eq:Sigma-A-Sigma-UB-12}}

With the above bounds in place, we can further obtain 
\begin{align*}
&\big\|\bm{\Sigma}^{\frac{1}{2}}(\bm{A}^{\top})^{-1}\bm{\Sigma}\bm{A}^{-1}\bm{\Sigma}^{\frac{1}{2}}\big\| \\ 
&=\big\|\big(\bm{\Sigma}^{-\frac{1}{2}}\bm{A}\bm{\Sigma}^{-1}\bm{A}^{\top}\bm{\Sigma}^{-\frac{1}{2}}\big)^{-1}\big\| \\ 
&\leq\frac{1}{\lambda_{\min}\big(\bm{\Sigma}^{-\frac{1}{2}}\bm{A}\bm{\Sigma}^{-1}\bm{A}^{\top}\bm{\Sigma}^{-\frac{1}{2}}\big)}\leq\frac{1}{(1-\gamma)^{2}},
\end{align*}
where $\lambda_{\min}(\bm{B})$ denotes the smallest eigenvalue of
$\bm{B}$, and the last inequality comes from \eqref{eq:Sigma-A-Sigma-LB-2}.
This establishes \eqref{eq:Sigma-A-Sigma-UB-1}. The inequality \eqref{eq:Sigma-A-Sigma-UB-2}
follows from a similar argument. This also implies that 
\[
\big\|\bm{\Sigma}^{\frac{1}{2}}\bm{A}^{-1}\bm{\Sigma}^{\frac{1}{2}}\big\|=\sqrt{\big\|\bm{\Sigma}^{\frac{1}{2}}(\bm{A}^{\top})^{-1}\bm{\Sigma}\bm{A}^{-1}\bm{\Sigma}^{\frac{1}{2}}\big\|}\leq\frac{1}{1-\gamma},
\]
as claimed in \eqref{eq:Sigma-A-Sigma-UB-12}.

\paragraph{Proof of Eqn.~\eqref{eq:I-eta-A-spectral-norm-bound}}

Recalling that $\bm{\Sigma}=\bm{\Phi}^{\top}\bm{D}_{\mu}\bm{\Phi}$,
we can arrange terms to derive 
\begin{align*}
\bm{A}+\bm{A}^{\top} & =\bm{\Phi}^{\top}\bm{D}_{\mu}(\Ind-\gamma\bm{P})\bm{\Phi}+\bm{\Phi}^{\top}(\Ind-\gamma\bm{P}^{\top})\bm{D}_{\mu}\bm{\Phi}\\
 & =2\bm{\Sigma}-\gamma\bm{\Sigma}^{\frac{1}{2}}\Bigg\{ \bm{\Sigma}^{-\frac{1}{2}}\bm{\Phi}^{\top}\bm{D}_{\mu}\bm{P}\bm{\Phi}\bm{\Sigma}^{-\frac{1}{2}} \\ 
 &+\bm{\Sigma}^{-\frac{1}{2}}\bm{\Phi}^{\top}\bm{P}^{\top}\bm{D}_{\mu}\bm{\Phi}\bm{\Sigma}^{-\frac{1}{2}}\Bigg\} \bm{\Sigma}^{\frac{1}{2}}\\
 & =\bm{\Sigma}^{\frac{1}{2}}\left\{ 2\Ind-\gamma\big(\widetilde{\bm{\Phi}}^{\top}\bm{P}_{\bm{D}_{\mu}}\widetilde{\bm{\Phi}}+\widetilde{\bm{\Phi}}^{\top}\bm{P}_{\bm{D}_{\mu}}^{\top}\widetilde{\bm{\Phi}}\big)\right\} \bm{\Sigma}^{\frac{1}{2}}\\
 & \succeq\bm{\Sigma}^{\frac{1}{2}}\left\{ 2\Ind-2\gamma\big\|\widetilde{\bm{\Phi}}^{\top}\bm{P}_{\bm{D}_{\mu}}\widetilde{\bm{\Phi}}\big\|\Ind\right\} \bm{\Sigma}^{\frac{1}{2}}\\
 & \succeq2(1-\gamma)\bm{\Sigma},
\end{align*}
where $\widetilde{\bm{\Phi}}$ and $\bm{P}_{\bm{D}_{\mu}}$ are defined
in \eqref{eq:defn-Phitilde-Ptilde-1}. Here, the last line follows
since $\big\|\widetilde{\bm{\Phi}}^{\top}\bm{P}_{\bm{D}_{\mu}}\widetilde{\bm{\Phi}}\big\|\leq1$
--- a fact that has already been shown in \eqref{eq:Phi-tilde-P-norm}.
In addition, the following identity 
\begin{align*}
\bm{A}\bm{A}^{\top} & =\bm{\Sigma}^{\frac{1}{2}}\widetilde{\bm{\Phi}}^{\top}\left(\Ind-\gamma\bm{P}_{\bm{D}_{\mu}}\right)\widetilde{\bm{\Phi}}\bm{\Sigma}\widetilde{\bm{\Phi}}^{\top}\left(\Ind-\gamma\bm{P}_{\bm{D}_{\mu}}^{\top}\right)\widetilde{\bm{\Phi}}\bm{\Sigma}^{\frac{1}{2}}
\end{align*}
allows us to bound 
\begin{align*}
\big\|\bm{\Sigma}^{-\frac{1}{2}}\bm{A}\bm{A}^{\top}\bm{\Sigma}^{-\frac{1}{2}}\big\| & =\big\|\widetilde{\bm{\Phi}}^{\top}\big(\Ind-\gamma\bm{P}_{\bm{D}_{\mu}}\big)\widetilde{\bm{\Phi}}\bm{\Sigma}\widetilde{\bm{\Phi}}^{\top}\big(\Ind-\gamma\bm{P}_{\bm{D}_{\mu}}^{\top}\big)\widetilde{\bm{\Phi}}\big\|\\
 & \leq\big\|\Ind-\gamma\bm{P}_{\bm{D}_{\mu}}\big\|^{2}\big\|\widetilde{\bm{\Phi}}\big\|^{4}\left\Vert \bm{\Sigma}\right\Vert \\ 
 &=\|\Ind-\gamma\bm{P}_{\bm{D}_{\mu}}\|^{2}\|\bm{\Sigma}\|\\
 & \leq\big(1+\gamma\big\|\bm{P}_{\bm{D}_{\mu}}\big\|\big)^{2}\|\bm{\Sigma}\|\leq4\|\bm{\Sigma}\|,
\end{align*}
where the last line makes use of Lemma \ref{lemma:D-P-D-Phi}. This
essentially tells us that 
\begin{align*}
\bm{0}\preceq\bm{\Sigma}^{-\frac{1}{2}}\bm{A}\bm{A}^{\top}\bm{\Sigma}^{-\frac{1}{2}} & \preceq4\|\bm{\Sigma}\|\Ind
\end{align*}
\[
\Longrightarrow\qquad\bm{A}\bm{A}^{\top}\preceq4\|\bm{\Sigma}\|\bm{\Sigma}.
\]
Putting the preceding bounds together implies that: for any $0<\eta<\frac{1-\gamma}{4\left\Vert \bm{\Sigma}\right\Vert }$
one has 
\begin{align*}
\bm{0}\preceq\left(\Ind-\eta\bm{A}\right)\left(\Ind-\eta\bm{A}^{\top}\right) & =\Ind-\eta(\bm{A}+\bm{A}^{\top})+\eta^{2}\bm{A}\bm{A}^{\top}\\
 & \preceq\Ind-2\eta(1-\gamma)\bm{\Sigma}+4\eta^{2}\|\bm{\Sigma}\|\bm{\Sigma}\\
 & =\Ind-\left\{ 2\eta(1-\gamma)-4\eta^{2}\|\bm{\Sigma}\|\right\} \bm{\Sigma}\\
 & \preceq\Ind-\eta(1-\gamma)\bm{\Sigma}\\
 & \preceq\left(1-\eta(1-\gamma)\lambda_{\min}(\bm{\Sigma})\right)\Ind,
\end{align*}
thus indicating that 
\begin{align*}
\big\|\Ind-\eta\bm{A}\big\|&\leq\sqrt{\left\Vert \left(\Ind-\eta\bm{A}\right)\left(\Ind-\eta\bm{A}^{\top}\right)\right\Vert } \\ 
& \leq\sqrt{1-\eta(1-\gamma)\lambda_{\min}(\bm{\Sigma})} \\ 
&\leq1-\frac{1}{2}\eta(1-\gamma)\lambda_{\min}(\bm{\Sigma}).
\end{align*}

\paragraph{Proof of Eqn.~\eqref{eq:Sigma-norm-bound}} 
For any unit vector $\bm{u}$, the assumption $\max_{s}\|\bm{\phi}(s)\|_{2}\leq1$
guarantees that 
\[
\|\bm{\Phi}\bm{u}\|_{\infty} \leq \max_{s} | \bm{\phi}(s)^{\top}\bm{u} | \leq  \max_{s}\|\bm{\phi}(s)\|_{2} \|\bm{u}\|_2 \leq 1,
\]
where in the last inequality we have used Cauchy-Schwartz inequality. 
Consequently, for any unit vector $\bm{u}$, by H\"{o}lder's inequality,
\[
\bm{u}^{\top}\bm{\Phi}^{\top}\bm{D}_{\mu}\bm{\Phi}\bm{u}\leq\|\bm{\Phi}\bm{u}\|_{\infty}\cdot\bm{1}^{\top}\bm{D}_{\mu}\bm{1}\leq1,
\]
thus proving that $\|\bm{\Sigma}\|\leq1$. This immediately implies
that $\|\bm{\Sigma}^{-1}\| \geq 1/\|\bm{\Sigma}\|\geq1.$

\paragraph{Proof of Eqn.~\eqref{eq:Sigma-A-b-UB}}

Finally, we observe that 
\begin{align*}
\big\|\bm{\Sigma}^{-\frac{1}{2}}\bm{b}\big\|_{2} & =\big\|\bm{\Sigma}^{-\frac{1}{2}}\bm{\Phi}^{\top}\bm{D}_{\mu}^{\frac{1}{2}}\bm{D}_{\mu}^{\frac{1}{2}}\bm{r}\big\|_{2} \\ 
&\leq\big\|\bm{\Sigma}^{-\frac{1}{2}}\bm{\Phi}^{\top}\bm{D}_{\mu}^{\frac{1}{2}}\big\|\cdot\big\|\bm{D}_{\mu}^{\frac{1}{2}}\bm{r}\big\|_{2}\\ 
&\overset{(\mathrm{i})}{\leq}\big\|\bm{D}_{\mu}^{\frac{1}{2}}\bm{r}\big\|_{2}\leq1
\end{align*}
as claimed. Here, (i) follows from Lemma~\ref{lemma:D-P-D-Phi} and
(ii) holds true since $\big\|\bm{D}_{\mu}^{\frac{1}{2}}\bm{r}\big\|_{2}=\sqrt{\sum_{s}\mu(s)\big(r(s)\big)^{2}}\leq\sqrt{\sum_{s}\mu(s)}=1.$\end{proof}

The following lemmas, about the concentration of $\hat{\bm{A}}$, will be useful in our analysis.

\begin{lemma}\label{lemma:Ainv-bhat-concentration}Consider any $0<\delta<1$,
and suppose that $T\gtrsim \log \big(\frac{d}{\delta}\big)$.
Then the vector $\bm{b}$ defined in \eqref{eq:defn-bt-mean} obeys
that, with probability exceeding $1-\delta$, 
\begin{align*}
&\big\|\bm{A}^{-1}\big(\widehat{\bm{b}}-\bm{b}\big)\big\|_{\bm{\Sigma}}\\ 
 & \lesssim\sqrt{\frac{\max_{s\in\mathcal{S}}\bm{\phi}(s)^{\top}\bm{\Sigma}^{-1}\bm{\phi}(s)}{T(1-\gamma)^{2}}\log\Big(\frac{d}{\delta}\Big)}.
\end{align*}
\end{lemma}
\begin{proof} See Section \ref{subsec:Proof-of-Lemma:Ahat-A-concentration-bound}.
\end{proof}

\begin{lemma}\label{lemma:Ahat-A-concentration-bound}For any $0<\delta<1$,
it follows that $\widehat{\bm{A}}$ is invertible and that 
\begin{align*}
&\big\|\bm{\Sigma}^{1/2}\bm{A}^{-1}\big(\bm{A}-\widehat{\bm{A}}\big)\bm{\Sigma}^{-1/2}\big\| \\ 
 & \lesssim\sqrt{\frac{\max_{s}\bm{\phi}(s)^{\top}\bm{\Sigma}^{-1}\bm{\phi}(s)}{T(1-\gamma)^{2}}\log \Big(\frac{d}{\delta}\Big)}
\end{align*}
with probability at least $1-\delta$, as long as $T\geq c_{2}\max_{s}\bm{\phi}(s)^{\top}\bm{\Sigma}^{-1}\bm{\phi}(s)\log(\frac{d}{\delta})$
for some sufficiently large constant $c_{2}>0$.\end{lemma}
\begin{proof} See Section \ref{subsec:Proof-of-Lemma:Ahat-A-concentration-bound}.
\end{proof}

\section{Proof of Theorem~\ref{thm:minimax} (minimax lower bounds)\label{sec:pf-lb}}

This theorem is proved by constructing a set of  MDP instances that are hard to distinguish among each other. Based on this construction, the estimation error can be lower bounded via Fano's inequality, which reduces to control the KL-divergence between marginal likelihood functions. We start by constructing a sequence of hard MDP instances. 
\paragraph{Construction of MDP instances and their properties}
Given the state space $\cS$, define a sequence of MDP $\{\mathcal{M}_{\bm{q}}\}$ indexed by $\bm{q} \in \mathcal{Q} \subset \{q_+, q_-\}^{d-1}$ where for each $\bm{q}$, the transition kernel equals to 
\begin{align}
\label{eqn:lb-kernel}
&P_{\bm{q}}(s' \mymid s) \nonumber \\ 
&= 
\left\{ \begin{array}{lcc}
q_s\ind(s' = s) + \frac{1-q_s}{|\cS|-d+1}\ind(s' \ge d) & \text{for} & s < d; \\[0.2cm]
\frac{\gamma}{|\cS|-d+1}\ind(s' \ge d) + \frac{1-q_{s'}}{d-1}\ind(s' < d) & \text{for} & s \ge d.
\end{array}\right.
\end{align}
and the reward function equals to $r(s) = \ind(s \ge d)$.  

Here, for each $i \in [d-1]$, $q_{i}$ is taken to be either $q_+$ or $q_{-}$ where  
\begin{align*}
q_+ \defn \gamma + (1-\gamma)^2\varepsilon,\qquad\text{and}\qquad q_- \defn \gamma - (1-\gamma)^2\varepsilon.
\end{align*}
\begin{comment}
In addition, the additional constants in expression~\eqref{eqn:lb-kernel} are chosen as 
\begin{align*}
p_s \defn \frac{1-q_s}{(d-1)(1-q)}\qquad\text{for } q \defn \frac{1}{d - 1} \sum_{s = 1}^{d-1} q_s.
\end{align*}
\end{comment}
We further impose the constraint that the number of $q_+$'s and $q_-$'s in $\bm{q}$ are the same, namely, 
\begin{align}
\label{eqn:brahms}
\sum_{s=1}^{d-1} \ind(q_s = q_+) = \sum_{s=1}^{d-1} \ind(q_s = q_-)= (d-1)/2.
\end{align}
Here without loss of generality, assume $d$ is an odd number. 
With these definitions in place, it can be easily verified that the stationary distribution for $\bm{P}$ obeys  
\begin{align}
\mu(s) = \left\{ \begin{array}{lcc}
\frac{1}{2(d-1)} & \text{for} & s < d; \\[0.2cm]
\frac{1}{2(|\cS| - d + 1)} & \text{for} & s \ge d. 
\end{array}\right.
\end{align}
Moreover, suppose the feature map is taken to be   
\begin{align*}
	\bm{\phi}(s) = \bm{e}_{s \wedge d} \in \mathbb{R}^d, 
\end{align*}
then one can further verify that
\begin{align}
\theta^{\star}(d) &= V^{\star}(s) = \frac{1}{1-\gamma^2 - \sum_{i = 1}^{d-1}\frac{\gamma^2(1 - q_i)^2}{(d-1)(1-\gamma q_i)}}, \\
\theta^{\star}(i) &= V^{\star}(i) = \frac{\gamma(1 - q_i)}{1-\gamma q_i}V^{\star}(s), 
\text{ for } s \ge d \text{ and } i < d.
\end{align}
From the expressions above, we remark that, the values of $q$ and $V^\star(s)$ with $s \geq d$ are fixed for all $\bm{q} \in \mathcal{Q}$ which is ensured by the construction~\eqref{eqn:brahms}. 

\paragraph{Calculations of several key quantities} 
Based on the above constructions, let us compute several key quantities. 
To begin with, some direct algebra leads to 
\begin{align*}
\bm{\Sigma} = \bm{\Phi}^\top \bm{D}_{\bm{\mu}}\bm{\Phi} 
= \sum_{s=1}^{d-1}\frac{1}{2(d-1)} \bm{e}_s\bm{e}_s^\top + \frac{1}{2}\bm{e}_d\bm{e}_d^\top,
\end{align*}
as well as 
\begin{align*}
\bm{\phi}(s)^\top \bm{\Sigma}^{-1} \bm{\phi}(s) = \left \{\begin{array}{lcc}
2(d-1) & \text{for} & s < d;\\
2 & \text{for} & s \geq d.
\end{array}\right.
\end{align*}
As a consequence, one has 
\begin{align}
\label{eqn:phi-norm}
\max_{s}\{\bm{\phi}(s)^{\top}\bm{\Sigma}^{-1}\bm{\phi}(s)\} \asymp d.
\end{align}
Next, we move on to compute $\big\|\bm{\theta}^{\star}\big\|_{\bm{\Sigma}}.$
First notice that for $\varepsilon \le \frac{c_1 \gamma }{1-\gamma}$ with constant $c_{1}$ small enough, $(1-\gamma)^2\varepsilon \leq c_{1}\gamma (1 - \gamma)$ and hence, $1-\gamma q_+, 1-\gamma q_- \asymp 1-\gamma$,  which guarantees that $V^\star(s) \asymp \frac{1}{1-\gamma}$. 
In view of these calculations, it satisfies that  
\begin{align}
\big\|\bm{\theta}^{\star}\big\|_{\bm{\Sigma}}^2 
&= \sum_{i=1}^{d-1}  \frac{1}{2(d-1)}\theta^{\star 2}(i) + \frac{1}{2}\theta^{\star 2}(d) \\
&= \sum_{i=1}^{d-1} \frac{1}{2(d-1)} \left[\frac{\gamma(1-q_i)}{1-\gamma q_i} V^\star(s)\right]^2 + \frac{1}{2}[V^\star(s)]^2\nonumber \\
\notag &\asymp \sum_{i=1}^{d-1} \frac{1}{2(d-1)} \left[\frac{\gamma(1-\gamma)}{1-\gamma} \frac{1}{1-\gamma}\right]^2 + \frac{1}{2}\left[\frac{1}{1-\gamma}\right]^2\\
&\asymp \frac{1}{(1-\gamma)^2}. \label{eqn:theta-sigma-norm}
\end{align}

\paragraph{Application of Fano's inequality}
Armed with the properties derived above, we are ready to establish the desired lower bound. 
First notice that for $\bm{q},\bm{q}'\in \mathcal{Q}$, if at some $i \in [d-1]$, $q_i \neq {q}'_i$, then
\begin{align*}
|\theta^\star(i) - \theta'^\star(i)|&=\gamma V^\star(s) \left|\frac{1-q_i}{1-\gamma q_i} - \frac{1-{q}'_i}{1-\gamma {q}'_i}\right|\\ 
&= \gamma V^\star(s) \frac{2\varepsilon (1-\gamma)^3}{(1-\gamma q_i)(1-\gamma {q}'_i)} \\
&\gtrsim (2\gamma) \frac{1}{1-\gamma} \frac{\varepsilon (1-\gamma)^3}{(1-\gamma)^2} \gtrsim \varepsilon,
\end{align*}
where the penultimate inequality follows from $V^\star(s) \asymp \frac{1}{1-\gamma}$. 
Consequently, we can bound $\|\bm{\theta}^{\star} - {\bm{\theta}}'^{\star}\|_{\bm{\Sigma}}^2$ as 
\begin{align*}
\big\|\bm{\theta}^{\star} - {\bm{\theta}}'^{\star}\big\|_{\bm{\Sigma}}^2 &\geq \sum_{s=1}^{d-1} 
|\theta^\star(s) - {\theta}'^\star(s)|^2 \frac{1}{2(d-1)} \\ 
&\gtrsim \varepsilon^2 \frac{1}{d-1} \sum_{s=1}^{d-1} \ind(q_s \ne {q}'_s).
\end{align*}
This relation guarantees that if $\sum_{s = 1}^{d - 1} \ind(q_s \ne {q}'_s) \geq (d-1)/16$, one has
\begin{align}
\big\|\bm{\theta}^{\star} - {\bm{\theta}}'^{\star}\big\|_{\bm{\Sigma}} \gtrsim \varepsilon. 
\end{align}
In other words, if we want each $\thetastar$ to be $\varepsilon$ apart from each other, 
it is sufficient to construct a set $\mathcal{Q}$ where every $\bm{q}$ and $\bm{q}'$ are $(d-1)/16$ apart in Hamming distance. 
By virtue of the Gilbert-Varshamov lemma \citep{gilbert1952comparison}, 
there exists a set $\mathcal{Q}$ such that 
\begin{align}
	M\defn|\mathcal{Q}| \ge e^{d/16}
	\qquad \text{and} \qquad 
	\sum_{s = 1}^{d - 1} \ind(q_s \ne {q}'_s) \ge \frac{d}{16}
	\label{eq:property-Q} 
\end{align}
for any $q,q' \in \mathcal{Q}$ obeying $q \ne q'$.

The Fano method transforms the problem of estimating $\thetastar$ into an $M$-ary testing problem among the above MDPs $\{\mathbb{P}_{\bm{q}^1}, \mathbb{P}_{\bm{q}^2} \ldots, \mathbb{P}_{\bm{q}^M}\}.$
More specifically, in view of Fano's inequality (\cite{tsybakov2009introduction}), the probability of interest thus satisfies 
\begin{align}
\label{eqn:mahler}
	&\mathbb{P}\Big(\big\|\widehat{\bm{\theta}}-\bm{\theta}^{\star}\big\|_{\bm{\Sigma}}\gtrsim\varepsilon\Big) \nonumber\\
	&\geq 
	1 - \frac{1}{\log M}\Big(\frac{1}{M^2}\sum_{j,k=1}^M \mathsf{KL}(\mathbb{P}^T_{\bm{q}^j} \parallel \mathbb{P}^T_{\bm{q}^k})+\log 2\Big),
\end{align}
given $T$ independent sample pairs $\{(s_t, s_t')\}_{t=1}^{T}.$ 
To control the right hand side, we proceed by computing the KL-divergence between every $\mathbb{P}_{\bm{q}}$
and $\mathbb{P}_{\bm{q}'}.$ 
Here $\mathbb{P}_{\bm{q}}$ denotes the joint distribution of $(s, s')$ when the transition is made according to $P_{\bm{q}}(s ' \mymid s)$ (cf.~\eqref{eqn:lb-kernel}). 
More specifically, given $s \sim \mu_q$ and $s'|s \sim P_{\bm{q}}(s'|s)$, one has 
\begin{align*}
\mathbb{P}_{\bm{q}}(s,s') &= \mu(s) P(s'|s) \\ 
&= \left\{\begin{array}{ccc}
\frac{1}{2(d-1)}  q_s \ind(s' = s), & \text{for} & s<d,s'<d;\\
\frac{1-q_s}{2(d-1)(S-d+1)}, &\text{for} & s<d,s'>d;\\
\frac{1-q_{s'}}{2(d-1)(S-d+1)}, &\text{for} & s>d, s'<d;\\
\frac{\gamma}{2(S-d+1)^2}, & \text{for} & s>d,s'>d.
\end{array}\right.
\end{align*}
Recognizing the relation between the KL divergence and the $\chi^{2}$ divergence, 
$\mathsf{KL}(\mathbb{P}_{\bm{q}} \parallel \mathbb{P}_{\bm{q}'})$ satisfies 
\begin{align*}
&\mathsf{KL}(\mathbb{P}_{\bm{q}}\parallel \mathbb{P}_{\bm{q}'})\\ 
&\leq  \chi^2(\mathbb{P}_{\bm{q}'} \parallel \mathbb{P}_{\bm{q}})\\
&= \sum_{s,s'} \frac{(\mathbb{P}_{\bm{q}}(s,s')-\mathbb{P}_{\bm{q}'}(s,s'))^2}{\mathbb{P}_{\bm{q}}(s,s')}\\
&= \sum_{s<d,s'<d}  \frac{(\mathbb{P}_{\bm{q}}(s,s')-\mathbb{P}_{\bm{q}'}(s,s'))^2}{\mathbb{P}_{\bm{q}}(s,s')}\\
&+ \sum_{s<d,s'\geq d}\frac{(\mathbb{P}_{\bm{q}}(s,s')-\mathbb{P}_{\bm{q}'}(s,s'))^2}{\mathbb{P}_{\bm{q}}(s,s')}\\
&+ \sum_{s\geq d, s' < d} \frac{(\mathbb{P}_{\bm{q}}(s,s')-\mathbb{P}_{\bm{q}'}(s,s'))^2}{\mathbb{P}_{\bm{q}}(s,s')}\\
&+\sum_{s \geq d, s' \geq d} \frac{(\mathbb{P}_{\bm{q}}(s,s')-\mathbb{P}_{\bm{q}'}(s,s'))^2}{\mathbb{P}_{\bm{q}}(s,s')}\\
&= \sum_{s=1}^{d-1} \frac{1}{2(d-1)} \frac{(q_s-{q}'_s)^2}{q_s} \\ 
&+  \sum_{s<d,s'\geq d} \frac{1}{2(d-1)(S-d+1)} \frac{[(1-q_s)-(1-{q}'_s)^2]}{1-q_s} \\
&+ \sum_{s \geq d,s'< d} \frac{1}{2(d-1)(S-d+1)} \frac{[(1-q_{s'})-(1-{q}'_{s'})^2]}{1-q_{s'}} + 0 \\
&\lesssim \sum_{s=1}^{d-1} \frac{1}{2(d-1)} \frac{[2\varepsilon(1-\gamma)^2]^2}{1-\gamma} \\
&+ \sum_{s<d} \frac{1}{2(d-1)} \frac{[2\varepsilon(1-\gamma)^2]^2}{1-\gamma} \\
&+ \sum_{s'<d} \frac{1}{2(d-1)} \frac{[2\varepsilon(1-\gamma)^2]^2}{1-\gamma} \\
& \asymp \varepsilon^2(1-\gamma)^3.
\end{align*}
As a result, we have 
\begin{align} 
\mathsf{KL}\big(\mathbb{P}_{\bm{q}}^T \parallel \mathbb{P}_{\bm{q}'}^T \big) \lesssim \varepsilon^2(1-\gamma)^3T.
\end{align}
Substituting the above relation into \eqref{eqn:mahler} gives  
\begin{align*}
	\mathbb{P}\Big(\big\|\widehat{\bm{\theta}}-\bm{\theta}^{\star}\big\|_{\bm{\Sigma}}\gtrsim\varepsilon\Big) 
	\geq 
	1 - \frac{1}{d/16}\Big(c \varepsilon^2(1-\gamma)^3T + \log 2\Big).
\end{align*}
To prove Theorem~\ref{thm:minimax}, it is enough to take the above together with relations~\eqref{eqn:phi-norm} and \eqref{eqn:theta-sigma-norm}.

\section{Proofs of auxiliary lemmas and claims}

\subsection{Proof of Lemma \ref{lemma:exp-decay-sum} \label{subsec:Proof-of-Lemma-exp-decay-sum}}

Here and throughout, we denote by $\Exs_{i}[\cdot]$ the expectation
conditioned on the probability space generated by the samples $\{(s_{j},s_{j}')\}_{j\leq i}$
(more formally, $\Exs_{i}[\cdot]$ represents the expectation conditioned
on the filtration $\mathcal{F}_{i}$ --- the $\sigma$-algebra generated
by $\{(s_{j},s_{j}')\}_{j\leq i}$). It is then easy to check that
$\{(\Ind-\eta\bm{A})^{t-i-1}\widetilde{\bm{\xi}}_{i}\}$ forms a martingale
difference sequence, which motivates us to apply matrix Freedman's
inequality. 

To this end, one needs to upper bound the following two quantities
\begin{align}
&W\defn\sum_{i=l}^{u}\Exs_{i-1}\Big[\big\|\bmSigma^{1/2}(\Ind-\eta\bm{A})^{t-i-1}\bm{\xi}_{i}\big\|_{2}^{2}\ind\{\mathcal{H}_{i}\}\Big],\quad{\rm and}\nonumber \\ 
&B\defn\max_{i:l\leq i\leq u}\big\|\bmSigma^{1/2}(\Ind-\eta\bm{A})^{t-i-1}\bm{\xi}_{i}\ind\{\mathcal{H}_{i}\}\big\|_{2},
\end{align}
which we accomplish in the sequel. For notational convenience, we
set 
\begin{align}
	\alpha\coloneqq \Big(1-\frac{1}{2}\eta(1-\gamma)\lambda_{\min}(\bm{\Sigma}) \Big)^{t-u-1}.
	\label{defn:alpha}
\end{align}

\paragraph{Control of $W$}

Direct calculations yield 
\begin{align}
W= & ~\sum_{i=l}^{u}\Exs_{i-1}\Big[\bm{\xi}_{i}^{\top}(\Ind-\eta\bm{A}^{\top})^{t-i-1}\bm{\Sigma}(\Ind-\eta\bm{A})^{t-i-1}\bm{\xi}_{i}\ind\{\mathcal{H}_{i}\}\Big]\nonumber \\
\le & ~\sum_{i=l}^{u}\big\|(\Ind-\eta\bm{A}^{\top})^{t-i-1}\bm{\Sigma}(\Ind-\eta\bm{A})^{t-i-1}\big\|\nonumber \\ 
\cdot&\Exs_{i-1}\big[\|\bm{\xi}_{i}\|_{2}^{2}\ind\{\mathcal{H}_{i}\}\big]\nonumber \\
\stackrel{\mathrm{(i)}}{\le} & ~\sum_{i=l}^{u}\|\bm{\Sigma}\|\left(1-\frac{1}{2}\eta(1-\gamma)\lambdamin\right)^{2t-2i-2}\nonumber \\ 
\cdot & 2\max_{i:l\leq i\leq u}\Big\{\Exs_{i-1}\big[\|(\bm{A}_{i}-\bm{A})\bmtheta_{i}\|_{2}^{2}\ind\{\mathcal{H}_{i}\}\big]\nonumber \\ 
&+\Exs_{i-1}\big[\|\bm{b}_{i}-\bm{b}\|_{2}^{2}\big]\Big\}\nonumber \\
\stackrel{\mathrm{(ii)}}{\le} & ~\frac{4\|\bm{\Sigma}\|\alpha^{2}}{\eta(1-\gamma)\lambdamin}
\cdot  \max_{i:l\leq i\leq u}\Big\{\Exs_{i-1}\big[\|(\bm{A}_{i}-\bm{A})\bmtheta_{i}\|_{2}^{2}\ind\{\mathcal{H}_{i}\}\big]\nonumber \\ 
&+\Exs_{i-1}\big[\|\bm{b}_{i}-\bm{b}\|_{2}^{2}\big]\Big\},\label{eqn:bound-vt}
\end{align}
where (i) follows from the property~\eqref{eq:I-eta-A-spectral-norm-bound} (together with the assumption  $\eta<(1-\gamma)/(4\|\bm{\Sigma}\|)$)
and the elementary inequality $\|\bm{a}+\bm{b}\|_{2}^{2}\leq2\|\bm{a}\|_{2}^{2}+2\|\bm{b}\|_{2}^{2}$,
and (ii) uses the elementary upper bound for the sum of geometric
series as well as the definition \eqref{defn:alpha} of $\alpha$. 

We then turn attention to $\Exs_{i-1}\big[\|(\bm{A}_{i}-\bm{A})\bmtheta_{i}\|_{2}^{2}\ind\{\mathcal{H}_{i}\}\big]$
and $\Exs_{i-1}\big[\|\bm{b}_{i}-\bm{b}\|_{2}^{2}\big]$. First, given
that $\Exs_{i-1}[\bm{A}_{i}\bmtheta_{i}\ind\{\mathcal{H}_{i}\}]=\bm{A}\bmtheta_{i}\ind\{\mathcal{H}_{i}\}$,
one can derive 
\begin{align}
 & \Exs_{i-1}\big[\|(\bm{A}_{i}-\bm{A})\bmtheta_{i}\|_{2}^{2}\ind\{\mathcal{H}_{i}\}\big]\nonumber \\ 
 &\leq\Exs_{i-1}\big[\ltwo{\bm{A}_{i}\bmtheta_{i}}^{2}\ind\{\mathcal{H}_{i}\}\big]\nonumber \\
 & =\Exs_{i-1}\big[\bmtheta_{i}^{\top}\big(\bm{\phi}(s_{i})-\gamma\bm{\phi}(s_{i}')\big)\bm{\phi}(s_{i})^{\top}\nonumber \\ 
 &\qquad \quad ~\bm{\phi}(s_{i})\big(\bm{\phi}(s_{i})-\gamma\bm{\phi}(s_{i}')\big)^{\top}\bmtheta_{i}\ind\{\mathcal{H}_{i}\}\big]\nonumber \\
 & \leq\max_{s}\ltwo{\bm{\phi}(s)}^{2}\cdot\Exs_{i-1}\big[\bmtheta_{i}^{\top}\big(\bm{\phi}(s_{i})-\gamma\bm{\phi}(s_{i}')\big)\nonumber \\ 
 &\qquad \qquad \qquad \qquad \qquad \big(\bm{\phi}(s_{i})-\gamma\bm{\phi}(s_{i}')\big)^{\top}\bmtheta_{i}\ind\{\mathcal{H}_{i}\}\big]\nonumber \\
 & \overset{\mathrm{(i)}}{\leq}2\max_{s}\ltwo{\bm{\phi}(s)}^{2}\Big(\Exs_{i-1}\left[\bmtheta_{i}^{\top}\bm{\phi}(s_{i})\bm{\phi}(s_{i})^{\top}\bmtheta_{i}\ind\{\mathcal{H}_{i}\}\right]\nonumber \\ 
 &\qquad \qquad \qquad \qquad  +\gamma^{2}\Exs_{i-1}\left[\bmtheta_{i}^{\top}\bm{\phi}(s_{i}')\bm{\phi}(s_{i}')^{\top}\bmtheta_{i}\ind\{\mathcal{H}_{i}\}\right]\Big)\nonumber \\
 & \overset{\mathrm{(ii)}}{=}2\max_{s}\ltwo{\bm{\phi}(s)}^{2}\Big(\Exs_{i-1}\big[\bmtheta_{i}^{\top}\bm{\Sigma}\bmtheta_{i}\ind\{\mathcal{H}_{i}\}\big]\nonumber \\ 
 &\qquad \qquad \qquad \qquad +\gamma^{2}\Exs_{i-1}\big[\bmtheta_{i}^{\top}\bm{\Sigma}\bmtheta_{i}\ind\{\mathcal{H}_{i}\}\big]\Big)\nonumber \\
 & \overset{\mathrm{(iii)}}{\leq}4\lsigma{\bmtheta_{i}}^{2}\ind\{\mathcal{H}_{i}\}\leq4\big(\lsigma{\bm{\theta}^{\star}}+\lsigma{\bm{\Delta}_{i}}\big)^{2}\ind\{\mathcal{H}_{i}\}\nonumber \\
 & \leq4(\lsigma{\bm{\theta}^{\star}}+R)^{2},\label{eqn:Aiadagio2}
\end{align}
where (i) relies on the elementary inequality $(\bm{a}+\bm{b})(\bm{a}+\bm{b})^{\top}\preceq2\bm{a}\bm{a}^{\top}+2\bm{b}\bm{b}^{\top}$,
(ii) follows from the definition~\eqref{eq:defn-Sigma} of $\bm{\Sigma}$
and the fact that $s_i,s_i'\sim\mu$ in this case, (iii) holds due to
the assumption $\max_{s}\|\bm{\phi}(s)\|_{2}\leq1$, and the last
inequality results from the definition \eqref{eqn:defn-xi-tilde}
of the event $\mathcal{H}_{i}$. 
Similarly, one can derive 
\begin{align}
\Exs_{i-1}[\|\bm{b}_{i}-\bm{b}\|_{2}^{2}]\leq\Exs_{i-1}[\ltwo{\bm{b}_{i}}^{2}]=\Exs_{i-1}\big[\ltwo{\bm{\phi}(s_{i})r(s_{i})}^{2}\big]\leq1,\label{eqn:Aiadagio3}
\end{align}
where the last inequality holds since $\max_{s}\|\bm{\phi}(s)\|_{2}\leq1$
and $\max_{s}|r(s)|\leq1$. Substitution into \eqref{eqn:bound-vt}
yields 
\begin{align}
W\leq\frac{4\kappa}{\eta(1-\gamma)}\alpha^{2}\Big\{4(\lsigma{\bm{\theta}^{\star}}+R)^{2}+1\Big\}\eqqcolon W_{\max}.\label{eqn:bound-vt-2}
\end{align}

\paragraph{Control of $B$}

By definition of $B$, one can write 
%
\begin{align}
B=& \max_{i:l\leq i\leq u}\big\|\bm{\Sigma}^{\frac{1}{2}}(\Ind-\eta\bm{A})^{t-i-1}\bm{\xi}_{i}\big\|_{2}\ind\{\mathcal{H}_{i}\}\nonumber \\ 
=&\max_{i:l\leq i\leq u}\big\|\bm{\Sigma}^{\frac{1}{2}}(\Ind-\eta\bm{A})^{t-i-1}\bm{\Sigma}^{\frac{1}{2}}\bm{\Sigma}^{-\frac{1}{2}}\bm{\xi}_{i}\big\|_{2}\ind\{\mathcal{H}_{i}\}\nonumber \\
\le & \|\bm{\Sigma}\|\max_{i:l\leq i\leq u}\big\|\Ind-\eta\bm{A}\big\|^{t-i-1}\cdot\max_{i:l\leq i\leq u}\big\|\bm{\Sigma}^{-\frac{1}{2}}\bm{\xi}_{i}\big\|_{2}\ind\{\mathcal{H}_{i}\}\nonumber \\
\leq & \alpha\|\bm{\Sigma}\|\max_{i:l\leq i\leq u}\Big\{\big\|\bm{\Sigma}^{-\frac{1}{2}}(\bm{A}_{i}-\bm{A})\bmtheta_{i}\big\|_{2}\ind\{\mathcal{H}_{i}\}\nonumber \\ 
&\qquad \qquad \qquad +\big\|\bm{\Sigma}^{-\frac{1}{2}}(\bm{b}_{i}-\bm{b})\big\|_{2}\Big\},
\label{eq:Bt-crude-bound-1}
\end{align}
where the last step results from \eqref{eq:I-eta-A-spectral-norm-bound}
(with the restriction that $\eta<(1-\gamma)/(4\|\bm{\Sigma}\|)$) and the definition \eqref{defn:alpha} of $\alpha$.
It then suffices to control the two terms on the right-hand side of
\eqref{eq:Bt-crude-bound-1}. To begin with, we have 
\begin{align*}
&\big\|\bm{\Sigma}^{-\frac{1}{2}}(\bm{A}_{i}-\bm{A})\bmtheta_{i}\big\|_{2} \\ 
& \leq\big\|\bm{\Sigma}^{-\frac{1}{2}}(\bm{A}_{i}-\bm{A})\bm{\Sigma}^{-\frac{1}{2}}\big\|\|\bmtheta_{i}\|_{\bm{\Sigma}}\\
 & \leq\Big(\|\bm{\Sigma}^{-\frac{1}{2}}\bm{A}_{i}\bm{\Sigma}^{-\frac{1}{2}}\|+\|\bm{\Sigma}^{-\frac{1}{2}}\bm{A}\bm{\Sigma}^{-\frac{1}{2}}\|\Big)\big(\lsigma{\bm{\theta}^{\star}}+\lsigma{\bm{\Delta}_{i}}\big).
\end{align*}
Recall from (\ref{eq:useful-for-TD}) that $\|\bm{\Sigma}^{-\frac{1}{2}}\bm{A}_{i}\bm{\Sigma}^{-\frac{1}{2}}\|\leq2\max_{s}\|\bm{\Sigma}^{-1/2}\bm{\phi}(s)\|_{2}^{2}$, and similarly $\|\bm{\Sigma}^{-\frac{1}{2}}\bm{A}\bm{\Sigma}^{-\frac{1}{2}}\|\leq2\max_{s}\|\bm{\Sigma}^{-1/2}\bm{\phi}(s)\|_{2}^{2}$.
We then have 
\begin{align}\label{eq:Ai-sup}
&\big\|\bm{\Sigma}^{-\frac{1}{2}}(\bm{A}_{i}-\bm{A})\bmtheta_{i}\big\|_{2}\nonumber \\ 
&\leq4\max_{s}\big\{\bm{\phi}(s)^{\top}\bmSigma^{-1}\bm{\phi}(s)\big\}\big(\lsigma{\bm{\theta}^{\star}}+\lsigma{\bm{\Delta}_{i}}\big).
\end{align}
Regarding the second term of \eqref{eq:Bt-crude-bound-1}, direct
calculations give 
\begin{align}\label{eq:bi-sup}
&\|\bm{\Sigma}^{-\frac{1}{2}}(\bm{b}_{i}-\bm{b})\|_{2}^{2} \nonumber \\ 
& \leq2\|\bm{\Sigma}^{-\frac{1}{2}}\bm{b}_{i}\|_{2}^{2}+2\|\bm{\Sigma}^{-\frac{1}{2}}\bm{b}\|_{2}^{2}\nonumber \\ 
&=2\big\|\bm{\Sigma}^{-\frac{1}{2}}\bm{\phi}(s_{i})r(s_{i})\big\|_{2}^{2}+2\big\|\bm{\Sigma}^{-\frac{1}{2}}\mathbb{E}_{s\sim\mu}\big[\bm{\phi}(s)r(s)\big]\big\|_{2}^{2}\nonumber \\
 & \leq4\max_{s}\big\{\bm{\phi}(s)^{\top}\bmSigma^{-1}\bm{\phi}(s)\big\}\max_{s}|r(s)|^{2}\nonumber \\ 
 &\leq4\max_{s}\big\{\bm{\phi}(s)^{\top}\bmSigma^{-1}\bm{\phi}(s)\big\}.
\end{align}
Substituting the preceding two bounds into \eqref{eq:Bt-crude-bound-1},
we arrive at 
\begin{align}
B & \leq4\alpha\|\bm{\Sigma}\|\Big(\max_{s}\big\{\bm{\phi}(s)^{\top}\bmSigma^{-1}\bm{\phi}(s)\big\}\max_{i:i<t}\big(\lsigma{\bm{\theta}^{\star}}+\lsigma{\bm{\Delta}_{i}}\big)\nonumber \\ 
&\qquad \qquad \quad \ind\{\mathcal{H}_{i}\}+\sqrt{\max_{s}\{\bm{\phi}(s)^{\top}\bmSigma^{-1}\bm{\phi}(s)\}}\Big)\nonumber \\
 & \leq4\alpha\|\bm{\Sigma}\|\Big(\max_{s}\big\{\bm{\phi}(s)^{\top}\bmSigma^{-1}\bm{\phi}(s)\big\}\big(\lsigma{\bm{\theta}^{\star}}+R\big)\nonumber \\ 
 &\qquad \qquad \quad +\sqrt{\max_{s}\{\bm{\phi}(s)^{\top}\bmSigma^{-1}\bm{\phi}(s)\}}\Big)\nonumber \\
 & \leq4\alpha\|\bm{\Sigma}\|\max_{s}\big\{\bm{\phi}(s)^{\top}\bm{\Sigma}^{-1}\bm{\phi}(s)\big\}\big(\|\bm{\theta}^{\star}\|_{\bm{\Sigma}}+R+1\big)\nonumber \\
 & \leq4\alpha\|\bm{\Sigma}\|\|\bm{\Sigma}^{-1}\|\big(\|\bm{\theta}^{\star}\|_{\bm{\Sigma}}+R+1\big)\nonumber \\ 
 &=4\kappa\alpha\big(\|\bm{\theta}^{\star}\|_{\bm{\Sigma}}+R+1\big)\eqqcolon B_{\max}.\label{eqn:bound-bt}
\end{align}
Here, the last line follows from the assumption $\max\|\bm{\phi}(s)\|_{2}\leq1$,
while the second to last inequality holds since $\max_{s}\big\{\bm{\phi}(s)^{\top}\bm{\Sigma}^{-1}\bm{\phi}(s)\big\}\geq1$
(cf.~(\ref{eq:useful-for-TD-2})). 

\paragraph{Invoking matrix Freedman's inequality}


Equipped with the above bounds \eqref{eqn:bound-vt-2} and \eqref{eqn:bound-bt},
we are ready to apply Freedman's inequality \cite[Corollary 1.3]{tropp2011freedman} (or a version in \cite[Section A]{li2021q}), which asserts that 
\begin{align}
&\Big\|\sum_{i=0}^{t-1}(\Ind-\eta\bm{A})^{t-i-1}\bm{\widetilde{\xi}}_{i}\Big\|_{\bm{\Sigma}} \nonumber \\ 
& \leq2\sqrt{W_{\max}\log\frac{2dT}{\delta}}+\frac{4}{3}B_{\max}\log\frac{2dT}{\delta}\nonumber \\
 & =\alpha\cdot\Bigg\{ 2\sqrt{\frac{4\kappa}{\eta(1-\gamma)}\Big\{4(\|\bm{\theta}^{\star}\|_{\bm{\Sigma}}+R)^{2}+1\Big\}\log\frac{2dT}{\delta}}\nonumber \\ 
 &+\frac{16\kappa}{3}\big(\|\bm{\theta}^{\star}\|_{\bm{\Sigma}}+R+1\big)\log\frac{2dT}{\delta}\Bigg\} \nonumber \\
 & \leq16(1-\frac{1}{2}\eta(1-\gamma)\lambda_{\min}(\bm{\Sigma}))^{t-u-1}\nonumber \\ 
 &\quad (\|\bm{\theta}^{\star}\|_{\bm{\Sigma}}+R+1)\sqrt{\frac{\kappa\log\frac{2dT}{\delta}}{\eta(1-\gamma)}}\label{eqm:induction-termtwo-tilde}
\end{align}
holds with probability at least $1-\delta/T$, provided that $0<\eta\leq\frac{1}{\kappa(1-\gamma)\log\frac{2dT}{\delta}}$.
Here in the last line, we identify $\alpha$ with $(1-\frac{1}{2}\eta(1-\gamma)\lambda_{\min}(\bm{\Sigma}))^{t-u-1}$.
The proof is completed by observing that any $0<\eta\leq\frac{1-\gamma}{\kappa\log\frac{2dT}{\delta}}$
satisfies the two requirements $0<\eta\leq\frac{1}{\kappa(1-\gamma)\log\frac{2dT}{\delta}}$
and $\eta<(1-\gamma)/(4\|\bm{\Sigma}\|)$ (given that $\|\bm{\Sigma}\|\leq 1$ according to \eqref{eq:Sigma-norm-bound}). 

\subsection{Proof of the inequalities~\eqref{eqn:deltazeroterm} and \eqref{eqn:schubert}}
\label{sec:schubert}

\paragraph{Proof of the inequality~\eqref{eqn:deltazeroterm}}

Combining the triangle inequality with the definition~(\ref{eq:defn-Ait})
ensures that 
\begin{align}
&\big\|\bm{A}_{0}^{(T+1)}\bm{\Delta}_{0} \big\|_{\bm{\Sigma}}\nonumber \\ 
 & \le\|\bm{A}^{-1}\bm{\Delta}_{0}\|_{\bm{\Sigma}}+\|\bm{A}^{-1}(\Ind-\eta\bm{A})^{T+1}\bm{\Delta}_{0}\|_{\bm{\Sigma}}\nonumber \\
 & =\big\|\bm{\Sigma}^{\frac{1}{2}}\bm{A}^{-1}\bm{\Sigma}^{\frac{1}{2}}\bm{\Sigma}^{-1}\bm{\Sigma}^{\frac{1}{2}}\bm{\Delta}_{0}\big\|_{2}\nonumber \\ 
 &+\|\bm{\Sigma}^{\frac{1}{2}}\bm{A}^{-1}\bm{\Sigma}^{\frac{1}{2}}\bm{\Sigma}^{-\frac{1}{2}}(\Ind-\eta\bm{A})^{T+1}\bm{\Sigma}^{-\frac{1}{2}}\bm{\Sigma}^{\frac{1}{2}}\bm{\Delta}_{0}\|_{2}\nonumber \\
 & \leq \big\|\bm{\Sigma}^{\frac{1}{2}}\bm{A}^{-1}\bm{\Sigma}^{\frac{1}{2}}\big\|\cdot\big\|\bm{\Sigma}^{-1}\big\|\cdot\big\|\bm{\Delta}_{0}\big\|_{\bm{\Sigma}}\nonumber \\ 
 &+\big\|\bm{\Sigma}^{\frac{1}{2}}\bm{A}^{-1}\bm{\Sigma}^{\frac{1}{2}}\big\|\cdot\big\|\bm{\Sigma}^{-\frac{1}{2}}\big\|^{2}\cdot \| \Ind-\eta\bm{A}\|^{T+1} \cdot\big\|\bm{\Delta}_{0}\big\|_{\bm{\Sigma}}\nonumber \\
 & \leq\frac{\|\bm{\Sigma}^{-1}\|}{1-\gamma}\left\{ 1+\left(1-\frac{1}{2}\eta(1-\gamma)\lambdamin\right)^{T+1}\right\} \big\|\bm{\Delta}_{0}\|_{\bm{\Sigma}}\nonumber \\
 & \leq\frac{2\|\bm{\Sigma}^{-1}\|}{1-\gamma}\big\|\bm{\Delta}_{0}\|_{\bm{\Sigma}}
\end{align}
as claimed. 
Here, the second to last step follows from \eqref{eq:Sigma-A-Sigma-UB-12}
and \eqref{eq:I-eta-A-spectral-norm-bound}, provided that $\eta\leq(1-\gamma)/(4\|\bm{\Sigma}\|)$. 


\paragraph{Proof of the inequality~\eqref{eqn:schubert}}

Again, the triangle inequality together with the definition~(\ref{eq:defn-Ait}) 
yields
\begin{align}
&\Big\|\sum_{i=0}^{T-1}\bm{A}_{i}^{(T)}\bm{\xi}_{i}\Big\|_{\bm{\Sigma}}\nonumber \\ 
 & \le\Big\|\sum_{i=0}^{T-1}\bm{A}^{-1}\bm{\xi}_{i}\Big\|_{\bm{\Sigma}}+\Big\|\sum_{i=0}^{T-1}\bm{A}^{-1}(\Ind-\eta\bm{A})^{T-i}\bm{\xi}_{i}\Big\|_{\bm{\Sigma}}\nonumber \\
 & \le\Big\|\bm{A}^{-1}\sum_{i=0}^{T-1}(\bm{A}_{i}-\bm{A})\bmtheta_{i}\Big\|_{\bm{\Sigma}}+\Big\|\bm{A}^{-1}\sum_{i=0}^{T-1}(\bm{b}_{i}-\bm{b})\Big\|_{\bm{\Sigma}}\nonumber \\ 
 &+\Big\|\sum_{i=0}^{T-1}\bm{A}^{-1}(\Ind-\eta\bm{A})^{T-i}\bm{\xi}_{i}\Big\|_{\bm{\Sigma}},\label{eq:sum-Ai-xi-i-Bound1}
\end{align}
leaving us with three terms to handle. Here in the second line, we
substitute in the definition of $\bm{\xi}_{i}$ \eqref{eqn:defn-xi}. 
\begin{itemize}
\item 

The second term of \eqref{eq:sum-Ai-xi-i-Bound1} can be bounded by
Lemma~\ref{lemma:Ainv-bhat-concentration}, which asserts that
\begin{align}
&\frac{1}{T}\Big\|\sum_{i=0}^{T-1}\bm{A}^{-1}\big(\bm{b}_{i}-\bm{b}\big)\Big\|_{\bm{\Sigma}}\nonumber \\ 
&\lesssim\sqrt{\frac{\max_{s}\bm{\phi}(s)^{\top}\bm{\Sigma}^{-1}\bm{\phi}(s)}{T(1-\gamma)^{2}}\log\Big(\frac{d}{\delta}\Big)}
	\label{eq:Ainv-bi-b-Sigma-bound}
\end{align}
holds with probability at least $1-\delta$, 
as long as $T\gtrsim\log\frac{d}{\delta}$. 
\item For the third term of \eqref{eq:sum-Ai-xi-i-Bound1}, invoking the
property~\eqref{eq:Sigma-A-Sigma-UB-12} again yields 
\begin{align}
&\Big\|\sum_{i=0}^{T-1}\bm{A}^{-1}(\Ind-\eta\bm{A})^{T-i}\bm{\xi}_{i}\Big\|_{\bm{\Sigma}} \nonumber \\ 
& =\Big\|\bm{\Sigma}^{\frac{1}{2}}\bm{A}^{-1}\bm{\Sigma}^{\frac{1}{2}}\bm{\Sigma}^{-1}\sum_{i=0}^{T-1}\bm{\Sigma}^{\frac{1}{2}}(\Ind-\eta\bm{A})^{T-i}\bm{\xi}_{i}\Big\|_{2}\nonumber \\
 & \le\big\|\bm{\Sigma}^{\frac{1}{2}}\bm{A}^{-1}\bm{\Sigma}^{\frac{1}{2}}\big\|\cdot\big\|\bm{\Sigma}^{-1}\big\|\cdot\Big\|\sum_{i=0}^{T-1}\bm{\Sigma}^{\frac{1}{2}}(\Ind-\eta\bm{A})^{T-i}\bm{\xi}_{i}\Big\|_{2}\nonumber \\
 & \leq\frac{\|\bm{\Sigma}^{-1}\|}{1-\gamma}\Big\|\sum_{i=0}^{T-1}(\Ind-\eta\bm{A})^{T-i}\bm{\xi}_{i}\Big\|_{\bm{\Sigma}}.\label{eq:Ainv-A-sum-UB1}
\end{align}
Repeating the same analysis as in Step~3 to see that 
\begin{align}
\Big\|\sum_{i=0}^{T-1}(\Ind-\eta\bm{A})^{T-i}\bm{\xi}_{i}\Big\|_{\bm{\Sigma}}\leq16(2\|\bm{\theta}^{\star}\|_{\bm{\Sigma}}+3)\sqrt{\frac{\kappa\log\frac{2dT}{\delta}}{\eta(1-\gamma)}}\label{eqm:induction-termtwo-new}
\end{align}
with probability at least $1-\delta$. 
Substitution into \eqref{eq:Ainv-A-sum-UB1} leads to 
\begin{align}
&\Big\|\sum_{i=0}^{T-1}\bm{A}^{-1}(\Ind-\eta\bm{A})^{T-i}\bm{\xi}_{i}\Big\|_{\bm{\Sigma}} \nonumber \\ 
& \leq16\big(2\|\bm{\theta}^{\star}\|_{\bm{\Sigma}}+3\big)\big\|\bm{\Sigma}^{-1}\big\|\sqrt{\frac{\kappa\log\frac{2dT}{\delta}}{\eta(1-\gamma)^{3}}}.\label{eq:Ainv-A-sum-UB2}
\end{align}
\item It then boils down to bounding the first term of \eqref{eq:sum-Ai-xi-i-Bound1}.
In light of \eqref{eq:Delta-t-two-cases}, we decompose it as follows
\begin{align}
&\Big\|\frac{1}{T}\sum_{i=0}^{T-1}\bm{\Sigma}^{\frac{1}{2}}\bm{A}^{-1}(\bm{A}_{i}-\bm{A})\bm{\theta}_{i}\Big\|_{2}\nonumber \\ 
&\leq\Big\|\frac{1}{T}\sum_{i=0}^{\widetilde{t}_{\mathsf{seg}}-1}\bm{\Sigma}^{\frac{1}{2}}\bm{A}^{-1}(\bm{A}_{i}-\bm{A})\bm{\theta}_{i}\Big\|_{2}\nonumber \\ 
&+\Big\|\frac{1}{T}\sum_{i=\widetilde{t}_{\mathsf{seg}}}^{T-1}\bm{\Sigma}^{\frac{1}{2}}\bm{A}^{-1}(\bm{A}_{i}-\bm{A})\bm{\theta}_{i}\Big\|_{2}.
\end{align}
Bounding these terms requires the following lemma, whose proof is deferred to Section \ref{subsec:Proof-of-Lemma-TD-main-term}.
\begin{lemma}\label{lemma:TD-main-term}
Fix any $R>0$ and define a collection of auxiliary random vectors
for $0\leq i\leq T-1$
\begin{align}
\bm{\theta}_{i}'\coloneqq \bm{\theta}_{i}\ind\{\mathcal{H}_{i}\},\qquad\mathcal{H}_{i}\coloneqq\big\{\|\bm{\Delta}_{i}\|_{\bm{\Sigma}}\leq R\},\label{eq:defn-theta-i-tilde}
\end{align}
Then for any indices $(l,u,t)$ obeying $0\leq l\leq u\leq T-1$, one has with probability
at least $1-\delta$ that
\begin{align}
 & \Big\|\frac{1}{u-l+1}\sum_{i=l}^{u}\bm{\Sigma}^{\frac{1}{2}}\bm{A}^{-1}(\bm{A}_{i}-\bm{A}){\bm{\theta}}_{i}'\Big\|_{2}\nonumber \\ 
 &\leq\frac{16\big(\|\bm{\theta}^{\star}\|_{\bm{\Sigma}}+R\big)}{1-\gamma}\sqrt{\frac{\max_{s}\bm{\phi}(s)^{\top}\bm{\Sigma}^{-1}\bm{\phi}(s)\log\frac{2d}{\delta}}{u-l+1}}
	\label{eq:sum-Ainv-Ai-theta-bound-00}
\end{align}
provided that 
\begin{align*}
u-l+1\geq\frac{4{\max_{s}\bm{\phi}(s)^{\top}\bm{\Sigma}^{-1}\bm{\phi}(s)}\log\frac{2d}{\delta}}{9}.
\end{align*}
\end{lemma}
Apply Lemma~\ref{lemma:TD-main-term} with $R=R_{0},l=0$
and $u=t_{\mathsf{seg}}'-1$ to obtain with probability of at least $1-\delta$ that
\begin{align*}
&\Big\|\frac{1}{t_{\mathsf{seg}}'}\sum_{i=0}^{t_{\mathsf{seg}}'-1}\bm{\Sigma}^{\frac{1}{2}}\bm{A}^{-1}(\bm{A}_{i}-\bm{A})\bm{\theta}_{i}\Big\|_{2} \\ 
& =\Big\|\frac{1}{t_{\mathsf{seg}}'}\sum_{i=0}^{t_{\mathsf{seg}}'-1}\bm{\Sigma}^{\frac{1}{2}}\bm{A}^{-1}(\bm{A}_{i}-\bm{A})\bm{\theta}_{i}\ind\{\|\bm{\Delta}_{i}\|\leq R_{0}\}\Big\|_{2}\\
 & \leq\frac{16\big(\|\bm{\theta}^{\star}\|_{\bm{\Sigma}}+R_{0}\big)}{1-\gamma}\sqrt{\frac{\max_{s}\bm{\phi}(s)^{\top}\bm{\Sigma}^{-1}\bm{\phi}(s)\log\frac{2d}{\delta}}{t_{\mathsf{seg}}'}} ,
\end{align*}
as long as $t_{\mathsf{seg}}'\geq \frac{4 \| \bm{\Sigma}^{-1} \|\log\frac{2d}{\delta}}{9} \geq \frac{4{\max_{s}\bm{\phi}(s)^{\top}\bm{\Sigma}^{-1}\bm{\phi}(s)}\log\frac{2d}{\delta}}{9}$.
Here, the
identity holds since $\|\bm{\Delta}_{i}\|_{\bm{\Sigma}}\leq R_{0}$
for $i\leq t_{\mathsf{seg}}'-1$ with probability of at least $1-\delta$. Similarly, invoke Lemma~\ref{lemma:TD-main-term}
with $R=32\sqrt{\frac{\eta\kappa\log\frac{2dT}{\delta}}{1-\gamma}}(\|\bm{\theta}^{\star}\|_{\bm{\Sigma}}+2),l=t_{\mathsf{seg}}'$
and $u=T-1$ to obtain with probability of at least $1-\delta$ that
\begin{align*}
&\Big\|\frac{1}{T-t_{\mathsf{seg}}'}\sum_{i=t_{\mathsf{seg}}'}^{T-1}\bm{\Sigma}^{\frac{1}{2}}\bm{A}^{-1}(\bm{A}_{i}-\bm{A})\bm{\theta}_{i}\Big\|_{2} \\ 
& \leq\frac{16\big(\|\bm{\theta}^{\star}\|_{\bm{\Sigma}}+32\sqrt{\frac{\eta\kappa\log\frac{2dT}{\delta}}{1-\gamma}}(\|\bm{\theta}^{\star}\|_{\bm{\Sigma}}+2)\big)}{1-\gamma}\\ 
&\qquad \sqrt{\frac{\max_{s}\bm{\phi}(s)^{\top}\bm{\Sigma}^{-1}\bm{\phi}(s)\log\frac{2d}{\delta}}{T-t_{\mathsf{seg}}'}}\\
 & \leq\frac{16\big(1.5\|\bm{\theta}^{\star}\|_{\bm{\Sigma}}+2\big)}{1-\gamma}\sqrt{\frac{\max_{s}\bm{\phi}(s)^{\top}\bm{\Sigma}^{-1}\bm{\phi}(s)\log\frac{2d}{\delta}}{T-t_{\mathsf{seg}}'}}
\end{align*}
provided that $T-t_{\mathsf{seg}}'\geq\frac{4{\max_{s}\bm{\phi}(s)^{\top}\bm{\Sigma}^{-1}\bm{\phi}(s)}\log\frac{2d}{\delta}}{9}$.
Here, the last inequality arises from the relation 
\[
32\sqrt{\frac{\eta\kappa\log\frac{2dT}{\delta}}{1-\gamma}}(\|\bm{\theta}^{\star}\|_{\bm{\Sigma}}+2)\leq0.5\|\bm{\theta}^{\star}\|_{\bm{\Sigma}}+2,
\]
which is an immediate consequence of the assumption that $\frac{\eta\kappa\log\frac{2dT}{\delta}}{1-\gamma}$
is sufficiently small. Therefore,
\begin{align}
&\Big\|\frac{1}{T}\sum_{i=0}^{T-1}\bm{\Sigma}^{\frac{1}{2}}\bm{A}^{-1}(\bm{A}_{i}-\bm{A})\bm{\theta}_{i}\Big\|_{2}\nonumber \\ 
 & \leq\frac{32\big(\|\bm{\theta}^{\star}\|_{\bm{\Sigma}}+\sqrt{\frac{t_{\mathsf{seg}}'}{T}}R_{0}+1\big)}{1-\gamma}\sqrt{\frac{\max_{s}\bm{\phi}(s)^{\top}\bm{\Sigma}^{-1}\bm{\phi}(s)\log\frac{2d}{\delta}}{T}} .\label{eq:sum-Ainv-Ai-theta-bound}
\end{align}
\end{itemize}
Combining the preceding bounds \eqref{eq:Ainv-bi-b-Sigma-bound},
\eqref{eq:Ainv-A-sum-UB2} and \eqref{eq:sum-Ainv-Ai-theta-bound}
with \eqref{eq:sum-Ai-xi-i-Bound1}, we reach the conclusion that with probability of at least $1-\delta$,
\begin{align*}
\Big\|\sum_{i=0}^{T-1}\bm{A}_{i}^{(t)}\bm{\xi}_{i}\Big\|_{\bm{\Sigma}}
 & \asymp\Bigg\{ \sqrt{\frac{\max_{s}\bm{\phi}(s)^{\top}\bm{\Sigma}^{-1}\bm{\phi}(s)\log\frac{2d}{\delta}}{T(1-\gamma)^{2}}}\nonumber \\ 
 &\qquad +\frac{\big\|\bm{\Sigma}^{-1}\big\|}{T}\sqrt{\frac{\kappa\log\frac{dT}{\delta}}{\eta(1-\gamma)^{3}}}\Bigg\} \big(\|\bm{\theta}^{\star}\|_{\bm{\Sigma}}+1\big),
\end{align*}
as long as $T \geq t_{\mathsf{seg}}'\kappa\|\bm{\Delta}_{0}\|_{\bm{\Sigma}}^{2}$,
where we use the definition \eqref{eq:bound_delta} of $R_{0}$.
It thus establishes  the inequality~\eqref{eqn:schubert}.

\subsection{Proof of Lemma \ref{lemma:psi-norm}}\label{subsec:proof-of-lemma-psi-norm}
We first decompose $\bm{\Psi}$ into 
\begin{align*}
\bm{\Psi}&=\left[\begin{array}{cc}
\bm{I}-\alpha\widetilde{\bm{A}}^{\top}\widetilde{\bm{\Sigma}}^{-1}\widetilde{\bm{A}} & \bm{0}\\
\bm{0} & \bm{I}-\beta\widetilde{\bm{\Sigma}}
\end{array}\right] \\ 
&+\left[\begin{array}{cc}
\bm{0} & -\frac{1}{\kappa}\alpha\gamma\bm{\Pi}^{\top}\\
-\kappa\alpha(1-\gamma\widetilde{\bm{\Sigma}}^{-1}\bm{\Pi})\widetilde{\bm{A}}^{\top}\widetilde{\bm{\Sigma}}^{-1}\widetilde{\bm{A}} & -\alpha\gamma(\bm{I}-\gamma\widetilde{\bm{\Sigma}}^{-1}\bm{\Pi})\bm{\Pi}^{\top}
\end{array}\right].
\end{align*}
Then the triangle inequality together with the properties of the operator
norm tells us that 
\begin{align*}
\|\bm{\Psi}\| & \leq\max\{\|\bm{I}-\alpha\widetilde{\bm{A}}^{\top}\widetilde{\bm{\Sigma}}^{-1}\widetilde{\bm{A}}\|,\|\bm{I}-\beta\widetilde{\bm{\Sigma}}\|\}+\|\frac{1}{\kappa}\alpha\gamma\bm{\Pi}^{\top}\|\\
 & \quad+\|\kappa\alpha(\bm{I}-\gamma\widetilde{\bm{\Sigma}}^{-1}\bm{\Pi})\widetilde{\bm{A}}^{\top}\widetilde{\bm{\Sigma}}^{-1}\widetilde{\bm{A}}\|+\|\alpha\gamma(\bm{I}-\gamma\widetilde{\bm{\Sigma}}^{-1}\bm{\Pi})\bm{\Pi}^{\top}\|.
\end{align*}
Note that by definition of $\lambda_{w}$ and $\lambda_{w}$, we find 
\begin{align*}
\|\bm{I}-\alpha\widetilde{\bm{A}}^{\top}\widetilde{\bm{\Sigma}}^{-1}\widetilde{\bm{A}}\| & \leq1-\alpha\lambda_{\theta},\\
\|\bm{I}-\beta\widetilde{\bm{\Sigma}}\| & \leq1-\beta\lambda_{w}.
\end{align*}
In addition, some direct algebra suggests 
\begin{align*}
&\left\|\frac{1}{\kappa}\alpha\gamma\bm{\Pi}^{\top}\right\|  \leq\frac{\alpha\gamma\rho_{\max}}{\kappa},\\
&\|\kappa\alpha(\bm{I}-\gamma\widetilde{\bm{\Sigma}}^{-1}\bm{\Pi})\widetilde{\bm{A}}^{\top}\widetilde{\bm{\Sigma}}^{-1}\widetilde{\bm{A}}\|\\ 
 & \leq\kappa\alpha(1+\gamma\lambda_{\Sigma}\rho_{\max})\lambda_{\Sigma}(2\rho_{\max})^{2},\quad \text{and} \\
&\|\alpha\gamma(\bm{I}-\gamma\widetilde{\bm{\Sigma}}^{-1}\bm{\Pi})\bm{\Pi}^{\top}\| \leq\alpha\gamma(\rho_{\max}+\gamma\lambda_{\Sigma}\rho_{\max}^{2}).
\end{align*}
In summary, as long as 
\begin{align*}
\alpha\gamma(\rho_{\max}+\gamma\lambda_{\Sigma}\rho_{\max}^{2}) & \ll\beta\lambda_{w},\\
\frac{\alpha\gamma\rho_{\max}}{\kappa}+\kappa\alpha(1+\gamma\lambda_{\Sigma}\rho_{\max})\lambda_{\Sigma}(2\rho_{\max})^{2} & \ll \sqrt{\alpha\lambda_{\theta}\beta\lambda_{w}}, 
\end{align*}
one has 
\[
\|\bm{\Psi}\|\leq1-\frac{1}{2}\alpha\lambda_{\theta}.
\]

\subsection{Proof of Lemma \ref{lemma:exp-decay-sum-xt}}\label{subsec:proof-of-lemma-exp-decay-sum-xt}
Using the same notation of $\mathbb{E}_{i-1}$ as in Section \ref{subsec:Proof-of-Lemma-exp-decay-sum}, we observe that $\{\bm{\Psi}^{t-i-1} \widetilde{\bm{\zeta}}_i\}$ forms a martingale difference sequence. Furthermore, define
\begin{align}
&\widetilde{W}:= \sum_{i=0}^{t-1} \mathbb{E}_{i-1} \left[\left\|\bm{\Psi}^{t-i-1} \bm{{\zeta}}_i\right\|_{\widetilde{\bm{\Sigma}}}^2 \ind\left\{\widetilde{\mathcal{H}}_i\right\}\right], 
\quad \text{and} \nonumber \\ 
&\widetilde{B}:= \max_{i:0 \leq i \leq t-1} \left\| \bm{\Psi}^{t-i-1} \bm{\zeta}_i \ind\left\{\widetilde{\mathcal{H}}_i\right\} \right\|_{\widetilde{\bm{\Sigma}}}.
\end{align}
In order to bound $\widetilde{W}$ and $\widetilde{B}$, we will firstly need to bound the norm of $\widetilde{\bm{\zeta}}_i$, as is shown in the following paragraph.

\paragraph{Controlling the norm of $\bm{\zeta}_i$} We firstly observe that since $\|\bm{\phi}(s)\|_2 \leq 1 $ and $r(s) \leq 1$ for all $s \in \mathcal{S}$, with similar logic as \eqref{eqn:Aiadagio2}, \eqref{eqn:Aiadagio3}, \eqref{eq:Ai-sup} and \eqref{eq:bi-sup}, the following bounds hold true:
\begin{itemize}
\item For any $\mathcal{F}_{i-1}$-measurable $\widetilde{\bm{\theta}}_{i} \in \mathbb{R}^d$, the norm of $(\widetilde{\bm{A}}_i-\widetilde{\bm{A}}) \widetilde{\bm{\theta}}_{i}$ is bounded by
\begin{align}
&\mathbb{E}_{i-1}\left\|(\widetilde{\bm{A}}_i-\widetilde{\bm{A}}) \widetilde{\bm{\theta}}_{i}\right\|_{2}^2 \leq 4 \rho_{\max}^2 \left(\|\widetilde{\bm{\theta}}^\star\|_{\widetilde{\bm{\Sigma}}}^2 + \|\widetilde{\bm{\Delta}}_{i}\|_{\widetilde{\bm{\Sigma}}}^2 \right), \quad \text{and}\\
&\left\|\widetilde{\bm{\Sigma}}^{-1/2}(\widetilde{\bm{A}}_i-\widetilde{\bm{A}}) \widetilde{\bm{\theta}}_{i}\right\|_{2} \nonumber \\ 
&\leq 4\rho_{\max} \max_s \left\{\bm{\phi}(s) \widetilde{\bm{\Sigma}}^{-1} \bm{\phi}(s)\right\}\left(\|\widetilde{\bm{\theta}}^\star\|_{\widetilde{\bm{\Sigma}}} + \|\widetilde{\bm{\Delta}}_{i}\|_{\widetilde{\bm{\Sigma}}} \right);
\end{align}
\item For any $\mathcal{F}_{i-1}$-measurable $\bm{z}_{i} \in \mathbb{R}^d$, the norm of $(\bm{\Pi}_i-\bm{\Pi})^\top \bm{z}_{i}$ is bounded by
\begin{align}
&\mathbb{E}_{i-1}\left\|\left(\bm{\Pi}_i - \bm{\Pi}\right)^\top \bm{z}_{i}\right\|_2^2 \leq  \rho_{\max}^2 \|\bm{z}_{i}\|_{\widetilde{\bm{\Sigma}}}^2, \quad \text{and}\\
& \left\|\widetilde{\bm{\Sigma}}^{-1/2}\left(\bm{\Pi}_i - \bm{\Pi}\right)^\top \bm{z}_{i}\right\|_2 \nonumber \\ 
&\leq 2 \rho_{\max} \max_s \left\{\bm{\phi}(s) \widetilde{\bm{\Sigma}}^{-1} \bm{\phi}(s)\right\} \|\bm{z}_{i}\|_{\widetilde{\bm{\Sigma}}};
\end{align}
\item For any $\mathcal{F}_{i-1}$-measurable $\bm{z}_{i} \in \mathbb{R}^d$, the norm of $(\widetilde{\bm{\Sigma}}_i-\widetilde{\bm{\Sigma}}) \bm{z}_{i}$ is bounded by
\begin{align}
&\mathbb{E}_{i-1}\left\|\left(\widetilde{\bm{\Sigma}}_i - \widetilde{\bm{\Sigma}}\right) \bm{z}_{i}\right\|_2^2 \leq \|\bm{z}_{i}\|_{\widetilde{\bm{\Sigma}}}^2, \quad \text{and}\\
&\left\|\widetilde{\bm{\Sigma}}^{-1/2}\left(\widetilde{\bm{\Sigma}}_i - \widetilde{\bm{\Sigma}}\right) \bm{z}_{i}\right\|_2 \leq 2\max_s \left\{\bm{\phi}(s) \widetilde{\bm{\Sigma}}^{-1} \bm{\phi}(s)\right\} \|\bm{z}_{i}\|_{\widetilde{\bm{\Sigma}}};
\end{align}
\item The norm of $\widetilde{\bm{b}}_i-\widetilde{\bm{b}}$ is bounded by
\begin{align}
&\mathbb{E}_{i-1}\left\|\widetilde{\bm{b}}_i-\widetilde{\bm{b}} \right\| _2^2 \leq \rho_{\max}^2,\quad \text{and}\\
&\left\|\widetilde{\bm{\Sigma}}^{-1/2}\left(\widetilde{\bm{b}}_i-\widetilde{\bm{b}}\right)\right\|_2^2 \leq 4\rho_{\max}^2 \max_s \left\{\bm{\phi}(s) \widetilde{\bm{\Sigma}}^{-1} \bm{\phi}(s)\right\}.
\end{align}
\end{itemize}

Therefore, by triangle inequality, the norm of $\bm{\nu_i}$ can be bounded by
\begin{align}\label{eq:omega-bound}
&\mathbb{E}_{i-1} \left\|\bm{\nu}_i\right\|_2^2 \lesssim \rho_{\max}^2 \left[\left(\|\widetilde{\bm{\theta}}^\star\|_{\widetilde{\bm{\Sigma}}}^2 + \|\widetilde{\bm{\Delta}}_{i}\|_{\widetilde{\bm{\Sigma}}}^2 \right)+1 + \gamma^2 \|\bm{w}_{i}\|_{\widetilde{\bm{\Sigma}}}^2\right],
\end{align}
and 
\begin{align}
& \left\|\widetilde{\bm{\Sigma}}^{-1/2} \bm{\nu}_i\right\|_2 \nonumber \\ 
&\lesssim \rho_{\max} \max_s \left\{\bm{\phi}(s) \widetilde{\bm{\Sigma}}^{-1} \bm{\phi}(s)\right\} \Big[\left(\|\widetilde{\bm{\theta}}^\star\|_{\widetilde{\bm{\Sigma}}} + \|\widetilde{\bm{\Delta}}_{i}\|_{\widetilde{\bm{\Sigma}}} \right)+ \nonumber \\ 
&\qquad \qquad \qquad \qquad \qquad \qquad \quad \gamma \|\bm{w}_{i}\|_{\widetilde{\bm{\Sigma}}} + 1\Big];
\end{align}
similarly, the norm of $\bm{\eta}_i$ can be bounded by
\begin{align}\label{eq:eta-bound}
&\mathbb{E}_{i-1} \left\|\bm{\eta}_i\right\|_2^2  \lesssim \rho_{\max}^2 \left[\left(\|\widetilde{\bm{\theta}}^\star\|_{\widetilde{\bm{\Sigma}}}^2 + \|\widetilde{\bm{\Delta}}_{i}\|_{\widetilde{\bm{\Sigma}}}^2 \right)+1 \right]+  \|\bm{w}_{i}\|_{\widetilde{\bm{\Sigma}}}^2, \\
&\text{and} \\ 
&\left\|\widetilde{\bm{\Sigma}}^{-1/2} \bm{\eta_i} \right\|_2  
\lesssim  \max_s \left\{\bm{\phi}(s) \widetilde{\bm{\Sigma}}^{-1} \bm{\phi}(s)\right\} \nonumber \\
&\qquad \quad \quad \quad \quad \left\{\rho_{\max}\left[\left(\|\widetilde{\bm{\theta}}^\star\|_{\widetilde{\bm{\Sigma}}} + \|\widetilde{\bm{\Delta}}_{i}\|_{\widetilde{\bm{\Sigma}}} \right)+ 1\right]+ \|\bm{w}_{i}\|_{\widetilde{\bm{\Sigma}}}\right\}. 
\end{align}

By combining \eqref{eq:omega-bound} and \eqref{eq:eta-bound} with the definition of $\bm{\zeta}_i$ \eqref{eq:defn-zeta}, we obtain the following bound:
\begin{align}
&\mathbb{E}_{i-1}\|\bm{\zeta}_i\|_2^2\nonumber \\ 
 &\lesssim  \alpha^2 \mathbb{E}_{i-1}\|\bm{\nu}_i\|_2^2 + \varkappa^2 \alpha^2 \|\bm{I} - \gamma \widetilde{\bm{\Sigma}}^{-1}\bm{\Pi}\|^2 \mathbb{E}_{i-1}\|\bm{\nu}_i\|_2^2 \nonumber \\ 
 &+ \varkappa^2 \beta^2 \mathbb{E}_{i-1} \|\bm{\eta}_i\|_2^2 \nonumber \\
&\lesssim \alpha^2 \left(1+\varkappa^2(1+\gamma \lambda_{\Sigma}\rho_{\max})^2\right) \cdot \rho_{\max}^2 \nonumber \\ 
&\quad \left[4\left(\|\widetilde{\bm{\theta}}^\star\|_{\widetilde{\bm{\Sigma}}}^2 + \|\widetilde{\bm{\Delta}}_{i}\|_{\widetilde{\bm{\Sigma}}}^2 \right)+1 + \gamma^2 \|\bm{w}_{i}\|_{\widetilde{\bm{\Sigma}}}^2\right] \nonumber \\
&+ \varkappa^2\beta^2 \cdot \left\{\rho_{\max}^2 \left[\left(\|\widetilde{\bm{\theta}}^\star\|_{\widetilde{\bm{\Sigma}}}^2 + \|\widetilde{\bm{\Delta}}_{i}\|_{\widetilde{\bm{\Sigma}}}^2 \right)+1 \right]+  \|\bm{w}_{i}\|_{\widetilde{\bm{\Sigma}}}^2\right\} \nonumber \\
&\lesssim \varkappa^2 \beta^2 \rho_{\max}^2 \left(\|\widetilde{\bm{\theta}}^\star\|_{\widetilde{\bm{\Sigma}}}^2 + \frac{1}{\varkappa^2}\|\bm{x}_{i}\|_{\widetilde{\bm{\Sigma}}}^2 + 1\right),
\end{align}
and
\begin{align}
&\left\|\widetilde{\bm{\Sigma}}^{-1/2}\bm{\zeta}_i\right\|_2 \nonumber \\ 
&\lesssim \alpha \|\widetilde{\bm{\Sigma}}^{-1/2}\bm{\nu}_i\|_2 + \alpha \varkappa \|\bm{I} - \gamma \widetilde{\bm{\Sigma}}^{-1}\bm{\Pi}\|\|\widetilde{\bm{\Sigma}}^{-1/2}\bm{\nu}_i\|_2 \nonumber \\ 
&+ \varkappa \beta \left\|\widetilde{\bm{\Sigma}}^{-1/2} \eta_i \right\|_2 \nonumber \\
&\lesssim \alpha\left(1+\varkappa (1+\gamma \lambda_{\Sigma}\rho_{\max})\right)\cdot \rho_{\max} \max_s \left\{\bm{\phi}(s) \widetilde{\bm{\Sigma}}^{-1} \bm{\phi}(s)\right\}  \nonumber \\
& \cdot \left[2\left(\|\widetilde{\bm{\theta}}^\star\|_{\widetilde{\bm{\Sigma}}} + \|\widetilde{\bm{\Delta}}_{i}\|_{\widetilde{\bm{\Sigma}}} \right)+ \gamma \|\bm{w}_{i}\|_{\widetilde{\bm{\Sigma}}} + 1\right]  \nonumber \\
&+\varkappa \beta \cdot  \max_s \left\{\bm{\phi}(s) \widetilde{\bm{\Sigma}}^{-1} \bm{\phi}(s)\right\} \nonumber \\ 
&\cdot \left\{\rho_{\max}\left[\left(\|\widetilde{\bm{\theta}}^\star\|_{\widetilde{\bm{\Sigma}}} + \|\widetilde{\bm{\Delta}}_{i}\|_{\widetilde{\bm{\Sigma}}} \right)+ 1\right]+ \|\bm{w}_{i}\|_{\widetilde{\bm{\Sigma}}}\right\} \nonumber \\
&\lesssim \varkappa \beta \max_s \left\{\bm{\phi}(s) \widetilde{\bm{\Sigma}}^{-1} \bm{\phi}(s)\right\} \left(\|\widetilde{\bm{\theta}}^\star\|_{\widetilde{\bm{\Sigma}}} + \frac{2}{\varkappa} \|\bm{x}_{i}\|_{\widetilde{\bm{\Sigma}}} + 1\right)
\end{align}

\paragraph{Control of $\widetilde{W}$ and $\widetilde{B}$} With the norm of $\widetilde{\bm{\zeta}}_i$ bounded, we can apply similar techniques as in equations \eqref{eqn:bound-vt}, \eqref{eqn:bound-vt-2}, \eqref{eq:Bt-crude-bound-1} and \eqref{eqn:bound-bt} of Section~\ref{subsec:Proof-of-Lemma-exp-decay-sum} to construct the following bound for $\widetilde{W}$:
\begin{align}
\widetilde{W}&\leq \|\widetilde{\bm{\Sigma}}\|\sum_{i=0}^{t-1} \|\bm{\Psi}^{t-i-1}\|^2 \cdot \mathbb{E}_{i-1}\left[\|\bm{\zeta}_i\|_2^2 \ind\left\{\widetilde{\mathcal{H}}_i\right\}\right] \nonumber \\
&\lesssim \|\widetilde{\bm{\Sigma}}\|\sum_{i=0}^{t-1} (1-\frac{1}{2}\alpha \lambda_{\theta})^{2t-2i-2} \nonumber \\ 
&\cdot \varkappa^2 \beta^2 \rho_{\max}^2 (2\|\widetilde{\bm{\theta}}^\star\|_{\widetilde{\bm{\Sigma}}} + 2\widetilde{R} + 1)^2 \nonumber \\
&\lesssim \frac{\|\widetilde{\bm{\Sigma}}\|}{\alpha \lambda_{\theta}} \varkappa^2 \beta^2 \rho_{\max}^2 (\|\widetilde{\bm{\theta}}^\star\|_{\widetilde{\bm{\Sigma}}} + \frac{1}{\varkappa}\widetilde{R} + 1)^2,
\end{align}
and the following bound for $\widetilde{B}$:
\begin{align}
\widetilde{B}&\leq \|\widetilde{\bm{\Sigma}}\| \max_{i:0 \leq i \leq t-1} \left\| \widetilde{\bm{\Sigma}}^{-1/2} \bm{\zeta}_i \ind\left\{\mathcal{H}_i\right\} \right\|_2 \nonumber \\
&\lesssim  \|\widetilde{\bm{\Sigma}}\|\varkappa \beta \rho_{\max} \max_s \left\{\bm{\phi}(s) \widetilde{\bm{\Sigma}}^{-1} \bm{\phi}(s)\right\}(\|\widetilde{\bm{\theta}}^\star\|_2 + \widetilde{R}  + 1)\nonumber \\ 
&=:\widetilde{B}_{\max}.
\end{align}

\paragraph{Invoking the matrix Freedman's inequality} With $\widetilde{W}$ and $\widetilde{B}$ bounded, we again invoke the matrix Freedman's inequality \cite[Corollary 1.3]{tropp2011freedman} to assert that
\begin{align}
&\left\|\sum_{i=0}^{t-1} \bm{\Psi}^{t-i-1} \widetilde{\bm{\zeta}}_i \right\|_2 \nonumber \\ 
&\leq 2 \sqrt{\widetilde{W}_{\max} \log \frac{2dT}{\delta}}+\frac{4}{3}\widetilde{B}_{\max} \log \frac{2dT}{\delta}\nonumber\\
&\lesssim \sqrt{\frac{\|\widetilde{\bm{\Sigma}}\|}{\alpha \lambda_{\theta}}\log \frac{2dT}{\delta}}\varkappa \beta \rho_{\max} (\|\widetilde{\bm{\theta}}^\star\|_{\widetilde{\bm{\Sigma}}} + \frac{1}{\varkappa}\widetilde{R}  + 1)
\end{align}
holds with probability at least $1-\delta/T$, proveded that $0 < \alpha < \frac{1}{\lambda_{\theta}\lambda_{\Sigma}^2 \|\widetilde{\bm{\Sigma}}\|\log \frac{2dT}{\delta}}$.

\subsection{Proof of Lemma \ref{lemma:Ainv-bhat-concentration} and Lemma \ref{lemma:Ahat-A-concentration-bound}}\label{subsec:Proof-of-Lemma:Ahat-A-concentration-bound}

\paragraph{Proof of Lemma \ref{lemma:Ahat-A-concentration-bound}: controlling $\|\bm{\Sigma}^{\frac{1}{2}}\bm{A}^{-1}\big(\bm{A}-\widehat{\bm{A}}\big)\bm{\Sigma}^{-\frac{1}{2}}\|$}

We intend to invoke the matrix Bernstein inequality to establish the advertised bound
\cite{Tropp:2015:IMC:2802188.2802189}. Note that 
\begin{align}
\bm{\Sigma}^{\frac{1}{2}}\bm{A}^{-1}\big(\bm{A}-\widehat{\bm{A}}\big)\bm{\Sigma}^{-\frac{1}{2}}=\frac{1}{T}\sum_{t=0}^{T-1}\underset{=:\bm{Z}_{t}}{\underbrace{\bm{\Sigma}^{\frac{1}{2}}\bm{A}^{-1}\big(\bm{A}-\bm{A}_{t}\big)\bm{\Sigma}^{-\frac{1}{2}}}}.\label{eq:defn-Zt-Lemma-Anorm}
\end{align}
In order to control it, we need to first control the following two
quantities: 
\begin{align*}
&v:=\max_{t}\left\{ \max\left\{ \big\|\mathbb{E}\left[\bm{Z}_{t}\bm{Z}_{t}^{\top}\right]\big\|,\,\big\|\mathbb{E}\left[\bm{Z}_{t}^{\top}\bm{Z}_{t}\right]\big\|\right\} \right\} \quad\text{and} \\ 
& B:=\max_{t}\|\bm{Z}_{t}\|.
\end{align*}

\medskip{}
\textit{Step 1: controlling $\big\|\mathbb{E}\left[\bm{Z}_{t}\bm{Z}_{t}^{\top}\right]\big\|$.}
Towards this, we first make the observation that 
\begin{align}
&\mathbb{E}\left[\bm{Z}_{t}\bm{Z}_{t}^{\top}\right] \nonumber \\ 
& =\mathbb{E}\left[\bm{\Sigma}^{\frac{1}{2}}\bm{A}^{-1}\big(\bm{A}-\bm{A}_{t}\big)\bm{\Sigma}^{-1}\big(\bm{A}-\bm{A}_{t}\big)^{\top}\big(\bm{A}^{\top}\big)^{-1}\bm{\Sigma}^{\frac{1}{2}}\right]\nonumber \\
 & \preceq\mathbb{E}\left[\bm{\Sigma}^{\frac{1}{2}}\bm{A}^{-1}\bm{A}_{t}\bm{\Sigma}^{-1}\bm{A}_{t}^{\top}\big(\bm{A}^{\top}\big)^{-1}\bm{\Sigma}^{\frac{1}{2}}\right]\nonumber \\
 & =\mathop{\mathbb{E}}\limits _{s\sim\mu,\,s'\sim P(\cdot|s)}\Big[\bm{\Sigma}^{\frac{1}{2}}\bm{A}^{-1}\bm{\phi}(s)\big(\bm{\phi}(s)-\gamma\bm{\phi}(s')\big)^{\top}\bm{\Sigma}^{-1}\nonumber \\ 
 &\qquad \qquad \qquad \big(\bm{\phi}(s)-\gamma\bm{\phi}(s')\big)\bm{\phi}(s)^{\top}\big(\bm{A}^{\top}\big)^{-1}\bm{\Sigma}^{\frac{1}{2}}\Big]\nonumber \\
 & \preceq\max_{s,s'}\left\{ \big(\bm{\phi}(s)-\gamma\bm{\phi}(s')\big)^{\top}\bm{\Sigma}^{-1}\big(\bm{\phi}(s)-\gamma\bm{\phi}(s')\big)\right\} \nonumber \\ 
 &\cdot \mathop{\mathbb{E}}\limits _{s\sim\mu}\left[\bm{\Sigma}^{\frac{1}{2}}\bm{A}^{-1}\bm{\phi}(s)\bm{\phi}(s)^{\top}\big(\bm{A}^{\top}\big)^{-1}\bm{\Sigma}^{\frac{1}{2}}\right]\nonumber \\
 & \preceq\max_{s,s'}\left\{ 2\bm{\phi}(s)^{\top}\bm{\Sigma}^{-1}\bm{\phi}(s)+2\gamma^{2}\bm{\phi}(s')^{\top}\bm{\Sigma}^{-1}\bm{\phi}(s')\right\} \nonumber \\ 
 &\cdot\left\{ \bm{\Sigma}^{\frac{1}{2}}\bm{A}^{-1}\bm{\Sigma}\big(\bm{A}^{\top}\big)^{-1}\bm{\Sigma}^{\frac{1}{2}}\right\} \nonumber \\
 & \preceq\frac{4\max_{s}\bm{\phi}(s)^{\top}\bm{\Sigma}^{-1}\bm{\phi}(s)}{(1-\gamma)^{2}}\Ind,\label{eq:E-Zt-ZtT-1}
\end{align}
where the second line holds since $\mathbb{E}\big[(\bm{M}-\mathbb{E}[\bm{M}])(\bm{M}-\mathbb{E}[\bm{M}])^{\top}\big]\preceq\mathbb{E}[\bm{M}\bm{M}^{\top}]$
for any random matrix $\bm{M}$, the second to last inequality holds
since $(\bm{a}-\bm{b})^{\top}\bm{\Sigma}^{-1}(\bm{a}-\bm{b})\leq2\bm{a}^{\top}\bm{\Sigma}^{-1}\bm{a}+2\bm{b}^{\top}\bm{\Sigma}^{-1}\bm{b}$,
and the last inequality comes from the assumption $\gamma<1$ and
Lemma \ref{lemma:Sigma-inv-A-lower-bound}.


\medskip{}
\textit{Step 2: controlling $\big\|\mathbb{E}\left[\bm{Z}_{t}^{\top}\bm{Z}_{t}\right]\big\|$.}
Similarly, one can obtain 
\begin{align*}
&\mathbb{E}\left[\bm{Z}_{t}^{\top}\bm{Z}_{t}\right]\\ 
 & =\mathbb{E}\left[\bm{\Sigma}^{-\frac{1}{2}}\big(\bm{A}-\bm{A}_{t}\big)^{\top}(\bm{A}^{\top})^{-1}\bm{\Sigma}\bm{A}^{-1}\big(\bm{A}-\bm{A}_{t}\big)\bm{\Sigma}^{-\frac{1}{2}}\right]\\
 & \preceq\mathbb{E}\left[\bm{\Sigma}^{-\frac{1}{2}}\bm{A}_{t}^{\top}(\bm{A}^{\top})^{-1}\bm{\Sigma}\bm{A}^{-1}\bm{A}_{t}\bm{\Sigma}^{-\frac{1}{2}}\right]\\
 & =\mathbb{E}\Big[\bm{\Sigma}^{-\frac{1}{2}}\big(\bm{\phi}(s_{t})-\gamma\bm{\phi}(s_{t}')\big)\bm{\phi}(s_{t})^{\top}(\bm{A}^{\top})^{-1}\bm{\Sigma}\bm{A}^{-1}\\ 
 &\quad \quad \bm{\phi}(s_{t})\big(\bm{\phi}(s_{t})-\gamma\bm{\phi}(s_{t}')\big)^{\top}\bm{\Sigma}^{-\frac{1}{2}}\Big]\\
 & \preceq\max_{s}\Big\{\bm{\phi}(s)^{\top}(\bm{A}^{\top})^{-1}\bm{\Sigma}\bm{A}^{-1}\bm{\phi}(s)\Big\}\\ 
 &\cdot \mathbb{E}\left[\bm{\Sigma}^{-\frac{1}{2}}\big(\bm{\phi}(s_{t})-\gamma\bm{\phi}(s_{t}')\big)\big(\bm{\phi}(s_{t})-\gamma\bm{\phi}(s_{t}')\big)^{\top}\bm{\Sigma}^{-\frac{1}{2}}\right]\\
 & \preceq\max_{s}\Big\{\bm{\phi}(s)^{\top}(\bm{A}^{\top})^{-1}\bm{\Sigma}\bm{A}^{-1}\bm{\phi}(s)\Big\}\\ 
 &\cdot2\mathbb{E}\left[\bm{\Sigma}^{-\frac{1}{2}}\big(\bm{\phi}(s_{t})\bm{\phi}(s_{t})^{\top}+\bm{\phi}(s_{t}')\bm{\phi}(s_{t}')^{\top}\big)\bm{\Sigma}^{-\frac{1}{2}}\right]\\
 & \preceq4\max_{s}\Big\{\bm{\phi}(s)^{\top}(\bm{A}^{\top})^{-1}\bm{\Sigma}\bm{A}^{-1}\bm{\phi}(s)\Big\}\Ind.
\end{align*}
Here, the second to last bound follows from the elementary inequality
$(\bm{a}-\bm{b})(\bm{a}-\bm{b})^{\top}\preceq2\bm{a}\bm{a}^{\top}+2\bm{b}\bm{b}^{\top}$
and the assumption $\gamma<1$, whereas the last line makes use of
the facts $s_{t}\sim\mu$, $s_{t}'\sim\mu$ and the definition (\ref{eq:defn-Sigma})
of $\bm{\Sigma}$. It then boils down to upper bounding $\max_{s}\big\{\bm{\phi}(s)^{\top}(\bm{A}^{\top})^{-1}\bm{\Sigma}\bm{A}^{-1}\bm{\phi}(s)\big\}$,
which can be accomplished as follows 
\begin{align*}
&\bm{\phi}(s)^{\top}(\bm{A}^{\top})^{-1}\bm{\Sigma}\bm{A}^{-1}\bm{\phi}(s)\\ 
 & =\bm{\phi}(s)^{\top}\bm{\Sigma}^{-\frac{1}{2}}\left\{ \bm{\Sigma}^{\frac{1}{2}}(\bm{A}^{\top})^{-1}\bm{\Sigma}\bm{A}^{-1}\bm{\Sigma}^{\frac{1}{2}}\right\} \bm{\Sigma}^{-\frac{1}{2}}\bm{\phi}(s)\\
 & \leq\big\|\bm{\Sigma}^{-\frac{1}{2}}\bm{\phi}(s)\big\|_{2}^{2}\cdot\big\|\bm{\Sigma}^{\frac{1}{2}}(\bm{A}^{\top})^{-1}\bm{\Sigma}\bm{A}^{-1}\bm{\Sigma}^{\frac{1}{2}}\big\|\\
 & \leq\frac{\max_{s}\bm{\phi}(s)^{\top}\bm{\Sigma}^{-1}\bm{\phi}(s)}{(1-\gamma)^{2}}.
\end{align*}
Here, the last line arises from Lemma \ref{lemma:Sigma-inv-A-lower-bound}.
Putting the above bounds together yields

\noindent 
\begin{align}
\mathbb{E}\left[\bm{Z}_{t}^{\top}\bm{Z}_{t}\right]\preceq\frac{4\max_{s}\bm{\phi}(s)^{\top}\bm{\Sigma}^{-1}\bm{\phi}(s)}{(1-\gamma)^{2}}\Ind.\label{eq:E-Zt-ZtT-2}
\end{align}

\begin{comment}
Further, 
\[
\bm{\Sigma}^{\frac{1}{2}}(\bm{A}^{\top})^{-1}\bm{\Sigma}\bm{A}^{-1}\bm{\Sigma}^{\frac{1}{2}}=\left\{ \bm{\Sigma}^{-\frac{1}{2}}\bm{A}\bm{\Sigma}^{-1}\bm{A}^{\top}\bm{\Sigma}^{-\frac{1}{2}}\right\} ^{-1}.
\]
Observing that for any unit vector $\bm{x}$, one has 
\begin{align*}
\bm{x}^{\top}\left\{ \bm{\Sigma}^{-\frac{1}{2}}\bm{A}\bm{\Sigma}^{-1}\bm{A}^{\top}\bm{\Sigma}^{-\frac{1}{2}}\right\} \bm{x} & =\big\|\bm{\Sigma}^{-\frac{1}{2}}\bm{A}^{\top}\bm{\Sigma}^{-\frac{1}{2}}\bm{x}\big\|_{2}^{2}\\
 & \geq\left(\bm{x}^{\top}\bm{\Sigma}^{-\frac{1}{2}}\bm{A}^{\top}\bm{\Sigma}^{-\frac{1}{2}}\bm{x}\right)^{2}\\
 & =\left(\bm{x}^{\top}\bm{\Sigma}^{-\frac{1}{2}}\bm{A}\bm{\Sigma}^{-\frac{1}{2}}\bm{x}\right)^{2}\\
 & \geq(1-\gamma)^{2}.
\end{align*}
Consequently, 
\begin{align*}
\bm{\phi}(s)^{\top}(\bm{A}^{\top})^{-1}\bm{\Sigma}\bm{A}^{-1}\bm{\phi}(s) & =\bm{\phi}(s)^{\top}\bm{\Sigma}^{-\frac{1}{2}}\left\{ \bm{\Sigma}^{\frac{1}{2}}(\bm{A}^{\top})^{-1}\bm{\Sigma}\bm{A}^{-1}\bm{\Sigma}^{\frac{1}{2}}\right\} \bm{\Sigma}^{-\frac{1}{2}}\bm{\phi}(s)\\
 & \leq\frac{1}{(1-\gamma)^{2}}\bm{\phi}(s)^{\top}\bm{\Sigma}^{-1}\bm{\phi}(s)\\
 & \leq\frac{\|\bm{\Sigma}^{-1}\|}{(1-\gamma)^{2}},
\end{align*}
thus indicating that 
\end{comment}

\medskip{}
\textit{Step 3: controlling $\|\bm{Z}_{t}\|$.} Our starting point
is the following triangle inequality 
\begin{align*}
\|\bm{Z}_{t}\| & =\big\|\bm{\Sigma}^{\frac{1}{2}}\bm{A}^{-1}\big(\bm{A}-\bm{A}_{t}\big)\bm{\Sigma}^{-\frac{1}{2}}\big\|\\ 
&\leq\big\|\bm{\Sigma}^{\frac{1}{2}}\bm{A}^{-1}\bm{A}_{t}\bm{\Sigma}^{-\frac{1}{2}}\big\|+\big\|\bm{\Sigma}^{\frac{1}{2}}\bm{A}^{-1}\bm{A}\bm{\Sigma}^{-\frac{1}{2}}\big\|\\
 & \leq\big\|\bm{\Sigma}^{\frac{1}{2}}\bm{A}^{-1}\bm{\Sigma}^{\frac{1}{2}}\big\|\cdot\big\|\bm{\Sigma}^{-\frac{1}{2}}\bm{A}_{t}\bm{\Sigma}^{-\frac{1}{2}}\big\|+1\\
 & \leq\frac{1}{1-\gamma}\big\|\bm{\Sigma}^{-\frac{1}{2}}\bm{A}_{t}\bm{\Sigma}^{-\frac{1}{2}}\big\|+1,
\end{align*}
where the last inequality follows from Lemma \ref{lemma:Sigma-inv-A-lower-bound}.
In addition, we see that 
\begin{align}
\big\|\bm{\Sigma}^{-\frac{1}{2}}\bm{A}_{t}\bm{\Sigma}^{-\frac{1}{2}}\big\|
&\leq\max_{s}\big\|\bm{\Sigma}^{-\frac{1}{2}}\bm{\phi}(s)\bm{\phi}(s)^{\top}\bm{\Sigma}^{-\frac{1}{2}}\big\|\nonumber \\ 
&+\gamma\max_{s,s'}\big\|\bm{\Sigma}^{-\frac{1}{2}}\bm{\phi}(s')\bm{\phi}(s)^{\top}\bm{\Sigma}^{-\frac{1}{2}}\big\|\nonumber \\ 
&\leq2\max_{s}\big\|\bm{\Sigma}^{-\frac{1}{2}}\bm{\phi}(s)\big\|_{2}^{2}.\label{eq:useful-for-TD}
\end{align}
This combined with the preceding bounds yields 
\begin{align}
\|\bm{Z}_{t}\| & \leq\frac{2\max_{s}\big\|\bm{\Sigma}^{-\frac{1}{2}}\bm{\phi}(s)\big\|_{2}^{2}}{1-\gamma}+1 \nonumber \\ 
&\leq\frac{4\max_{s}\big\|\bm{\Sigma}^{-\frac{1}{2}}\bm{\phi}(s)\big\|_{2}^{2}}{1-\gamma}\nonumber \\ 
&=\frac{4\max_{s}\bm{\phi}(s)^{\top}\bm{\Sigma}^{-1}\bm{\phi}(s)}{1-\gamma}.\label{eq:Zt-bound-1}
\end{align}
Here, the inequality follows since 
\begin{align}
\max_{s}\big\|\bm{\Sigma}^{-\frac{1}{2}}\bm{\phi}(s)\big\|_{2}^{2} & \geq\mathop{\mathbb{E}}\limits _{s\sim\mu}\left[\bm{\phi}(s)^{\top}\bm{\Sigma}^{-1}\bm{\phi}(s)\right]\nonumber \\ 
&=\mathop{\mathbb{E}}\limits _{s\sim\mu}\left[\mathrm{tr}\big(\bm{\Sigma}^{-1}\bm{\phi}(s)\bm{\phi}(s)^{\top}\big)\right]\nonumber \\ 
&=\mathrm{tr}(\Ind_{d})=d\geq1.\label{eq:useful-for-TD-2}
\end{align}

\noindent \medskip{}
\textit{Step 4: invoking the matrix Bernstein inequality.} With the
above bounds in mind, we are ready to apply the matrix Bernstein inequality
\citep{Tropp:2015:IMC:2802188.2802189} to obtain that: with probability
at least $1-\delta$ one has 
\begin{align}
&\big\|\bm{\Sigma}^{\frac{1}{2}}\bm{A}^{-1}\big(\bm{A}-\widehat{\bm{A}}\big)\bm{\Sigma}^{-\frac{1}{2}}\big\|\nonumber \\ 
 & \lesssim\sqrt{\frac{1}{T^{2}}\sum_{t=0}^{T-1}\max\left\{ \big\|\mathbb{E}\left[\bm{Z}_{t}\bm{Z}_{t}^{\top}\right]\big\|,\,\big\|\mathbb{E}\left[\bm{Z}_{t}^{\top}\bm{Z}_{t}\right]\big\|\right\} \log\Big(\frac{d}{\delta}\Big)}\nonumber \\ 
 &+\frac{\max_{t}\|\bm{Z}_{t}\|\log(\frac{d}{\delta})}{T}.\nonumber \\
 & \overset{(\mathrm{i})}{\lesssim}\sqrt{\frac{\max_{s}\bm{\phi}(s)^{\top}\bm{\Sigma}^{-1}\bm{\phi}(s)}{T(1-\gamma)^{2}}\log\Big(\frac{d}{\delta}\Big)}\nonumber \\ 
 &+\frac{\max_{s}\bm{\phi}(s)^{\top}\bm{\Sigma}^{-1}\bm{\phi}(s)\log(\frac{d}{\delta})}{T(1-\gamma)}\nonumber \\
 & \overset{(\mathrm{ii})}{\asymp}\sqrt{\frac{\max_{s}\bm{\phi}(s)^{\top}\bm{\Sigma}^{-1}\bm{\phi}(s)}{T(1-\gamma)^{2}}\log\Big(\frac{d}{\delta}\Big)}.\label{eq:Sigma-A-perturb-bound-iid}
\end{align}
Here, (i) results from the bounds \eqref{eq:E-Zt-ZtT-1}, \eqref{eq:E-Zt-ZtT-2}
and \eqref{eq:Zt-bound-1}, while (ii) holds as long as $T\gtrsim\max_{s}\bm{\phi}(s)^{\top}\bm{\Sigma}^{-1}\bm{\phi}(s)\log\big(\frac{d}{\delta}\big)$.

\noindent In addition, if $T\geq \frac{c_{2}\max_{s}\bm{\phi}(s)^{\top}\bm{\Sigma}^{-1}\bm{\phi}(s)\log\big(\frac{d}{\delta}\big)}{(1-\gamma)^2}$
for some constant $c_{2}$ large enough, then one has $\big\|\bm{\Sigma}^{\frac{1}{2}}\bm{A}^{-1}\big(\bm{A}-\widehat{\bm{A}}\big)\bm{\Sigma}^{-\frac{1}{2}}\big\|<1$.
Suppose that $\widehat{\bm{A}}$ is not invertible. Given that $\bm{A}$
and $\bm{\Sigma}$ are both invertible, this means that one can find
a unit vectors $\bm{u}$ obeying $\bm{A}^{-1}\widehat{\bm{A}}\bm{\Sigma}^{-\frac{1}{2}}\bm{u}=\bm{0}$,
which in turn implies 
\begin{align*}
&\bm{u}^{\top}\bm{\Sigma}^{\frac{1}{2}}\bm{A}^{-1}\big(\bm{A}-\widehat{\bm{A}}\big)\bm{\Sigma}^{-\frac{1}{2}}\bm{u}\\ 
 & =\bm{u}^{\top}\bm{\Sigma}^{\frac{1}{2}}\bm{A}^{-1}\bm{A}\bm{\Sigma}^{-\frac{1}{2}}\bm{u}-\bm{u}^{\top}\bm{\Sigma}^{\frac{1}{2}}\bm{A}^{-1}\widehat{\bm{A}}\bm{\Sigma}^{-\frac{1}{2}}\bm{u}\\
 & =1-0=1
\end{align*}
and hence contradicts the condition $\big\|\bm{\Sigma}^{\frac{1}{2}}\bm{A}^{-1}\big(\bm{A}-\widehat{\bm{A}}\big)\bm{\Sigma}^{-\frac{1}{2}}\big\|<1$.
As a result, we conclude that $\widehat{\bm{A}}$ is invertible as
long as $\big\|\bm{\Sigma}^{\frac{1}{2}}\bm{A}^{-1}\big(\bm{A}-\widehat{\bm{A}}\big)\bm{\Sigma}^{-\frac{1}{2}}\big\|<1$.

\paragraph{Proof of Lemma \ref{lemma:Ainv-bhat-concentration}: controlling $\big\|\bm{A}^{-1}\big(\widehat{\bm{b}}-\bm{b}\big)\big\|_{\bm{\Sigma}}$}

First of all, it is seen that 
\[
\big\|\bm{A}^{-1}\big(\widehat{\bm{b}}-\bm{b}\big)\big\|_{\bm{\Sigma}}=\Big\|\frac{1}{T}\sum_{t=0}^{T-1}\bm{\Sigma}^{\frac{1}{2}}\bm{A}^{-1}\big(\bm{b}_{t}-\bm{b}\big)\Big\|_{2}=\Big\|\frac{1}{T}\sum_{t=0}^{T-1}\bm{z}_{t}\Big\|_{2},
\]
where we define the vector $\bm{z}_{t}:=\bm{\Sigma}^{\frac{1}{2}}\bm{A}^{-1}\big(\bm{b}_{t}-\bm{b}\big)$.
Therefore, we need to look at the properties of $\bm{z}_{t}$. Towards
this end, we observe that 
\begin{align*}
\mathbb{E}\big[\bm{z}_{t}^{\top}\bm{z}_{t}\big] & =\mathbb{E}\left[\big(\bm{b}_{t}-\bm{b}\big)^{\top}\big(\bm{A}^{\top}\big)^{-1}\bm{\Sigma}\bm{A}^{-1}\big(\bm{b}_{t}-\bm{b}\big)\right]\\ 
&\preceq\mathbb{E}\left[\bm{b}_{t}^{\top}\big(\bm{A}^{\top}\big)^{-1}\bm{\Sigma}\bm{A}^{-1}\bm{b}_{t}\right]\\
 & \overset{(\mathrm{i})}{\leq}\Big\{\max_{s\in\mathcal{S}}\left|r(s)\right|^{2}\Big\}\mathbb{E}\left[\bm{\phi}(s_{t})^{\top}\big(\bm{A}^{\top}\big)^{-1}\bm{\Sigma}\bm{A}^{-1}\bm{\phi}(s_{t})\right]\\
 & \overset{(\mathrm{ii})}{\leq}\mathbb{E}\left[\bm{\phi}(s_{t})^{\top}\big(\bm{A}^{\top}\big)^{-1}\bm{\Sigma}\bm{A}^{-1}\bm{\phi}(s_{t})\right]\\
 & =\mathbb{E}\left[\bm{\phi}(s_{t})^{\top}\bm{\Sigma}^{-\frac{1}{2}}\bm{\Sigma}^{\frac{1}{2}}\big(\bm{A}^{\top}\big)^{-1}\bm{\Sigma}\bm{A}^{-1}\bm{\Sigma}^{\frac{1}{2}}\bm{\Sigma}^{-\frac{1}{2}}\bm{\phi}(s_{t})\right]\\
 & \leq\Big\{\max_{s\in\mathcal{S}}\big\|\bm{\Sigma}^{-\frac{1}{2}}\bm{\phi}(s)\big\|_{2}^{2}\Big\}\cdot\big\|\bm{\Sigma}^{\frac{1}{2}}\big(\bm{A}^{\top}\big)^{-1}\bm{\Sigma}\bm{A}^{-1}\bm{\Sigma}^{\frac{1}{2}}\big\|\\
 & \overset{(\mathrm{iii})}{\leq}\frac{1}{(1-\gamma)^{2}}\max_{s\in\mathcal{S}}\big\|\bm{\Sigma}^{-\frac{1}{2}}\bm{\phi}(s)\big\|_{2}^{2},
\end{align*}
where (i) holds since $\bm{b}_{t}=\bm{\phi}(s_{t})r(s_{t})$, (ii)
follows from the assumption $\max_{s}|r(s)|\leq1$, and (iii) arises
from Lemma \ref{lemma:Sigma-inv-A-lower-bound}. Additionally, 
\begin{align*}
\max_{t}\big\|\bm{z}_{t}\big\|_{2} & \leq\max_{t}\big\|\bm{\Sigma}^{\frac{1}{2}}\bm{A}^{-1}\bm{b}_{t}\big\|_{2}+\big\|\bm{\Sigma}^{\frac{1}{2}}\bm{A}^{-1}\bm{b}\big\|_{2}\\
 & \overset{(\mathrm{iv})}{\leq}2\max_{s}\big\|\bm{\Sigma}^{\frac{1}{2}}\bm{A}^{-1}\bm{\phi}(s)r(s)\big\|_{2}\\
 & \overset{(\mathrm{v})}{\leq}2\max_{s\in\mathcal{S}}\big\|\bm{\Sigma}^{\frac{1}{2}}\bm{A}^{-1}\bm{\phi}(s)\big\|_{2}\\
 & \leq2\big\|\bm{\Sigma}^{\frac{1}{2}}\bm{A}^{-1}\bm{\Sigma}^{\frac{1}{2}}\big\|\cdot\max_{s\in\mathcal{S}}\big\|\bm{\Sigma}^{-\frac{1}{2}}\bm{\phi}(s)\big\|_{2}\\
 & \leq\frac{2}{1-\gamma}\max_{s\in\mathcal{S}}\big\|\bm{\Sigma}^{-\frac{1}{2}}\bm{\phi}(s)\big\|_{2},
\end{align*}
where (iv) holds since $\bm{b}_{t}=\bm{\phi}(s_{t})r(s_{t})$ and
$\bm{b}=\mathbb{E}_{s\sim\mu}\big[\bm{\phi}(s)r(s)\big]$, (v) comes
from the assumption $\max_{s}|r(s)|\leq1$, and the last line is due
to Lemma \ref{lemma:Sigma-inv-A-lower-bound}. Consequently, the matrix
Bernstein inequality \cite{Tropp:2015:IMC:2802188.2802189} yields
\begin{align}
&\big\|\bm{A}^{-1}\big(\widehat{\bm{b}}-\bm{b}\big)\big\|_{\bm{\Sigma}} \nonumber \\ 
& =\Big\|\frac{1}{T}\sum_{t=1}^{T}\bm{z}_{t}\Big\|_{2}\nonumber \\ 
&\lesssim\sqrt{\frac{1}{T^{2}}\sum_{t=0}^{T-1}\mathbb{E}\big[\bm{z}_{t}^{\top}\bm{z}_{t}\big]\log\Big(\frac{d}{\delta}\Big)}+\frac{1}{T}\max_{t}\big\|\bm{z}_{t}\big\|_{2}\log\Big(\frac{d}{\delta}\Big)\nonumber \\
 & \lesssim\frac{\max_{s\in\mathcal{S}}\big\|\bm{\Sigma}^{-\frac{1}{2}}\bm{\phi}(s)\big\|_{2}}{1-\gamma}\sqrt{\frac{1}{T}\log\Big(\frac{d}{\delta}\Big)}\nonumber \\ 
 &+\frac{\max_{s\in\mathcal{S}}\big\|\bm{\Sigma}^{-\frac{1}{2}}\bm{\phi}(s)\big\|_{2}}{1-\gamma}\cdot\frac{1}{T}\log\Big(\frac{d}{\delta}\Big)\nonumber \\
 & \asymp\frac{\max_{s\in\mathcal{S}}\big\|\bm{\Sigma}^{-\frac{1}{2}}\bm{\phi}(s)\big\|_{2}}{1-\gamma}\sqrt{\frac{1}{T}\log\Big(\frac{d}{\delta}\Big)}\label{eq:Ainv-b-perturb-iid-bound}
\end{align}
with probability at least $1-\delta$, as long as $T\gtrsim\log\big(\frac{d}{\delta}\big)$.

\subsection{Proof of Lemma \ref{lemma:TD-main-term} \label{subsec:Proof-of-Lemma-TD-main-term}}

Recall from the proof of Lemma~\ref{lemma:exp-decay-sum} that $\Exs_{i}[\cdot]$
represents the expectation conditioned on the probability space generated
by the samples $\{(s_{j},s_{j}')\}_{j\leq i}$. It is easy to check
that $\{\bm{\Sigma}^{\frac{1}{2}}\bm{A}^{-1}(\bm{A}_{i}-\bm{A})\bm{\theta}'_{i}\}$
forms a martingale difference sequence, and  we seek to bound
$\Big\|\frac{1}{u-l+1}\sum_{i=l}^{u}\bm{\Sigma}^{\frac{1}{2}}\bm{A}^{-1}(\bm{A}_{i}-\bm{A})\bm{\theta}'_{i}\Big\|_{2}$
via matrix Freedman's inequality. The key is to control the following
quantities (here, we abuse  notation whenever it is clear from
 context): 
\begin{align}
&W\coloneqq\sum_{i=l}^{u}\mathbb{E}_{i-1}\Big[\big\|\bm{\Sigma}^{1/2}\bm{A}^{-1}(\bm{A}_{i}-\bm{A})\bm{\theta}'_{i}\big\|_{2}^{2}\Big]\quad\text{and}\nonumber \\ 
&B\coloneqq\max_{i:l\leq i\leq u} \big\|\bm{\Sigma}^{1/2}\bm{A}^{-1}(\bm{A}_{i}-\bm{A})\bm{\theta}'_{i}\big\|_{2}.\label{eq:defn-Wi-Bi-new}
\end{align}

\paragraph{Control of $B$}

To begin with, observe that 
\begin{align*}
B & =\max_{i:l\leq i\leq u}\big\|\bm{\Sigma}^{1/2}\bm{A}^{-1}(\bm{A}_{i}-\bm{A})\bm{\Sigma}^{-1/2}\big\|\cdot\big\|\bm{\theta}'_{i}\big\|_{\bm{\Sigma}}\\
 & \leq\frac{4\max_{s}\bm{\phi}(s)^{\top}\bm{\Sigma}^{-1}\bm{\phi}(s)}{1-\gamma}\max_{i:l\leq i\leq u}\left\{ \|\bm{\theta}^{\star}\|_{\bm{\Sigma}}+\|\bm{\Delta}_{i}\|_{\bm{\Sigma}}\right\} \ind\{\mathcal{H}_{i}\}\\
 & \leq\frac{4\max_{s}\bm{\phi}(s)^{\top}\bm{\Sigma}^{-1}\bm{\phi}(s)}{1-\gamma}\big(\|\bm{\theta}^{\star}\|_{\bm{\Sigma}}+R\big)\eqqcolon B_{\max},
\end{align*}
where the second to last inequality comes from (\ref{eq:Zt-bound-1})
and the triangle inequality, and the last line is due to the definition
of $\mathcal{H}_{i}$. 

\paragraph{Control of $W$}

Moreover, one can derive 
\begin{align*}
W & \coloneqq\sum_{i=l}^{u}\mathbb{E}_{i-1}\Big[\|\bm{\Sigma}^{1/2}\bm{A}^{-1}(\bm{A}_{i}-\bm{A})\bm{\Sigma}^{-1/2}\bm{\Sigma}^{1/2}\bm{\theta}_{i}'\big\|_{2}^{2}\Big]\\
 & =\sum_{i=l}^{u}\bm{\theta}_{i}'^{\top}\bm{\Sigma}^{1/2}\mathbb{E}_{i-1}\Big[\bm{\Sigma}^{-1/2}(\bm{A}_{i}-\bm{A})^{\top}(\bm{A}^{\top})^{-1}\bm{\Sigma}\bm{A}^{-1}\\ 
 &\qquad \qquad \qquad \qquad \quad (\bm{A}_{i}-\bm{A})\bm{\Sigma}^{-1/2}\Big]\bm{\Sigma}^{1/2}\bm{\theta}_{i}'\\
 & \leq\sum_{i=l}^{u}\frac{4\max_{s}\bm{\phi}(s)^{\top}\bm{\Sigma}^{-1}\bm{\phi}(s)}{(1-\gamma)^{2}}\big\|\bm{\Sigma}^{1/2}\bm{\theta}_{i}'\big\|_{2}^{2}\\
 & \leq\frac{4\max_{s}\bm{\phi}(s)^{\top}\bm{\Sigma}^{-1}\bm{\phi}(s)}{(1-\gamma)^{2}}\sum_{i=l}^{u}\big(\|\bm{\theta}^{\star}\|_{\bm{\Sigma}}+\|\bm{\Delta}_{i}\|_{\bm{\Sigma}}\big)^{2}\ind\{\mathcal{H}_{i}\}\\
 & \leq\frac{4(u-l+1)\max_{s}\bm{\phi}(s)^{\top}\bm{\Sigma}^{-1}\bm{\phi}(s)}{(1-\gamma)^{2}}\big(\|\bm{\theta}^{\star}\|_{\bm{\Sigma}}+R\big)^{2}\\ 
 &\eqqcolon W_{\max},
\end{align*}
where the first inequality arises from (\ref{eq:E-Zt-ZtT-2}), and
the last inequality makes use of the definition of $\mathcal{H}_{i}$.

With the above bounds in place, we can apply Freedman's inequality
\cite[Corollary 1.3]{tropp2011freedman} for matrix martingales
to demonstrate that 
\begin{align*}
 & \Big\|\frac{1}{u-l+1}\sum_{i=l}^{u}\bm{\Sigma}^{\frac{1}{2}}\bm{A}^{-1}(\bm{A}_{i}-\bm{A})\bm{\theta}'_{i}\Big\|_{2}\\ 
 &\leq\frac{2}{u-l+1}\sqrt{W_{\max}\log\frac{2d}{\delta}}+\frac{4}{3u-l+1}B_{\max}\log\frac{2d}{\delta}\\
 & \leq\frac{8\big(\|\bm{\theta}^{\star}\|_{\bm{\Sigma}}+R\big)}{1-\gamma}\sqrt{\frac{\max_{s}\bm{\phi}(s)^{\top}\bm{\Sigma}^{-1}\bm{\phi}(s)\log\frac{2d}{\delta}}{u-l+1}}\\ 
 &+\frac{16{\max_{s}\bm{\phi}(s)^{\top}\bm{\Sigma}^{-1}\bm{\phi}(s)}\log\frac{2d}{\delta}}{3(1-\gamma)(u-l+1)}\big(\|\bm{\theta}^{\star}\|_{\bm{\Sigma}}+R\big)\\
 & \leq\frac{16\big(\|\bm{\theta}^{\star}\|_{\bm{\Sigma}}+R\big)}{1-\gamma}\sqrt{\frac{\max_{s}\bm{\phi}(s)^{\top}\bm{\Sigma}^{-1}\bm{\phi}(s)\log\frac{2d}{\delta}}{u-l+1}}
\end{align*}
with probability at least $1-\delta$, as long as $u-l+1\geq\frac{4{\max_{s}\bm{\phi}(s)^{\top}\bm{\Sigma}^{-1}\bm{\phi}(s)}\log\frac{2d}{\delta}}{9}$.

\section{Comparisons with previous works}

\subsection{Comparisons with \cite{srikant2019finite}}
\label{app:compare-Srikant}

\cite{srikant2019finite} bounded the expectation of TD estimation error $\Exs\|\bm{\theta}_T-\bm{\theta}^\star\|_2^2$ with Markov samples by an iterative relation. 
For fair comparisons, we apply their ideas to bounding the error in $\bm{\Sigma}$-norm with independent samples. 
\paragraph{Iterative relation on $\mathbb{E} \|\bm{\Delta}_t\|_{\Sigma}^2$} Recall from the TD update rule \eqref{eq:TD-update-all} that
\begin{align*}
\bm{\Delta}_{t+1} &= \bm{\Delta}_t - \eta_t(\bm{A}_t \bm{\theta}_t - \bm{b}_t) \\ 
&=  (\bm{I} - \eta_t \bm{A}_t) \bm{\Delta}_t -\eta_t (\bm{A}_t \bm{\theta}^\star - \bm{b}_t).
\end{align*}
Therefore, the $\bm{\Sigma}$-norm of $\bm{\Delta}_{t+1}$ can be expressed as 
\begin{align*}
\left\|\bm{\Delta}_{t+1}\right\|_{\bm{\Sigma}}^2 &= \|\bm{\Delta}_t\|_{\bm{\Sigma}}^2 -2\eta_t \langle \bm{\Delta}_t, \bm{A}_t \bm{\Delta}_t \rangle_{\bm{\Sigma}} + \eta_t^2 \|\bm{A}_t \bm{\Delta}_t \|_{\bm{\Sigma}}^2\\ 
 &- 2\eta_t \langle \bm{\Delta}_t,\bm{A}_t \bm{\theta}^\star - \bm{b}_t\rangle_{\bm{\Sigma}} + 2\eta_t^2 \langle \bm{A}_t \bm{\Delta}_t, \bm{A}_t \bm{\theta}^\star - \bm{b}_t\rangle_{\bm{\Sigma}}\\ 
 & + \eta_t^2 \|\bm{A}_t \bm{\theta}^\star - \bm{b}_t\|_{\bm{\Sigma}}^2.
\end{align*}
Notice that by definition,
\begin{align*}
\mathbb{E}_t \langle \bm{\Delta}_t,\bm{A}_t \bm{\theta}^\star - \bm{b}_t\rangle_{\bm{\Sigma}} = \langle \bm{\Delta}_t, \bm{A} \bm{\theta}^\star - \bm{b} \rangle = 0,
\end{align*}
and that a basic property of inner product yields
\begin{align*}
2\langle \bm{A}_t \bm{\Delta}_t, \bm{A}_t \bm{\theta}^\star - \bm{b}_t\rangle_{\bm{\Sigma}}  \leq \|\bm{A}_t \bm{\Delta}_t\|_{\bm{\Sigma}}^2 +\|\bm{A}_t \bm{\theta}^\star - \bm{b}_t\|_{\bm{\Sigma}}^2.
\end{align*}
Therefore, we can apply the law of total expectations to obtain the following iterative relation:
\begin{align}\label{eq:Sigma-norm-iter}
\mathbb{E}\|\bm{\Delta}_{t+1}\|_{\bm{\Sigma}}^2  &= \mathbb{E}\|\bm{\Delta}_t\|_{\bm{\Sigma}}^2 - \underset{I_1}{\underbrace{2\eta_t \mathbb{E}[\bm{\Delta}_t^\top (\bm{A}^{\top} \bm{\Sigma} + \bm{\Sigma A})\bm{\Delta}_t}]} \nonumber \\ 
&+  \underset{I_2}{\underbrace{2\eta_t^2\mathbb{E}\|\bm{A}_t \bm{\Delta}_t\|_{\bm{\Sigma}}^2}} + \underset{I_3}{\underbrace{2\eta_t^2 \mathbb{E}\|\bm{A}_t \bm{\theta}^\star - \bm{b}_t\|_{\bm{\Sigma}}^2}}.
\end{align}
We now turn to bounding $I_1$, $I_2$ and $I_3$ in order.
\paragraph{Bounding $I_1$} In order to lower bound $I_1$ as a function of $\|\bm{\Delta}_t\|_{\bm{\Sigma}}^2$, we firstly express it as 
\begin{align*}
&\bm{\Delta}_t^\top (\bm{A}^{\top} \bm{\Sigma} + \bm{\Sigma A})\bm{\Delta}_t\\ 
&= \bm{\Delta}_t^\top \bm{\Sigma}^{1/2}\bm{\Sigma}^{-1/2}(\bm{A}^{\top} \bm{\Sigma} + \bm{\Sigma A})\bm{\Sigma}^{-1}\bm{\Sigma}^{1/2}\bm{\Delta}_t \\ 
&\geq \|\bm{\Sigma}^{1/2}\bm{\Delta}_t\|_2^2 \lambda_{\min}\left(\bm{\Sigma}^{-1/2}\bm{A}^\top \bm{\Sigma}^{1/2} + \bm{\Sigma}^{1/2}\bm{A}\bm{\Sigma}^{-1/2}\right) \\ 
&= \|\bm{\Delta}_t\|_{\bm{\Sigma}}^2 \lambda_{\min}\left(\bm{\Sigma}^{-1/2}\bm{A}^\top \bm{\Sigma}^{1/2} + \bm{\Sigma}^{1/2}\bm{A}\bm{\Sigma}^{-1/2}\right) .
\end{align*}
Recall from \eqref{eq:Sigma-A-Sigma-UB-12} that
\begin{align*}
\|\bm{\Sigma}^{\frac{1}{2}}\bm{A}^{-1}\bm{\Sigma}^{\frac{1}{2}}\| \leq (1-\gamma)^{-1},
\end{align*}
so the minimal eigenvalue of $\bm{\Sigma}^{-1/2}\bm{A}^\top \bm{\Sigma}^{1/2} + \bm{\Sigma}^{1/2}\bm{A}\bm{\Sigma}^{-1/2}$ is lower bounded by
\begin{align*}
&\lambda_{\min}\left(\bm{\Sigma}^{-1/2}\bm{A}^\top \bm{\Sigma}^{1/2} + \bm{\Sigma}^{1/2}\bm{A}\bm{\Sigma}^{-1/2}\right) \\ 
& \geq \lambda_{\min}(\bm{\Sigma}) \cdot \left[\gamma_{\min}\left(\bm{\Sigma}^{-\frac{1}{2}}\bm{A}^\top \bm{\Sigma}^{-\frac{1}{2}}\right) + \gamma_{\min}\left(\bm{\Sigma}^{-\frac{1}{2}}\bm{A} \bm{\Sigma}^{-\frac{1}{2}}\right)\right] \\ 
&\geq \frac{2\lambda_{\min}(\bm{\Sigma})}{\|\bm{\Sigma}^{\frac{1}{2}}\bm{A}^{-1}\bm{\Sigma}^{\frac{1}{2}}\|} \geq 2\lambda_{\min}(\bm{\Sigma})(1-\gamma).
\end{align*}
This directly implies that $I_1$ is lower bounded by
\begin{align}\label{Srikant-I1-bound}
I_1 \geq 2\eta_t (1-\gamma) \lambda_{\min}(\bm{\Sigma}) \mathbb{E}\|\bm{\Delta}_t\|_{\bm{\Sigma}}^2.
\end{align}
\paragraph{Bounding $I_2$} We aim to upper bound $I_2$ as a function of $\eta_t^2$ and $\|\bm{\Delta}_t\|_{\bm{\Sigma}}^2$, so that when $\eta_t$ is sufficiently small, $I_2$ is negligible compared to $I_1$. Specifically, for any $\bm{A}_t$ generated by \eqref{eq:defn-At} and any $\bm{\Delta}_t \in \mathbb{R}^d$, we observe
\begin{align*}
\|\bm{A}_t \bm{\Delta}_t\|_{\bm{\Sigma}}^2&= \bm{\Delta}_t^\top \bm{A}_t \bm{\Sigma} \bm{A}_t \bm{\Delta}_t  \\ 
&\leq \|\bm{\Delta}_t\|_2^2 \|\bm{A}\|^2 \|\bm{\Sigma}\| \leq 4 \|\bm{\Sigma}\| \|\bm{\Delta}_t\|_2^2 \\ 
&\leq 4 \|\bm{\Sigma}\| \|\bm{\Sigma}^{-1}\| \|\bm{\Sigma}^{\frac{1}{2}}\bm{\Delta}_t\|_2^2 = 4\kappa \|\bm{\Delta}_t\|_{\bm{\Sigma}}^2,
\end{align*}
where we recall $\kappa$ as the condition number of $\bm{\Sigma}$. Therefore, as long as
\begin{align*}
\eta_t \leq \frac{(1-\gamma) \lambda_{\min}(\bm{\Sigma})}{4\kappa},
\end{align*}
it can be guaranteed that $I_2 \leq \frac{1}{2}I_1$. 
\paragraph{Bounding $I_3$} In order to compare with our result (Theorem \ref{thm:ind-td} and Corollary \ref{cor:sample-complexity-TD-iid}), we aim to bound $I_3$ as a function of $\|\bm{\theta}^\star\|_{\bm{\Sigma}}$. Towards this end, we firstly notice that
\begin{align*}
\bm{A}_t \bm{\theta}^\star - \bm{b}_t = \bm{\phi}(s_t) \bm{\phi}(s_t)^\top \bm{\theta}^\star - \gamma \bm{\phi}(s_t) \bm{\phi}(s_t')^\top \bm{\theta}^\star -r(s_t) \bm{\phi}(s_t).
\end{align*}
Therefore, we can upper bound $\mathbb{E}\|\bm{A}_t \bm{\theta}^\star - \bm{b}_t\|_{\bm{\Sigma}}^2$ by
\begin{align*}
\mathbb{E}\|\bm{A}_t \bm{\theta}^\star - \bm{b}_t\|_{\bm{\Sigma}}^2 &\leq 3 \underset{s\sim \mu}{\mathbb{E}} \|\bm{\phi}(s) \bm{\phi}(s)^\top \bm{\theta}^\star\|_{\bm{\Sigma}}^2  \\ 
&+ 3 \underset{s \sim \mu, s' \sim \mathcal{P}(\cdot \mid s)}{\mathbb{E}}\|\bm{\phi}(s) \bm{\phi}(s')^\top \bm{\theta}^\star \|_{\bm{\Sigma}}^2 \\ 
&+ 3\underset{s\sim \mu}{\mathbb{E}} \|r(s) \bm{\phi}(s)\|_{\bm{\Sigma}}^2,
\end{align*}
where the three terms on the right-hand-side can be bounded respectively by
\begin{align*}
&\underset{s\sim \mu}{\mathbb{E}} \|\bm{\phi}(s) \bm{\phi}(s)^\top \bm{\theta}^\star\|_{\bm{\Sigma}}^2  \\ 
&= \underset{s\sim \mu}{\mathbb{E}} \left[\bm{\theta}^{\star \top} \bm{\phi}(s) \left(\bm{\phi}(s)^\top \bm{\Sigma} \bm{\phi}(s) \right)\bm{\phi}(s)^\top \bm{\theta}^\star \right] \\ 
&\leq \underset{s\sim \mu}{\mathbb{E}} \left[\bm{\theta}^{\star \top} \bm{\phi}(s) \|\bm{\Sigma}\| \bm{\phi}(s)^\top \bm{\theta}^\star \right]\\ 
&= \|\bm{\Sigma}\| \bm{\theta}^{\star \top} \underset{s\sim \mu}{\mathbb{E}} [\bm{\phi}(s) \bm{\phi}(s)^\top]\bm{\theta}^\star\\ 
&= \|\bm{\Sigma}\| \bm{\theta}^{\star \top} \bm{\Sigma} \bm{\theta}^\star = \|\bm{\Sigma}\| \|\bm{\theta}^\star\|_{\bm{\Sigma}}^2;
\end{align*} 
\begin{align*}
&\underset{s \sim \mu, s' \sim \mathcal{P}(\cdot \mid s)}{\mathbb{E}}\|\bm{\phi}(s) \bm{\phi}(s')^\top \bm{\theta}^\star \|_{\bm{\Sigma}}^2\\ 
 & = \underset{s\sim \mu,s' \sim \mathcal{P}(\cdot \mid s)}{\mathbb{E}} \left[\bm{\theta}^{\star \top} \bm{\phi}(s') \left(\bm{\phi}(s)^\top \bm{\Sigma} \bm{\phi}(s) \right)\bm{\phi}(s')^\top \bm{\theta}^\star \right]\\
&\leq \underset{s\sim \mu,s' \sim \mathcal{P}(\cdot \mid s)}{\mathbb{E}} \left[\bm{\theta}^{\star \top} \bm{\phi}(s') \|\bm{\Sigma}\| \bm{\phi}(s')^\top \bm{\theta}^\star \right]\\ 
&= \|\bm{\Sigma}\| \bm{\theta}^{\star \top} \underset{s'\sim \mu}{\mathbb{E}} [\bm{\phi}(s') \bm{\phi}(s')^\top]\bm{\theta}^\star\\ 
&= \|\bm{\Sigma}\| \bm{\theta}^{\star \top} \bm{\Sigma} \bm{\theta}^\star = \|\bm{\Sigma}\| \|\bm{\theta}^\star\|_{\bm{\Sigma}}^2, 
\end{align*}
\begin{align*}
\text{and}~~\underset{s\sim \mu}{\mathbb{E}} \|r(s) \bm{\phi}(s)\|_{\bm{\Sigma}}^2 &\leq \max_{s \in \mathcal{S}} r^2(s) \|\bm{\phi}(s)\|_2^2 \|\bm{\Sigma}\| \leq \|\bm{\Sigma}\|.
\end{align*}
Consequently, $I_3$ can be upper bounded by
\begin{align}\label{eq:Srikant-I3-bound}
I_3 \leq 6\eta_t^2 \|\bm{\Sigma}\|\left(2\|\bm{\theta}^{\star}\|_{\bm{\Sigma}}^2 + 1\right).
\end{align}
\paragraph{Bounding $\mathbb{E}\|\bm{\Delta}_T\|_{\bm{\Sigma}}^2$} By combining \eqref{eq:Sigma-norm-iter}, \eqref{Srikant-I1-bound} and \eqref{eq:Srikant-I3-bound} and recalling that $I_2 \leq \frac{1}{2}I_1$ when $\eta_t$ is sufficiently small, we obtain 
\begin{align}\label{eq:Srikant-iter}
\mathbb{E}\|\bm{\Delta}_{t+1}\|_{\bm{\Sigma}}^2 &\leq (1-(1-\gamma)\lambda_{\min}(\bm{\Sigma})\eta_t )\mathbb{E}\|\bm{\Delta}_t\|_{\bm{\Sigma}}^2\nonumber \\ 
& + 6\eta_t^2 \|\bm{\Sigma}\|\left(2\|\bm{\theta}^{\star}\|_{\bm{\Sigma}}^2 + 1\right).
\end{align}
Therefore, for constant stepsizes $\eta_0 = \eta_1 = \ldots = \eta_T = \eta$, it is easy to verify by induction that
\begin{align*}
\mathbb{E}\|\bm{\Delta}_T\|_{\bm{\Sigma}}^2 &\leq (1-(1-\gamma)\lambda_{\min}(\bm{\Sigma})\eta)^T \|\bm{\Delta}_0\|_{\bm{\Sigma}}^2 \\ 
&+ \frac{6\eta \|\bm{\Sigma}\|\left(2\|\bm{\theta}^{\star}\|_{\bm{\Sigma}}^2 + 1\right)}{(1-\gamma)\lambda_{\min}(\bm{\Sigma})}.
\end{align*}
Hence, in order to guarantee $\mathbb{E}\|\bm{\Delta}_T\|_{\bm{\Sigma}}^2\leq \varepsilon^2$, it suffices to take
\begin{align*}
&\frac{\eta \|\bm{\Sigma}\|\left(\|\bm{\theta}^{\star}\|_{\bm{\Sigma}}^2 + 1\right)}{(1-\gamma)\lambda_{\min}(\bm{\Sigma})} \lesssim \varepsilon^2; \quad \text{and} \\ 
&\exp \left(-(1-\gamma)\lambda_{\min}(\bm{\Sigma})\eta T \right)\|\bm{\Delta}_0\|_{\bm{\Sigma}}^2 \lesssim \varepsilon^2.
\end{align*}
This implies the following upper bound for the sample complexity:
\begin{align}\label{eq:Srikant-sample-complexity}
T \asymp \frac{\kappa \|\bm{\Sigma}^{-1}\|\left(\|\bm{\theta}^{\star}\|_{\bm{\Sigma}}^2 + 1\right)}{(1-\gamma)^2 } \frac{1}{\varepsilon^2} \log \frac{1}{\varepsilon},
\end{align}
with the proviso that we take the stepsize $\eta \asymp \frac{\|\bm{\Sigma}^{-1}\|}{1-\gamma}\frac{1}{T}$ and that $T \gtrsim \|\bm{\Sigma}^{-2}\|(1-\gamma)^{-2}$.

\subsection{Comparisons with \cite{bhandari2018finite}}\label{app:compare-Russo}
Theorem 2(c) in \citep{bhandari2018finite} shows that with decaying stepsizes $\eta_t = \frac{\beta}{\lambda+t}$ where
\begin{align}\label{eq:defn-beta-lambda}
\beta = \frac{2\|\bm{\Sigma}^{-1}\|}{(1-\gamma)}, \quad \lambda = \frac{16\|\bm{\Sigma}^{-1}\|}{(1-\gamma)^2},
\end{align}
the expected $\ell_2$ norm of TD estimation error is bounded by
\begin{align}
\label{eq:Russo-bound-2}
\mathbb{E}\|\bm{\theta}_T - \bm{\theta}^{\star}\|_2^2 \leq \frac{\nu}{\lambda+T},
\end{align}
where
\begin{align}\label{eq:defn-nu}
\nu = \max\left\{\frac{8\sigma^2\|\bm{\Sigma}^{-2}\|}{(1-\gamma)^2},\frac{16\|\bm{\theta}^\star\|_2^2\|\bm{\Sigma}^{-1}\|}{(1-\gamma)^2}\right\}.
\end{align}
\begin{itemize}
\item Suppose the maximum is attained at the second term for $\nu$ and $T$ is sufficiently large, \eqref{eq:Russo-bound-2} is simplified as
\begin{align*}
\mathbb{E}\|\bm{\theta}_T - \bm{\theta}^{\star}\|_2^2  \lesssim \frac{16\|\bm{\theta}^\star\|_2^2\|\bm{\Sigma}^{-1}\|}{(1-\gamma)^2 T}.
\end{align*}
In order for  $\mathbb{E}\|\bm{\theta}_T - \bm{\theta}^{\star}\|_{\bm{\Sigma}}^2 \leq \varepsilon^2,$ it suffices to take
\begin{align*}
\frac{\varepsilon^2}{\|\bm{\Sigma}\|} \geq \frac{16\|\bm{\theta}^\star\|_2^2\|\bm{\Sigma}^{-1}\|}{(1-\gamma)^2 T} \geq \mathbb{E}\|\bm{\theta}_T - \bm{\theta}^{\star}\|_2^2,
\end{align*}
which implies the following sample complexity:
\begin{align*}
T \asymp \frac{\|\bm{\Sigma}^{-1}\|\|\bm{\Sigma}\|\|\bm{\theta}^\star\|_2^2 }{(1-\gamma)^2 \varepsilon^2}
\end{align*}
\item Suppose that the first term on the right hand side of expression~\eqref{eq:defn-nu} is larger, \eqref{eq:Russo-bound-2} can be simplified as
\begin{align*}
\mathbb{E}\|\bm{\theta}_T - \bm{\theta}^{\star}\|_2^2  \lesssim \frac{\sigma^2 \|\bm{\Sigma}^{-2}\|}{(1-\gamma)^2 T}.
\end{align*}
Then similarly, the sample complexity is
\begin{align*}
T \asymp \frac{\|\bm{\Sigma}^{-2}\|\|\bm{\Sigma}\|\sigma^2 }{(1-\gamma)^2 \varepsilon^2},
\end{align*}{}
where $\sigma^2 = \mathbb{E}\|\bm{A}_t \bm{\theta}^\star - \bm{b}_t\|_2^2.$
\end{itemize}
In the worst-case scenario, it satisfies $\sigma^2 \asymp \|\bm{\theta}^\star\|_{\bm{\Sigma}}^2+ 1.$
Therefore, the sample complexity implied by Theorem 2(c) of \cite{bhandari2018finite} scales as 
\begin{align}
\label{eq:Russo-bound}
T \asymp \frac{\kappa \|\bm{\Sigma}^{-1}\|\left(\|\bm{\theta}^{\star}\|_{\bm{\Sigma}}^2 +1 \right)}{(1-\gamma)^2 } \frac{1}{\varepsilon^2}. 
\end{align}

\subsection{Comparison with \cite{mou2020optimal} and \cite{li2021accelerated}}
\label{sec:mou}

\cite{mou2020optimal} studied a more general problem of linear approximation for fixed point equations in Hilbert spaces, and considered its application to TD learning with linear function approximation and $i.i.d.$ samples. A similar result was reached by \cite{li2021accelerated} in their Theorem 2 and Theorem 3. While these works explored both the \emph{approximation error}, which measures the difference between $\bm{\Phi \theta}^\star$ and $\bm{V}^\star$ under our notation, and the \emph{statistical error}, which measures the convergence of $\bar{\bm{\theta}}_T$ to $\bm{\theta}^\star$, it is the latter that is directly comparable to our results. Therefore, we hereby provide a comparison between the statistical error term in their Corollary 5 and the sample complexity result of ours as shown in Theorem \ref{thm:ind-td}. Translated to our notation, \cite{mou2020optimal} proved that with a sufficiently large sample size $T$ and a stepsize of $\eta \asymp \frac{1}{\sqrt{T}}$, the estimation error of the averaged TD learning algorithm satisfies
\begin{align}\label{eq:mou-iid}
&\mathbb{E}_{s \sim \mu}[{V}_{\bar{\bm{\theta}}_T}(s) - V_{\bm{\theta}^\star}(s)]^2\nonumber \\ 
& = \|\bar{\bm{\theta}}_T - \bm{\theta}^\star\|_{\bm{\Sigma}}\nonumber \\ 
&\lesssim \frac{1}{T} \mathsf{Tr}\left[(\bm{I}-\bm{M})^{-1} (\bm{\Sigma}_{L} + \bm{\Sigma}_b)(\bm{I} - \bm{M})^{-\top}\right],
\end{align}
in which $\bm{M}$, $\bm{\Sigma}_L$ and $\bm{\Sigma}_b$ are defined as
\begin{align*}
&\bm{M} = \gamma \bm{\Sigma}^{-\frac{1}{2}} \mathbb{E}_{s \sim \mu, s' \sim P(\cdot \mid s)}[\bm{\phi}(s)\bm{\phi}(s')^\top] \bm{\Sigma}^{-\frac{1}{2}},\\
&\bm{\Sigma}_L = \mathbf{Cov}_{s_t \sim \mu, s_t' \sim P(\cdot \mid s_t)}[\bm{\Sigma}^{-\frac{1}{2}}\bm{A}_t\bm{\theta}^\star], \quad \text{and} \\ 
&\bm{\Sigma}_b = \mathbf{Cov}_{s_t \sim \mu}[\bm{\Sigma}^{-\frac{1}{2}}\bm{b}_t].
\end{align*}
\paragraph{Translation into our notation} We firstly translate the upper bound \eqref{eq:mou-iid} into our notation. By definition,
\begin{align*}
\mathbb{E}_{s \sim \mu, s' \sim P(\cdot \mid s)}[\bm{\phi}(s)\bm{\phi}(s')^\top] = \bm{\Phi}^\top \bm{D}_{\bm{\mu}}\bm{P} \bm{\Phi}.
\end{align*}
Therefore, the term $\bm{I}-\bm{M}$ can be expressed as
\begin{align*}
\bm{I} - \bm{M} &= \bm{\Sigma}^{-\frac{1}{2}} \bm{\Sigma} \bm{\Sigma}^{-\frac{1}{2}} - \gamma \bm{\Sigma}^{-\frac{1}{2}}\bm{\Phi}^\top \bm{D}_{\bm{\mu}}\bm{P} \bm{\Phi}\bm{\Sigma}^{-\frac{1}{2}} \\ 
&=\bm{\Sigma}^{-\frac{1}{2}} \left[\bm{\Phi}^\top \bm{D}_{\bm{\mu}}\bm{\Phi}- \gamma \bm{\Phi}^\top \bm{D}_{\bm{\mu}}\bm{P} \bm{\Phi} \right]\bm{\Sigma}^{-\frac{1}{2}}\\ 
&= \bm{\Sigma}^{-\frac{1}{2}} \bm{\Phi}^\top \bm{D}_{\bm{\mu}}(\bm{I}-\gamma \bm{P})\bm{\Phi}\bm{\Sigma}^{-\frac{1}{2}} = \bm{\Sigma}^{-\frac{1}{2}}\bm{A}\bm{\Sigma}^{-\frac{1}{2}}.
\end{align*}
Furthermore, the terms $\bm{\Sigma}_L$ and $\bm{\Sigma}_b$ can be expressed in our notation as
\begin{align*}
\bm{\Sigma}_L &= \bm{\Sigma}^{-\frac{1}{2}} \mathbf{Cov}_{s_t \sim \mu, s_t' \sim P(\cdot \mid s_t)}[\bm{A}_t \bm{\theta}^\star] \bm{\Sigma}^{-\frac{1}{2}} \\ 
&= \bm{\Sigma}^{-\frac{1}{2}} \mathbb{E}\left[[(\bm{A}_t - \bm{A})\bm{\theta}^\star][(\bm{A}_t - \bm{A})\bm{\theta}^\star]^\top \right]\bm{\Sigma}^{-\frac{1}{2}}, 
\end{align*}
and
\begin{align*} 
\bm{\Sigma}_b &= \bm{\Sigma}^{-\frac{1}{2}} \mathbf{Cov}_{s_t \sim \mu}[\bm{b}_t]\bm{\Sigma}^{-\frac{1}{2}}\\ 
&= \bm{\Sigma}^{-\frac{1}{2}} \mathbb{E}_{s_t \sim \mu}[(\bm{b}_t - \bm{b})(\bm{b}_t - \bm{b})^\top ]\bm{\Sigma}^{-\frac{1}{2}}.
\end{align*}
For simplicity, we will omit the subscript $s_t \sim \mu, s_t' \sim P(\cdot \mid s_t)$ in the following. Combining these terms, the upper bound in \eqref{eq:mou-iid} can be expressed as
\begin{align*}
&\frac{1}{T} \mathsf{Tr}\left[(\bm{I}-\bm{M})^{-1} (\bm{\Sigma}_{L} + \bm{\Sigma}_b)(\bm{I} - \bm{M})^{-\top}\right]\\ 
&= \frac{1}{T} \mathsf{Tr}\Big[ \left(\bm{\Sigma}^{\frac{1}{2}}\bm{A}^{-1}\bm{\Sigma}^{\frac{1}{2}}\right) \bm{\Sigma}^{-\frac{1}{2}}\mathbb{E}\Big[[(\bm{A}_t - \bm{A})\bm{\theta}^\star][(\bm{A}_t - \bm{A})\bm{\theta}^\star]^\top \\ 
&+(\bm{b}_t - \bm{b})(\bm{b}_t - \bm{b})^\top\Big]\bm{\Sigma}^{-\frac{1}{2}}\left(\bm{\Sigma}^{\frac{1}{2}}\bm{A}^{-\top}\bm{\Sigma}^{\frac{1}{2}}\right) \Big] \\ 
&= \frac{1}{T} \mathsf{Tr}\left[\bm{\Sigma}^{\frac{1}{2}}\bm{A}^{-1} \mathbb{E}\left[[(\bm{A}_t - \bm{A})\bm{\theta}^\star][(\bm{A}_t - \bm{A})\bm{\theta}^\star]^\top \right]\bm{A}^{-\top} \bm{\Sigma}^{\frac{1}{2}}\right] \\ 
&+ \frac{1}{T} \mathsf{Tr}\left[\bm{\Sigma}^{\frac{1}{2}}\bm{A}^{-1} \mathbb{E}\left[(\bm{b}_t - \bm{b})(\bm{b}_t - \bm{b})^\top \right]\bm{A}^{-\top} \bm{\Sigma}^{\frac{1}{2}}\right]\\ 
&=\frac{1}{T} \mathbb{E}\|\bm{A}^{-1}(\bm{A}_t - \bm{A})\bm{\theta}^\star\|_{\bm{\Sigma}}^2 + \frac{1}{T} \mathbb{E}\|\bm{A}^{-1}(\bm{b}_t - \bm{b})\|_{\bm{\Sigma}}^2
\end{align*}
So in summary, \cite{mou2020optimal} bounds the estimation error by 
\begin{align}\label{eq:mou-iid-translate}
&\|\bar{\bm{\theta}}_T - \bm{\theta}^\star\|_{\bm{\Sigma}}\nonumber \\ 
&\lesssim \frac{1}{T} \mathbb{E}\|\bm{A}^{-1}(\bm{A}_t - \bm{A})\bm{\theta}^\star\|_{\bm{\Sigma}}^2 + \frac{1}{T} \mathbb{E}\|\bm{A}^{-1}(\bm{b}_t - \bm{b})\|_{\bm{\Sigma}}^2.
\end{align}
\paragraph{Comparison to our results} In the following, we show that the upper bound \eqref{eq:mou-iid-translate} can be directly deducted from our proof of Theorem \ref{thm:ind-td}. Specifically, our analysis in \eqref{eq:Deltbar-decomposition-TD}, \eqref{eq:sum-Ai-xi-i-Bound1} and \eqref{eq:Ainv-A-sum-UB1} reveals that $\|\bar{\bm{\theta}}_T - \bm{\theta}^\star\|_{\bm{\Sigma}}^2$ is bounded by
\begin{align*}
\|\bar{\bm{\theta}}_T - \bm{\theta}^\star\|_{\bm{\Sigma}}^2 
&\lesssim \frac{1}{T^2} \left\|\sum_{i=0}^{T-1} \bm{A}^{-1}(\bm{A}_i - \bm{A})\bm{\theta}_i\right\|_{\bm{\Sigma}}^2 \\ 
& + \frac{1}{T^2} \left\|\sum_{i=0}^{T-1} \bm{A}^{-1}(\bm{b}_i - \bm{b})\right\|_{\bm{\Sigma}}^2 + o\left(\frac{1}{T}\right).
\end{align*}
Taking expectations on both sides and applying the martingale property, we obtain
\begin{align*}
&\mathbb{E}\|\bar{\bm{\theta}}_T - \bm{\theta}^\star\|_{\bm{\Sigma}}^2  \\ 
&\lesssim \frac{1}{T^2} \sum_{i=0}^{T-1} \mathbb{E}\left\|\bm{A}^{-1}(\bm{A}_i - \bm{A})\bm{\theta}_i\right\|_{\bm{\Sigma}}^2 + \frac{1}{T^2} \sum_{i=0}^{T-1} \mathbb{E}\left\|\bm{A}^{-1}(\bm{b}_i - \bm{b})\right\|_{\bm{\Sigma}}^2 \\ 
&\lesssim \frac{1}{T^2} \sum_{i=0}^{T-1} \mathbb{E}\left\|\bm{A}^{-1}(\bm{A}_i - \bm{A})\bm{\theta}^\star\right\|_{\bm{\Sigma}}^2 + \frac{1}{T^2} \sum_{i=0}^{T-1} \mathbb{E}\left\|\bm{A}^{-1}(\bm{b}_i - \bm{b})\right\|_{\bm{\Sigma}}^2\\ 
& + \frac{1}{T^2} \sum_{i=0}^{T-1} \mathbb{E}\left\|\bm{A}^{-1}(\bm{A}_i - \bm{A})\bm{\Delta}_i\right\|_{\bm{\Sigma}}^2\\ 
&= \frac{1}{T} \mathbb{E} \|\bm{A}^{-1}(\bm{A}_t - \bm{A})\bm{\theta}^\star\|_{\bm{\Sigma}}^2 + \frac{1}{T} \mathbb{E}\|\bm{A}^{-1}(\bm{b}_t - \bm{b})\|_{\bm{\Sigma}}^2\\ 
&+ \frac{1}{T^2} \sum_{i=0}^{T-1} \mathbb{E}\left\|\bm{A}^{-1}(\bm{A}_i - \bm{A})\bm{\Delta}_i\right\|_{\bm{\Sigma}}^2.
\end{align*}
Notice here that the first two terms are exactly the same as the right-hand-side of \eqref{eq:mou-iid-translate}. Hence, it now boils down to showing that
\begin{align}\label{eq:mou-remainder}
\frac{1}{T^2} \sum_{i=0}^{T-1} \mathbb{E}\left\|\bm{A}^{-1}(\bm{A}_i - \bm{A})\bm{\Delta}_i\right\|_{\bm{\Sigma}}^2 = o\left(\frac{1}{T}\right).
\end{align}
Towards this end, we firstly observe that
\begin{align*}
&\frac{1}{T^2} \sum_{i=0}^{T-1} \mathbb{E}\left\|\bm{A}^{-1}(\bm{A}_i - \bm{A})\bm{\Delta}_i\right\|_{\bm{\Sigma}}^2 \\ 
& \lesssim \frac{\|\bm{\Sigma}\|^2 \|\bm{\Sigma}^{-1}\|^2}{(1-\gamma)^2T^2} \sum_{i=0}^{T-1} \mathbb{E}\left\|\bm{\Delta}_i\right\|_{\bm{\Sigma}}^2.
\end{align*}
For the expectation of $\|\bm{\Delta}_i\|_{\bm{\Sigma}}^2$, we again apply the iterative relation deducted in \eqref{eq:Srikant-iter} and obtain
\begin{align*}
\mathbb{E}\|\bm{\Delta}_i\|_{\bm{\Sigma}}^2 &\leq (1-(1-\gamma)\lambda_{\min}(\bm{\Sigma})\eta)^i \|\bm{\Delta}_0\|_{\bm{\Sigma}}^2 \\ 
& + \frac{6\eta \|\bm{\Sigma}\|\left(2\|\bm{\theta}^{\star}\|_{\bm{\Sigma}}^2 + 1\right)}{(1-\gamma)\lambda_{\min}(\bm{\Sigma})}.
\end{align*}
Summing from $i=0$ through $i=T-1$ yields
\begin{align*}
\sum_{i=0}^{T-1} \mathbb{E} \|\bm{\Delta}_i\|_{\bm{\Sigma}}^2 &\leq \frac{1}{(1-\gamma)\lambda_{\min}(\bm{\Sigma})\eta}\|\bm{\Delta}_0\|_{\bm{\Sigma}}^2 \\ 
& + \frac{6\eta T \|\bm{\Sigma}\|\left(2\|\bm{\theta}^{\star}\|_{\bm{\Sigma}}^2 + 1\right)}{(1-\gamma)\lambda_{\min}(\bm{\Sigma})}.
\end{align*}
By setting $\eta \asymp T^{-1/2}$ as suggested by \cite{mou2020optimal}, this immediately implies
\begin{align*}
\frac{1}{T^2} \sum_{i=0}^{T-1} \mathbb{E}\left\|\bm{A}^{-1}(\bm{A}_i - \bm{A})\bm{\Delta}_i\right\|_{\bm{\Sigma}}^2 \lesssim \frac{1}{T^2}\cdot T^{1/2} = o\left(\frac{1}{T}\right).
\end{align*}
In summary, we have shown that the upper bound proposed by \cite{mou2020optimal} follows directly from our analysis. Our result, as is shown in Theorem \ref{thm:ind-td}, improves upon theirs in the sense that we use a stepsize $\eta$ that only depends on the logarithm of $T$, provide a bound with high probability instead of in expectation, and reveal a clearer dependence on the problem-related parameters.

\subsection{Comparison with \cite{durmus2022finite}}
\label{sec:durmus}

It is difficult to place the corresponding instance dependent results in comparison, so, we focus our attention on the minimax results.
In the following, we make use of the relations that $\|\bm{A} (\overline{\bm\theta}_T - {\bm\theta}^{\star})\|_2 \ge \|\bm{A} \bm{\Sigma}^{-1/2}\|\|\overline{\bm\theta}_T - {\bm\theta}^{\star}\|_{\bm{\Sigma}} \gtrsim (1-\gamma)\sqrt{\lambda_{\min}(\bm{\Sigma})}\|\overline{\bm\theta}_T - {\bm\theta}^{\star}\|_{\bm{\Sigma}}$, and $\mathbb{E}\|\bm{A}_t \bm{\theta}^\star - \bm{b}_t\|_2^2 \lesssim \frac{1}{(1-\gamma)^2}$, $\sup \|\bm{A}_t \bm{\theta}^\star - \bm{b}_t\|_2 \lesssim \frac{1}{1-\gamma}$.
We also consider the situations when $\|{\bm\theta}^{\star}\|_{\bm{\Sigma}} \lesssim \frac{1}{1-\gamma}$, and $\bm{\phi}(s)^{\top}\bm{\Sigma}^{-1}\bm{\phi}(s) \lesssim \lambda_{\min}(\bm{\Sigma})^{-1}$.
Notice that there exists an MDP instance such that equality can be attained in all these bounds simultaneously.
For ease of presentation, let us first rephrase the result \citet[Corollary 1]{durmus2022finite} in terms of our notation\footnote{We take $C_{\bm{A}} \lesssim (1-\gamma)^{-1}$, $a \asymp \|\bm{Q}\|^{-1} \lesssim (1-\gamma)\lambda_{\min}(\bm{\Sigma})$ and then $c_{\bm{A}} \lesssim \kappa^{-1}(1-\gamma)^3\lambda_{\min}(\bm{\Sigma})$ (see the definitions of these parameters in \citet{durmus2022finite}).}.
It is shown therein that for 
\begin{align*}
\eta \lesssim \frac{(1-\gamma)^3\lambda_{\min}(\bm{\Sigma})}{\kappa\sqrt{T}},
\end{align*}
with probability at least $1-\delta$, the averaged TD estimation error is bounded by
\begin{align}
\label{eq:durmus}
\|\overline{\bm\theta}_T - {\bm\theta}^{\star}\|_{\bm{\Sigma}} &\lesssim \sqrt{\frac{1}{\lambda_{\min}(\bm{\Sigma})(1-\gamma)^4T}},
\end{align}
when $T \gtrsim \frac{1}{c_{\bm{A}}^2} \gtrsim \frac{\kappa^2}{(1-\gamma)^6\lambda_{\min}(\bm{\Sigma})^2}$.
Here, we omit the dependency of log factors.
In comparison, our result delivers the same bound as long as $T \gtrsim \frac{\kappa^2}{(1-\gamma)^4\lambda_{\min}(\bm{\Sigma})}$.
We incur a lower born-in cost for the relation~\eqref{eq:durmus} to hold. 

\bibliographystyle{IEEEtran}
\bibliography{bibfileRL,bibfileRL-2}

\begin{IEEEbiographynophoto}{Gen Li}
(Member, IEEE) received his Ph.D. degree from the Department of Electronic Engineering at Tsinghua University, advised by Professor Yuantao Gu, and bachelor's degree from the Department of Electronic Engineering and Department of Mathematics at Tsinghua University in 2016.
He is currently an Assistant Professor at the Department of Statistic, the Chinese University of Hong Kong. 
Prior to that, he was a post-doctoral researcher at the Department of Statistics and Data Science, Wharton School, University of Pennsylvania. 
His research interests include diffusion based generative model, reinforcement learning, high-dimensional statistics, machine learning, signal processing, and mathematical optimization.
\end{IEEEbiographynophoto}

\begin{IEEEbiographynophoto}{Weichen Wu}
received the B.S. degree in Data Science and Big Data Technology from Peking University, Beijing, China, in 2015, and is currently a Ph.D. candidate from the Department of Statistics and Data Science at Carnegie Mellon University, co-advised by Alessandro Rinaldo and Yuting Wei. His research interests lie in Reinforcement Learning and Topological Data Analysis.
\end{IEEEbiographynophoto}

\begin{IEEEbiographynophoto}{Yuejie Chi}
(Fellow, IEEE; S'09--M'12--SM'17--F'23) received Ph.D. and M.A. in Electrical Engineering from Princeton University in 2012 and 2009, and B.E. (Hon.) in Electrical Engineering from Tsinghua University, Beijing, China, in 2007. She is currently the Sense of Wonder Group Endowed Professor of Electrical and Computer Engineering in AI Systems at Carnegie Mellon University, with courtesy appointments in the Machine Learning Department and CyLab. Her research interests lie in the theoretical and algorithmic foundations of data science, signal processing, machine learning and inverse problems, with applications in sensing, imaging, decision making, and AI systems. Among others, she is a recipient of Presidential Early Career Award for Scientists and Engineers (PECASE), SIAM Activity Group on Imaging Science Best Paper Prize, IEEE Signal Processing Society Young Author Best Paper Award, and the inaugural IEEE Signal Processing Society Early Career Technical Achievement Award for contributions to high-dimensional structured signal processing. She was named a Goldsmith Lecturer by IEEE Information Theory Society, a Distinguished Lecturer by IEEE Signal Processing Society, and a Distinguished Speaker by ACM. She currently serves or served as an Associate Editor for IEEE Trans. on Information Theory, IEEE Trans. on Signal Processing, IEEE Trans. on Pattern Recognition and Machine Intelligence, Information and Inference: A Journal of the IMA, and SIAM Journal on Mathematics of Data Science.
\end{IEEEbiographynophoto}

\begin{IEEEbiographynophoto}{Cong Ma}
(Member, IEEE) received the B.Eng. degree from Tsinghua University in 2015 and the Ph.D. degree from Princeton University in 2020. He was a Post-Doctoral Researcher with the Department of Electrical Engineering and Computer Sciences, UC Berkeley. He is currently an Assistant Professor with the Department of Statistics, The University of Chicago. His research interests include mathematics of data science, machine learning, high-dimensional statistics, convex, and nonconvex optimization. He has received the Student Paper Award from the International Chinese Statistical Association in 2017, the School of Engineering and Applied Science Award for Excellence from Princeton University in 2019, the AI Labs Fellowship from Hudson River Trading in 2019, the Hannan Graduate Student Travel Award from IMS in 2020, as well as the Best Paper Prize from SIAM Activity Group on Imaging Science in 2024.
\end{IEEEbiographynophoto}

\begin{IEEEbiographynophoto}{Alessandro Rinaldo}
is a Professor of Statistics and Data Science at the University of Texas at Austin. He received his Ph.D. in Statistics at Carnegie Mellon University (2005), and remained there as a faculty member until 2023. His research interests revolve around the theoretical properties of statistical and machine learning models for high-dimensional data under various structural assumptions, such as sparsity or intrinsic low dimensionality. An IMS fellow, he has served on the editorial boards of various journals, including the Annals of Statistics, the Journal of the American Statistics Association, and the Electronic Journal of Statistics.
\end{IEEEbiographynophoto}

\begin{IEEEbiographynophoto}{Yuting Wei}
(Member, IEEE) received the B.S. degree (Hons.) in statistics from Peking University, Beijing, China, in 2013, and the Ph.D. degree in statistics from the University of 
California, Berkeley in 2018, advised by Martin Wainwright and Aditya Guntuboyina. 
She is currently an Assistant Professor at the Department of Statistics and Data Science, Wharton School, University of Pennsylvania.
 Prior to that, She was with Carnegie Mellon University from 2019 to 2021 as an Assistant Professor of Statistics and Data Science, and with Stanford University as a Stein Fellow from 2018 to 2019. 
 Her research interests include high-dimensional and non-parametric statistics, statistical machine learning, and reinforcement learning. 
 She was a recipient of the NSF CAREER award, the Google Research Award, the ICSA Junior Research Award, honorable mention for Bernoulli Society's New Researcher Award, 
 and the Erich L. Lehmann Citation from Berkeley Statistics (awarded to the best dissertation in theoretical statistics).
\end{IEEEbiographynophoto}

\end{document}